\documentclass[11pt]{article}

\usepackage[numbers]{natbib}

    \usepackage[final]{neurips_2024}

\usepackage[colorlinks=True,citecolor=blue,urlcolor=blue,pagebackref=true,backref=true]
{hyperref}
\renewcommand*{\backrefalt}[4]{%
    \ifcase #1 \footnotesize{(Not cited.)}%
    \or        \footnotesize{(Cited on page~#2.)}%
    \else      \footnotesize{(Cited on pages~#2.)}%
    \fi}
\usepackage[normalem]{ulem}
\usepackage[none]{hyphenat}
\usepackage{breakcites}
\usepackage[utf8]{inputenc} %
\usepackage[T1]{fontenc}    %
\usepackage{hyperref}       %
\hypersetup{
  colorlinks = true,
  urlcolor = blue, 
  linkcolor = blue,
  citecolor = blue
}
\usepackage{booktabs}       %
\usepackage{amsfonts}       %
\usepackage{nicefrac}       %
\usepackage{microtype}      %
\usepackage{xcolor}         %
\usepackage{caption}
\usepackage{subcaption}
\usepackage{float}
\usepackage{amsmath,amsthm}
\usepackage{amssymb}
\usepackage{mathtools}
\usepackage{soul}
\usepackage{bm}
\usepackage{enumitem}
\usepackage{multirow}
\usepackage{wrapfig}
\usepackage{scalerel}
\allowdisplaybreaks
\usepackage{dsfont}
\usepackage[ruled,algo2e]{algorithm2e}%
\usepackage{xcolor}
\usepackage{thmtools,thm-restate}
\usepackage{cleveref}  
\usepackage{autonum} %

\usepackage{xspace}
\newtheorem{remark}{Remark}

\newtheorem{definition}{Definition}

\newtheorem{assumption}{Assumption}
\newtheorem{theorem}{Theorem}
\newtheorem{corollary}{Corollary}
\newtheorem{lemma}{Lemma}
\newtheorem{proposition}{Proposition}

\newcommand{\pmin}{p_{\mrm{\min}}}

\newcommand{\pmax}{p_{\max}}

\newcommand{\noise}{v}

\newcommand{\order}{\mc{O}}

\newcommand{\mfk}{\mathfrak}

\newcommand{\set}[1]{\mc{#1}}
\newcommand{\ball}{\mathbb{B}}
\newcommand{\X}{\set{X}}
\newcommand{\indicator}{\mbf 1}

\newcommand{\snorm}[1]{\Vert #1 \Vert}
\newcommand{\sinfnorm}[1]{\snorm{#1}_\infty}

\newcommand{\lone}{{L}^1}
\newcommand{\ltwo}{{L}^2}

\newcommand{\sumn}[1][i]{\sum_{#1=1}^n}
\newcommand{\sumnout}[1][i]{\sum_{#1=1}^{\nout}}
\newcommand{\sless}[1]{\stackrel{#1}{\leq}}
\newcommand{\sgrt}[1]{\stackrel{#1}{\geq}}

\newcommand{\seq}[1]{\stackrel{#1}{=}}

\newcommand{\x}{x}

\newcommand{\z}{z}

\newcommand{\axi}[1][i]{\x_{#1}}

\newcommand{\kernel}{\mbf{k}}
\newcommand{\kersplit}{\kernel_{\trm{split}}}
\newcommand{\rkhs}{\mc{H}}
\renewcommand{\k}{\kernel}

\newcommand{\mkernel}{\mbf{K}}

\newcommand{\hnorm}[1]{\Vert{#1}\Vert_{\rkhs}}

\newcommand{\knorm}[1]{\Vert{#1}\Vert_{\kernel}}

\DeclareMathOperator{\mmd}{MMD}

\newcommand{\eps}{\epsilon}

\def\Matern{Mat\'ern\xspace}

\newcommand{\proofref}[1]{\noindent\text{See \cref{#1} for the proof.\\}}

\newcommand{\ncref}[1]{\cref{#1}: \nameref*{#1}} %

\newcommand{\pcref}[1]{Proof of \ncref{#1}} %

\newcommand{\vareps}{\varepsilon}

\newcommand{\Linf}{L^\infty}

\newcommand{\cover}{\mc C}

\newcommand{\tail}[1][\kernel]{\tau_{#1}}

\newcommand{\radius}{R}

\newcommand{\sgparam}[1][i]{\sigma_{#1}}
\newcommand{\vmax}[1][i]{\mathfrak{b}_{#1}}

\newcommand{\eventnotag}{\mc{E}}
\newcommand{\event}[1][]{\eventnotag_{#1}}

\newcommand{\wtil}[1]{\widetilde{#1}}

\newcommand{\cnew}[1][i]{\mfk{a}_{#1}}

\newcommand{\pseqxn}[1][n]{(\axi[i])_{i\geq 1}} %
\newcommand{\pseqxnn}[1][n]{(\axi[i])_{i=1}^n} %

\newcommand{\coveringnumber}{\mc N}

\newcommand{\coreset}[1][j]{\mathcal{S}^{(#1)}}

\newcommand{\brackets}[1]{\left[ #1 \right]}

\newcommand{\parenth}[1]{\left( #1 \right)}

\newcommand{\sbraces}[1]{\{ #1  \}}
\newcommand{\braces}[1]{\left\{ #1 \right \}}
\newcommand{\bigbraces}[1]{\big\{ #1 \big \}}
\newcommand{\biggbraces}[1]{\bigg\{ #1 \bigg \}}
\newcommand{\abss}[1]{\left| #1 \right |}

\newcommand{\inv}{^{-1}}
\newcommand{\real}{\ensuremath{\mathbb{R}}}

\newcommand{\Prob}{\ensuremath{{\mathbb{P}}}}

\def\balign#1\ealign{\begin{align}#1\end{align}}
\def\baligns#1\ealigns{\begin{align*}#1\end{align*}}
\def\balignat#1\ealign{\begin{alignat}#1\end{alignat}}
\def\balignats#1\ealigns{\begin{alignat*}#1\end{alignat*}}
\def\bitemize#1\eitemize{\begin{itemize}#1\end{itemize}}
\def\benumerate#1\eenumerate{\begin{enumerate}#1\end{enumerate}}

\newenvironment{talign*}
 {\csname align*\endcsname}
 {\endalign}
\newenvironment{talign}
 {\csname align\endcsname}
 {\endalign}

\def\balignst#1\ealignst{\begin{talign*}#1\end{talign*}}
\def\balignt#1\ealignt{\begin{talign}#1\end{talign}}
\newcommand{\qtext}[1]{\quad\text{#1}\quad}

\let\originalleft\left
\let\originalright\right
\renewcommand{\left}{\mathopen{}\mathclose\bgroup\originalleft}
\renewcommand{\right}{\aftergroup\egroup\originalright}

\def\Holder{H\"older\xspace}

\def\Nystrom{Nystr\"om\xspace}

\def\Matern{Mat\'ern\xspace}

\def\tinycitep*#1{{\tiny\citep*{#1}}}
\def\tinycitealt*#1{{\tiny\citealt*{#1}}}
\def\tinycite*#1{{\tiny\cite*{#1}}}
\def\smallcitep*#1{{\scriptsize\citep*{#1}}}
\def\smallcitealt*#1{{\scriptsize\citealt*{#1}}}
\def\smallcite*#1{{\scriptsize\cite*{#1}}}

\def\mbf#1{\mathbf{#1}}
\def\mbb#1{\mathbb{#1}}
\def\mc#1{\mathcal{#1}}
\def\mrm#1{\mathrm{#1}}
\def\trm#1{\textrm{#1}}
\def\tbf#1{\textbf{#1}}

\def\reals{\mathbb{R}} %
\def\R{\mathbb{R}}

\def\naturals{\mathbb{N}} %

\def\<{\left\langle} %
\def\>{\right\rangle}

\def\iff{\Leftrightarrow}
\def\implies{\quad\Longrightarrow\quad}

\def\defeq{\triangleq} %
\def\half{\frac{1}{2}}

\newcommand{\floor}[1]{\lfloor{#1}\rfloor}
\newcommand{\ceil}[1]{\lceil{#1}\rceil}
\newcommand{\boldone}{\mbf{1}} %
\newcommand{\ident}{\mbf{I}} %
\def\norm#1{\left\|{#1}\right\|} %
\newcommand{\twonorm}[1]{\norm{#1}_2}
\newcommand{\stwonorm}[1]{\snorm{#1}_2}%
\newcommand{\infnorm}[1]{\norm{#1}_{\infty}} %
\newcommand{\nnorm}[1]{\norm{#1}_n} %

\def\staticnorm#1{\|{#1}\|} %
\newcommand{\statictwonorm}[1]{\staticnorm{#1}_2} %
\newcommand{\staticnnorm}[1]{\staticnorm{#1}_n} %
\newcommand{\inner}[2]{\langle{#1},{#2}\rangle} %
\def\what#1{\widehat{#1}}

\def\indic#1{\mbb{I}\left[{#1}\right]} %
\def\E{\mbb{E}} %
\def\Earg#1{\E\left[{#1}\right]}

\def\Esubarg#1#2{\E_{#1}\left[{#2}\right]}
\def\bigO#1{\mathcal{O}(#1)} %
\def\P{\mbb{P}} %
\def\Parg#1{\P\left({#1}\right)}
\def\Psubarg#1#2{\P_{#1}\left[{#2}\right]}
\def\Var{\mrm{Var}} %

\newcommand{\Gsn}{\mathcal{N}}

\newcommand{\iid}{\textrm{i.i.d.}\xspace}
\newcommand{\dist}{\sim}
\newcommand{\distiid}{\overset{\textrm{\tiny\iid}}{\dist}}

\providecommand{\argmin}{\mathop\mathrm{arg min}}

\def\rank#1{\mathrm{rank}({#1})}

\ifdefined\nonewproofenvironments\else
\ifdefined\ispres\else
\newenvironment{proof-sketch}{\noindent\textbf{Proof Sketch}
  \hspace*{1em}}{\qed\bigskip\\}
\newenvironment{proof-idea}{\noindent\textbf{Proof Idea}
  \hspace*{1em}}{\qed\bigskip\\}
\newenvironment{proof-of-lemma}[1][{}]{\noindent\textbf{Proof of Lemma {#1}}
  \hspace*{1em}}{\qed\\}
\newenvironment{proof-of-theorem}[1][{}]{\noindent\textbf{Proof of Theorem {#1}}
  \hspace*{1em}}{\qed\\}
\newenvironment{proof-attempt}{\noindent\textbf{Proof Attempt}
  \hspace*{1em}}{\qed\bigskip\\}

\newcommand{\cset}{\mc{S}}
\newcommand{\inputcoreset}{\cset_{\mrm{in}}}
\newcommand{\outputcoreset}{\cset_{\mrm{out}}}
\newcommand{\ktcoreset}{\cset_{\mrm{KT}}}

\newcommand{\err}{\mathfrak{M}}%

\newcommand{\kt}{\textsc{KT}\xspace}
\newcommand{\ktsplit}{\hyperref[algo:ktsplit]{\color{black}{\textsc{kt-split}}}\xspace}

\newcommand{\ktswap}{\hyperref[algo:ktswap]{\color{black}{\textsc{kt-swap}}}\xspace}

\newcommand{\Wendland}[1][\wendparam]{{Wendland$(#1)$}\xspace}

\newcommand{\nout}{n_{\mrm{out}}}
\newcommand{\ininfnorm}[1]{\norm{#1}_{\infty,\mrm{in}}}

\newcommand{\ktcomplexity}{\eta_{n,\kernel}}
\newcommand{\ktcomplexityconstant}{\mfk a}
\newcommand{\Rin}{\mfk R_{\mrm{in}}}

\newcommand{\nw}{\widehat f }
\newcommand{\nwkt}{\widehat f_{\mrm{KT}}}
\newcommand{\numer}{\widehat A}
\newcommand{\denom}{\widehat p}
\newcommand{\numerkt}{\widehat A_{\mrm{KT}}}
\newcommand{\denomkt}{\widehat p_{\mrm{KT}}}
\newcommand{\ymax}{Y_{\mrm{max}}}
\newcommand{\kernelnw}{\kernel_{\mrm{NW}}}

\newcommand{\modelclass}{\mathcal{F}}
\newcommand{\fstar}{f^{\star}}

\newcommand{\Y}{\mathcal{Y}}

\newcommand{\fullnormbound}{\mc E_{\trm{full}, n}}
\newcommand{\ktnormbound}{\mc E_{\trm{KT}, n}}

\newcommand{\deltakt}{\widehat \Delta_{\mrm{KT}}}
\newcommand{\deltafull}{\widehat \Delta_{\mrm{full}}}

\newcommand{\kernelrr}{\kernel_{\mrm{RR}}}

\newcommand{\krr}{\widehat{f}_{\mrm{full},\lambda}}
\newcommand{\krrkt}{\widehat{f}_{\mrm{KT},\lambda^\prime}}
\newcommand{\krrktname}{\textsc{KT-KRR}\xspace}
\newcommand{\krrstname}{\textsc{ST-KRR}\xspace}
\newcommand{\krrfullname}{\textsc{Full-KRR}\xspace}

\newcommand{\nwfullname}{\textsc{Full-NW}\xspace}
\newcommand{\nwstname}{\textsc{ST-NW}\xspace}
\newcommand{\nwktname}{\textsc{KT-NW}\xspace}

\newcommand{\noutnorm}[1]{\norm{#1}_{\nout}}
\newcommand{\snoutnorm}[1]{\snorm{#1}_{\nout}}
\newcommand{\staticnoutnorm}[1]{\staticnorm{#1}_{\nout}} \newcommand{\snnorm}[1]{\snorm{#1}_{n}}
\newcommand{\ag}[1]{}
\newcommand{\kc}[1]{}
\newcommand{\wip}[1]{}
\newcommand{\todo}[1]{}
\newcommand{\tocheck}[1]{}

\newcommand{\kalg}{\kernel_{\textsc{Alg}}}
\newcommand{\kalgnorm}[1]{\norm{#1}_{\kalg}}

\newcommand{\rounds}{\beta_n}

\newcommand{\Pin}{\mbb P_{\textup{in}}}
\newcommand{\Qout}{\mbb Q_{\textup{out}}}

\newcommand{\nin}{n}

\newcommand{\osname}{compression level\xspace}
\newcommand{\ossymb}{\mfk g}
\newcommand{\compress}{\textsc{Compress}\xspace}

\newcommand{\compresssub}{\textsc{C}}
\newcommand{\compressppsub}{\textsc{C++}\xspace}

\newcommand{\subg}{\nu}

\newcommand{\vspone}[1][1]{\vspace{-#1mm}}
\newtheorem{myproposition}{Proposition}

\newtheorem{mycorollary}{Corollary}

\newtheorem{mylemma}{Lemma}

\newtheorem{mydefinition}{Definition}

\newtheorem{myremark}{Remark}

\newtheorem{myassumption}{Assumption}

 \crefname{appendix}{App.}{App.}
\crefname{equation}{}{}
\crefname{lemma}{Lem.}{Lem.}
\crefname{theorem}{Thm.}{Thm.}
\crefname{Corollary}{Cor.}{Cors.}
\crefname{algorithm}{Alg.}{Algs.}

\crefname{section}{Sec.}{Sec.}
\crefname{table}{Tab.}{Tab.}
\crefname{remark}{Rem.}{Rem.}
\crefname{definition}{Def.}{Def.}
\crefname{Proposition}{Prop.}{Prop.}
\crefname{myremark}{Rem.}{Rem.}
\crefname{mylemma}{Lem.}{Lem.}
\crefname{mydefinition}{Def.}{Defs.}
\crefname{myproposition}{Prop.}{Prop.}
\crefname{mycorollary}{Cor.}{Cors.}
\crefname{myassumption}{Assum.}{Assum.}
\crefname{figure}{Fig.}{Fig.}
\crefname{enumi}{}{}
\crefname{name}{}{} %

\title{Supervised Kernel Thinning}
\author{
\textbf{Albert Gong \quad Kyuseong Choi \quad Raaz Dwivedi} \\[2mm]
Cornell Tech, Cornell University\\[2mm]
\texttt{agong,kc728,dwivedi@cornell.edu}
}
\date{\today}

\begin{document}

\sloppy
\maketitle

\begin{abstract}
The kernel thinning algorithm of \citep{dwivedi2024kernel} provides a better-than-i.i.d. compression of a generic set of points. By generating high-fidelity coresets of size significantly smaller than the input points, KT is known to speed up unsupervised tasks like Monte Carlo integration, uncertainty quantification, and non-parametric hypothesis testing, with minimal loss in statistical accuracy. In this work, we generalize the KT algorithm to speed up supervised learning problems involving kernel methods. 
Specifically, we combine two classical algorithms---Nadaraya-Watson (NW) regression or kernel smoothing, and kernel ridge regression (KRR)---with KT to provide a \emph{quadratic} speed-up in both training and inference times. We show how distribution compression with KT in each setting reduces to constructing an appropriate kernel, and introduce the Kernel-Thinned NW and Kernel-Thinned KRR estimators. We prove that KT-based regression estimators enjoy significantly superior computational efficiency over the full-data estimators and improved statistical efficiency over i.i.d. subsampling of the training data. En route, we also provide a novel multiplicative error guarantee for compressing with KT.  We validate our design choices with both simulations and real data experiments.

\end{abstract}

\section{Introduction}

In supervised learning, the goal of coreset methods is to find a representative set of points on which to perform model training and inference. 
On the other hand, coreset methods in \textit{unsupervised} learning have the goal of finding a representative set of points, which can then be utilized for a broad class of downstream tasks---from integration \citep{dwivedi2024kernel,dwivedi2022generalized} to non-parametric hypothesis testing \citep{domingo2023compress}. 
This work aims to bridge these two research threads. 

Leveraging recent advancements from compression in the unsupervised setting, we tackle the problem of non-parametric regression (formally defined in \cref{sec:problem setup}). Given a dataset of $n$ \iid samples, $(x_i, y_i)_{i=1}^n$, we want to learn a function $f$ such that $f(x_i) \approx y_i$.
The set of allowable functions is determined by the kernel function, which is a powerful building block for capturing complex, non-linear relationships. Due to its powerful performance in practice and closed-form analysis, non-parametric regression methods based on kernels (a.k.a ``kernel methods'') have  become a popular choice for a wide range of supervised learning tasks \citep{huang2022robust,radhakrishnan2022feature,singh2021sequential}.

There are two popular approaches to non-parametric kernel regression.
First, perhaps a more classical approach, is kernel smoothing, also referred to as Nadaraya-Watson (NW) regression. The NW estimator at a point $x$ is effectively a smoothing of labels $y_i$ such that $x_i$ is close to $x$. These weights are computed using the kernel function (see \cref{sec:problem setup} for formal definitions).
Importantly, the NW estimator takes $\Theta(n)$ pre-processing time (to simply store the data) and $\Theta(n)$ inference time for each test point $x$ ($n$ kernel evaluations and $n$ simple operations).

Another popular approach is kernel ridge regression (KRR), which solves a non-parametric least squares subject to the regression function lying in the reproducing kernel Hilbert space (RKHS) of a specified reproducing kernel function. 
Remarkably, KRR admits a closed-form solution via inverting the associated kernel matrix, and takes $\bigO{n^3}$ training time and $\Theta(n)$ inference time for each test point $x$.

Our goal is to overcome the computational bottlenecks of kernel methods, while retaining their favorable statistical properties. Previous attempts at using coreset methods include the work of
\citet{boutsidis2013near, zheng2017coresets, phillips2017coresets}, which depend on a projection type compression, having similar spirit to the celebrated Johnson–Lindenstrauss lemma, a metric preserving projection result. So accuracy and running depend unfavorably on the desired statistical error rate. 
\citet{kpotufe2009fast} propose an algorithm to reduce the query time of the NW estimator to $\bigO{\log n}$, but the algorithm requires super-linear preprocessing time. 

 Other lines of work exploit the structure of kernels more directly, especially in the KRR literature. A slew of techniques from numerical analysis have been developed, including work on \Nystrom subsampling by \citet{el2014fast,avron2017faster,diaz2023robust}.
\citet{camoriano2016nytro} and \citet{rudi2017falkon} combine early stopping with \Nystrom subsampling.
Though more distant from our approach, we also note the approach of \citet{rahimi2007random} using random features, \citet{zhang2015divide} using Divide-and-Conquer, and \citet{tu2016large} using block coordinate descent. 

\paragraph{Our contributions.}
In this work, we show how coreset methods can be used to speed up both training and inference in non-parametric regression for a large class of function classes/kernels. At the heart of these algorithms is a general procedure called kernel thinning \cite{dwivedi2024kernel, dwivedi2022generalized}, which provides a worst-case bound on integration error (suited for problems in the original context of unsupervised learning and MCMC simulations). In \cref{sec:speeding-up}, we introduce a meta-algorithm that recovers our two thinned non-parametric regression methods each based on NW and KRR. We introduce the \emph{kernel-thinned Nadaraya-Watson estimator} (\nwktname) and the \emph{kernel-thinned kernel ridge regression} estimator (\krrktname). 

We show that \nwktname requires $\bigO{n \log^3 n}$ time during training and $\bigO{\sqrt n}$ time at inference, while achieving a mean square error (MSE) rate of $n^{-\frac{\beta}{\beta+d}}$ (\cref{thm:nadaraya-kt})---a strict improvement over uniform subsampling of the original input points. We show that \krrktname requires $\bigO{n^{3/2}}$ time during training and and $\bigO{\sqrt n}$ time during inference, while achieving an near-minimax optimal rate of $\frac{m \log n}{n}$ when the kernel has finite dimension (\cref{thm:kt-krr-finite}). We show how our \krrktname guarantees can also be extended to the infinite-dimension setting (\cref{cor:krr-kt}). 
In \cref{sec:results}, we apply our proposed methods to both simulated and real-world data. In line with our theory, \nwktname and \krrktname outperform standard thinning baselines in terms of accuracy while retaining favorable runtimes.

\section{Problem setup}
\label{sec:problem setup}

We now formally describe the non-parametric regression problem.
Let $x_1,\ldots,x_n$ be i.i.d. samples from the data distribution $\P$ (with density $p$) over the domain $\X\subset \reals^d$ and $w_1,\ldots,w_n$ be i.i.d. samples from $\Gsn(0,1)$. In the sequel, we use $\snorm{\cdot}$ to denote the Euclidean norm unless otherwise stated. Then define the response variables $y_1,\ldots,y_n$ by the follow data generating process:
\begin{align}\label{eq:process}
    y_i \defeq \fstar(x_i)+\noise_i \qtext{for} i=1,2,\ldots,n,
\end{align}
where $\fstar:\X\to \Y\subset \reals$ is the \emph{regression function} and $\noise_i \defeq \sigma w_i$ for some noise level $\sigma>0$. Our task is to build an estimate for $\fstar$ given the $n$ observed points, denoted by
\begin{align}
    \inputcoreset \defeq \parenth{(x_1,y_1),\ldots,(x_n,y_n)}.
\end{align}

\paragraph{Nadaraya-Watson (NW) estimator.} 
A classical approach to estimate the function $\fstar$ is kernel smoothing, where one estimates the function value at a point $z$ using a weighted average of the observed outcomes. The weight for outcome $y_i$ depends on how close $x_i$ is to the point $z$; let $\kappa:\real^d \to \real$ denote this weighting function such that the weight for $x_i$ is proportional to $\kappa\parenth{\norm{x_i-z}/h}$ for some bandwidth parameter $h>0$. Let $\kernel:\real^d \times \real \to \real$ denote a shift-invariant kernel defined as 
\begin{align}\label{eq:shift-invariant}
    \kernel(x_1,x_2)=\kappa\parenth{\norm{x_1-x_2}/h}.
\end{align}
Then this smoothing estimator, also known as Nadaraya-Watson (NW) estimator, can be expressed as
\begin{align}\label{eq:nw}
    \nw(\cdot) \defeq  \frac{\sum_{(x,y)\in \inputcoreset} \kernel (\cdot,x) y}{ \sum_{x\in \inputcoreset} \kernel(\cdot,x)}
\end{align}
whenever the denominator in the above display is non-zero. In the case the denominator in \cref{eq:nw} is zero, we can make a default choice, which for simplicity here we choose as zero.
We refer to the estimator~\cref{eq:nw}  as \nwfullname estimator hereafter. One can easily note that \nwfullname requires $\bigO{n}$ storage for the input points and $\bigO{n}$ kernel queries for inference at each point.

\paragraph{Kernel ridge regression (KRR) estimator.} Another popular approach to estimate $\fstar$ is that of non-parametric (regularized) least squares. The solution in this approach, often called as the kernel ridge regression (KRR), is obtained by solving a least squares objective where the fitted function is posited to lie in the RKHS $\rkhs$ of a reproducing kernel $\kernel$, and a regularization term is added to the objective to avoid overfitting.\footnote{We note that while KRR approach~\cref{eq:kernel-krr} does require $\kernel$ to be reproducing, the NW approach~\cref{eq:nw} in full generality is valid even when $\kernel$ is a not a valid reproducing kernel.} Overall, the KRR estimate is the solution to the following regularized least-squares objective, where $\lambda>0$ denotes a regularization hyperparameter:
\begin{align}\label{eq:krr-objective}
    \min_{f\in\rkhs} L_{\inputcoreset} + \lambda \knorm{f}^2,\qtext{where} L_{\inputcoreset} \defeq \frac{1}{n} \sum_{(x,y)\in \inputcoreset} \parenth{f(x) - y}^2.
\end{align}
Like NW, an advantage of KRR is the existence of a closed-form solution
\begin{align}
    \krr(\cdot) &\defeq \sum_{i=1}^{n} \alpha_i \kernel(\cdot, x_i)
    \qtext{where} \label{eq:krr} \\
    \boldsymbol \alpha &\defeq (\mbf K+n\lambda \ident_n)\inv \begin{bmatrix}y_1 \\  \vdots \\  y_n\end{bmatrix} \in \reals^n \qtext{and} 
    \mkernel \defeq [\kernel(x_i,x_j)]_{i,j=1}^n \in \R^{n\times n}
    \label{eq:mkernel}.
\end{align}
Notably, the estimate $\krr$, which we refer to as the \krrfullname estimator, can also be seen as yet another instance of weighted average of the observed outcomes. Notably, NW estimator imposes that the weights across the points sum to $1$ (and are also non-negative whenever $\kernel$ is), KRR allows for generic weights that need not be positive (even when $\kernel$ is) and need not sum to $1$.
We note that naïvely solving $\krr$ requires $\order(n^2)$ kernel evaluations to compute the kernel matrix, $\order(n^3)$ to compute a matrix inverse, and $\order(n)$ kernel queries for inference at each point. One of our primary goals in this work is to tackle this high computational cost of \krrfullname.

\section{Speeding up non-parametric regression}\label{sec:speeding-up}

\newcommand{\critthresh}{\xi_n}
\newcommand{\lambdakt}{\lambda'}
\newcommand{\lambdafull}{\lambda}
We begin with a general approach to speed up regression by thinning the input datasets. While computationally superior, a generic approach suffers from a loss of statistical accuracy motivating the need for a strategic thinning approach. To that end, we briefly review kernel thinning and finally introduced our supervised kernel thinning approach.
\subsection{Thinned regression estimators: Computational and statistical tradeoffs}
Our generic approach comprises two main steps. First, we compress the input data by choosing a coreset $\outputcoreset \subset \inputcoreset$ of size $\nout \defeq | \outputcoreset |$. Second, we apply our off-the-shelf non-parametric regression methods from \cref{sec:problem setup} to the compressed data. By setting $\nout \ll n$, we can obtain notable speed-ups over the \textsc{Full} versions of NW and KRR.

Before we introduce the thinned versions of NW and KRR, let us define the following notation. Given an input sequence $\inputcoreset$ and output sequence  $\outputcoreset$, define the empirical probability measures
\begin{align}\label{eq:Pin-Qout}
    \Pin \defeq \frac{1}{n}\sum_{(x,y)\in \inputcoreset} \delta_{(x,y)} \qtext{and} \Qout \defeq \frac{1}{\nout}\sum_{(x,y)\in \outputcoreset} \delta_{(x,y)}.
\end{align}

\paragraph{Thinned NW estimator.}
The thinned NW estimator is the analog of Full-NW except that $\inputcoreset$ is replaced by  $\outputcoreset$ in \cref{eq:nw} so that the \textit{thinned-NW estimator} is given by
\begin{align}\label{eq:nadaraya-kt}
    \what f_{\outputcoreset}(\cdot) \defeq \frac{\sum_{(x,y)\in \outputcoreset} \kernel(\cdot,x) y}{ \sum_{x\in \outputcoreset} \kernel(\cdot, x)}
    = \frac{\Qout(y\kernel)}{\Qout \kernel}
\end{align}
whenever the denominator in the display is not zero; and $0$ otherwise. When compared to the \nwfullname estimator, we can easily deduce the computational advantage of this estimator: more efficient $\order(\nout)$ storage as well as the faster $\order(\nout)$ computation for inference at each point.

\paragraph{Thinned KRR estimator.}
Similarly, we can define the \textit{thinned KRR estimator} as
\begin{align}\label{eq:thinned-krr-solution}
    \what f_{\outputcoreset,\lambdakt}(\cdot) &= \sum_{i=1}^{\nout} \alpha_i' \kernel(\cdot,x_i'),\qtext{where} \\
    \boldsymbol{\alpha}' &\defeq \parenth{\mbf K'+\nout \lambda' \ident_{\nout}}\inv \begin{bmatrix}y_1' \\  \vdots \\  y_{\nout}'\end{bmatrix} \in \reals^{\nout} \qtext{and} \mkernel' \defeq [\kernel(x_i',x_j')]_{i,j=1}^{\nout}\in \R^{\nout \times \nout}
\end{align}
given some regularization parameter $\lambdakt >0$. When compared to \krrfullname, $\what f_{\outputcoreset,\lambdakt}$ has training time $\bigO{\nout^3}$ and prediction time $\bigO{\nout}$. 

A baseline approach is standard thinning, whereby we let $\outputcoreset$ be an i.i.d. sample of $\nout=\sqrt n$ points from $\inputcoreset$.
For NW, let us call the resulting $\what f_{\outputcoreset}$ \cref{eq:nadaraya-kt} the standard-thinned Nadaraya-Watson (\nwstname) estimator. When $\nout = \sqrt n$, \nwstname achieves an excess risk rate of $\bigO{ n^{-\frac{\beta}{2\beta+d}} }$ compared to the \nwfullname rate of $\bigO{n^{-\frac{2\beta}{2\beta+d}}}$.
For KRR, let us call the resulting $\what f_{\outputcoreset,\lambdakt}$ \cref{eq:thinned-krr-solution} the standard-thinned KRR (\krrstname) estimator. When $\nout = \sqrt n$, \krrstname achieves an excess risk rate of $\bigO{\frac{m}{\nout}}$ compared to the \krrfullname rate of $\bigO{\frac{m}{n}}$.
Our goal is to provide good computational benefits without trading off statistical error. Moreover, we may be able to do better by leveraging the underlying geometry of the input points and summarize of the input distribution more succinctly than i.i.d. sampling.

\subsection{Background on kernel thinning}
\label{sub:kt}
\newcommand{\ktcompresspp}{\hyperref[app:ktcompresspp]{\color{black}{\textsc{KT-Compress++}}}\xspace}
\newcommand{\ktcompress}{\hyperref[algo:ktcompress]{\color{black}{\textsc{KT-Compress}}}\xspace}
\newcommand{\errmmd}{\mfk W}

A subroutine central to our approach is kernel thinning (KT) from \citet[Alg.~1]{dwivedi2024kernel}. We use a variant called \ktcompresspp from \citet[Ex.~6]{shetty2022distribution} (see full details in \cref{app:ktcompresspp}), which provides similar approximation quality as the original KT algorithm of \citet[Alg.~1]{dwivedi2024kernel}, while reducing the runtime from $\bigO{n^2}$ to $\bigO{n \log^3 n}$.\footnote{In the sequel, we use ``KT'' and ``\ktcompresspp'' interchangeably since the underlying algorithm (kernel halving \cite[Alg.~1a]{dwivedi2024kernel}) and associated approximation guarantees are the same up to small constant factors.}
Given an input kernel $\kalg$ and input points $\inputcoreset$, \ktcompresspp outputs a coreset $\ktcoreset \subset \inputcoreset$ with size $\nout \defeq \sqrt n \ll n$. In this work, we leverage two guarantees of \ktcompresspp. Informally, $\ktcoreset$ satisfies (with high probability):
\begin{align}
    (\Linf~\text{bound})~\quad
    \infnorm{ (\Pin - \Qout) \kalg } &\leq C_1 \frac{\sqrt d \log \nout}{\nout} \label{eq:ktcompresspp-Linf-informal} \\
    (\mmd~\text{bound})~
    \sup_{\kalgnorm{h} \leq 1}  \abss{ (\Pin-\Qout) h} &\leq C_2 \frac{\sqrt{\log \nout \cdot \log \coveringnumber_{\kalg} (\ball_2(\Rin),1/\nout)}}{\nout} , \label{eq:ktcompresspp-mmd-informal}
\end{align}
where $C_1,C_2>0$ are constants that depend on the properties of the input kernel $\kalg$ and the chosen failure probability of \ktcompresspp, $\Rin$ characterizes the radius of $\braces{x_i}_{i=1}^n$, and $\coveringnumber_{\kalg} (\ball_2(\Rin), 1/\nout)$ denotes the kernel covering number of $\rkhs(\kalg)$ over the ball $\ball_2(\Rin)\subset \reals^d$ at a specified tolerance (see \cref{sec:kt-krr} for formal definitions).

At its highest level, KT provides good approximation of function averages. The bound \cref{eq:ktcompresspp-Linf-informal} (formally stated in \cref{lem:ktcompresspp-Linf}) controls the worst-case point-wise error, and is near-minimax optimal by \citet[Thm.~3.1]{phillips2018improved}. In the sequel, we leverage this type of result to derive generalization bounds for the kernel smoothing problem. 
The bound \cref{eq:ktcompresspp-mmd-informal} (formally stated in \cref{lem:ktcompresspp-mmd}) controls the integration error of functions in $\rkhs(\kalg)$ and is near-minimax optimal by \citet[Thm.~1,~6]{tolstikhin2017minimax}. In the sequel, we leverage this type of result to derive generalization bounds for the KRR problem.

\subsection{Supervised kernel thinning}
We show how the approximation results from kernel thinning can be extended to the regression setting. We construct two meta-kernels, the Nadaraya-Watson meta-kernel $\kernelnw$ and the ridge-regression meta-kernel $\kernelrr$, which take in a \textit{base kernel} $\kernel$ (defined over $\X$ only) and return a new kernel (defined over $\X\times \Y$). When running KT, we set this new kernel as $\kalg$.

\subsubsection{Kernel-thinned Nadaraya-Watson regression (\textsc{KT-NW})} 

A tempting choice of kernel for KT-NW is the kernel $\kernel$ itself. 
That is, we can thin the input points using the kernel
\begin{align}
    \kalg((x_1,y_1),(x_2,y_2)) &\defeq \mathbf{k}(x_1,x_2) \label{eq:baseline_1}.
\end{align}

This choice is sub-optimal since it ignores any information in the response variable $y$. For our supervised learning set-up, perhaps another intuitive choice would be to use KT with
\begin{align}
    \kalg((x_1,y_1),(x_2,y_2)) &\defeq \mathbf{k}((x_1, y_1),(x_2,y_2)), \label{eq:baseline_2}
\end{align}
where $(x,y)$ denotes the concatenation of $x$ and $y$. While this helps improve performance, there remains a better option as we illustrate next.

In fact, a simple but critical observation immediately reveals a superior choice of the kernel to be used in KT for NW estimator. We can directly observe that the NW estimator is a ratio of the averages of two functions: 
\begin{align}
    f_{\mrm{numer}}(x,y)(\cdot) &\defeq \kernel(x, \cdot) \inner{y}{1}_{\R} \\ \qtext{and} f_{\mrm{denom}}(x,y)(\cdot) &\defeq \kernel(x,\cdot),
\end{align}
over the empirical distribution $\Pin$ \cref{eq:Pin-Qout}. Recall that KT provides a good approximation of sample means of functions in an RKHS, so it suffices to specify a ``correct'' choice of the RKHS (or equivalently the ``correct'' choice of the reproducing kernel).
We can verify that $f_{\mrm{denom}}$ lies in the RKHS associated with kernel $\kernel(x_1,x_2)$ and $f_{\mrm{numer}}$ lies in the RKHS associated with kernel $\kernel(x_1,x_2)\cdot y_1 y_2$. This motivates our definition for the Nadaraya-Watson kernel:
\begin{align}\label{eq:kernel-nw}
    \kernelnw((x_1,y_1),(x_2,y_2)) \defeq \kernel(x_1,x_2) + \kernel(x_1,x_2) \cdot y_1 y_2
\end{align}
since then we do have $f_{\mrm{denom}}, f_{\mrm{numer}} \in \rkhs(\kernelnw)$. Intuitively, thinning with $\kernelrr$ should simultaneously provide good approximation of averages of $f_{\mrm{denom}}$ and $f_{\mrm{numer}}$ over $\Pin$ (see the formal argument in \cref{sec:kt-nw}). When $\outputcoreset = \ktcompresspp(\inputcoreset, \kernelnw, \delta)$, we call the resulting solution to \cref{eq:nadaraya-kt} the {kernel-thinned Nadaraya-Watson (\textsc{KT-NW}) estimator}, denoted by $\nwkt$. 

As we show in \cref{fig:ablations}(a), this theoretically principled choice does provide practical benefits in MSE performance across sample sizes.

\begin{figure}[ht]
    \centering
    \begin{tabular}{cc}
        \includegraphics[width=0.45\textwidth]{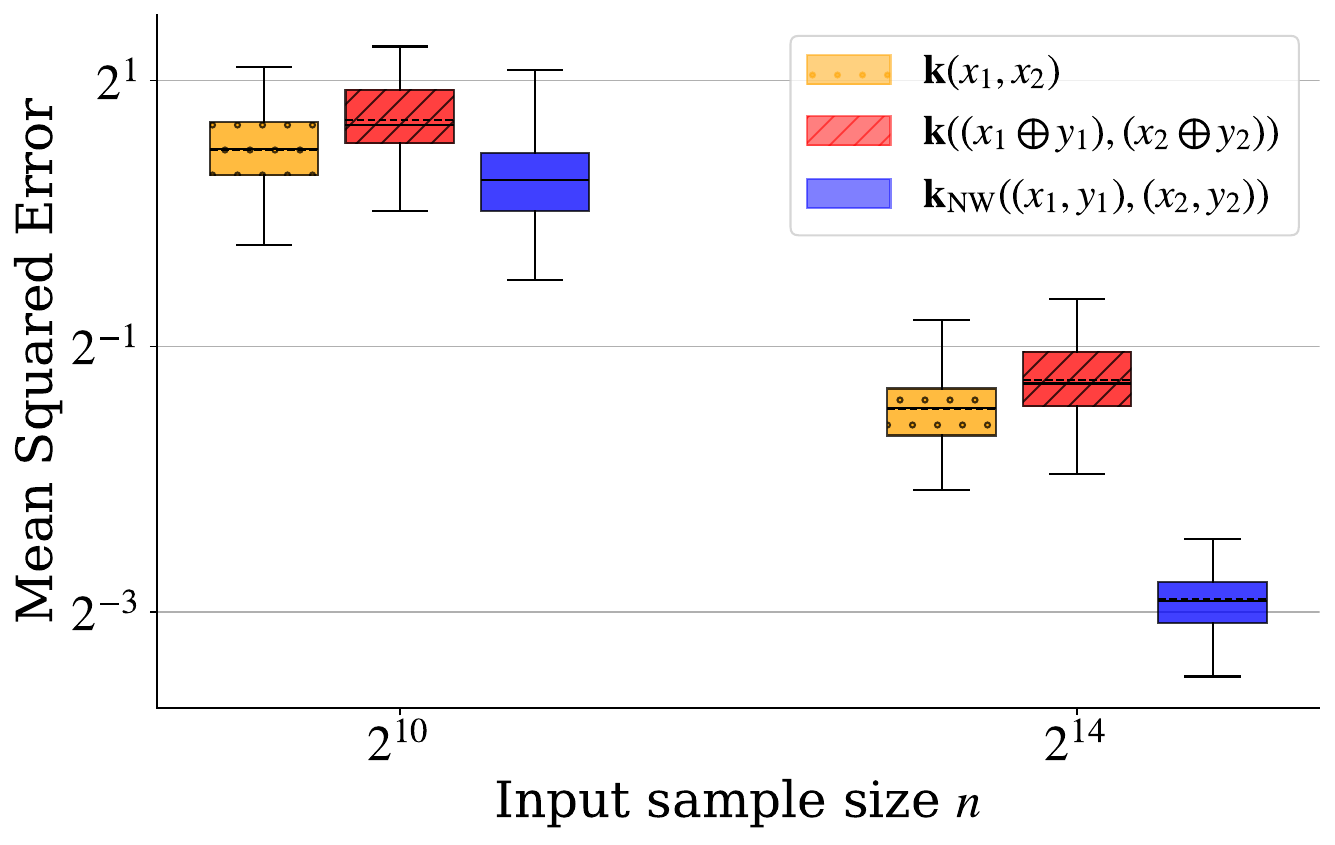} &  
        \includegraphics[width=0.45\textwidth]{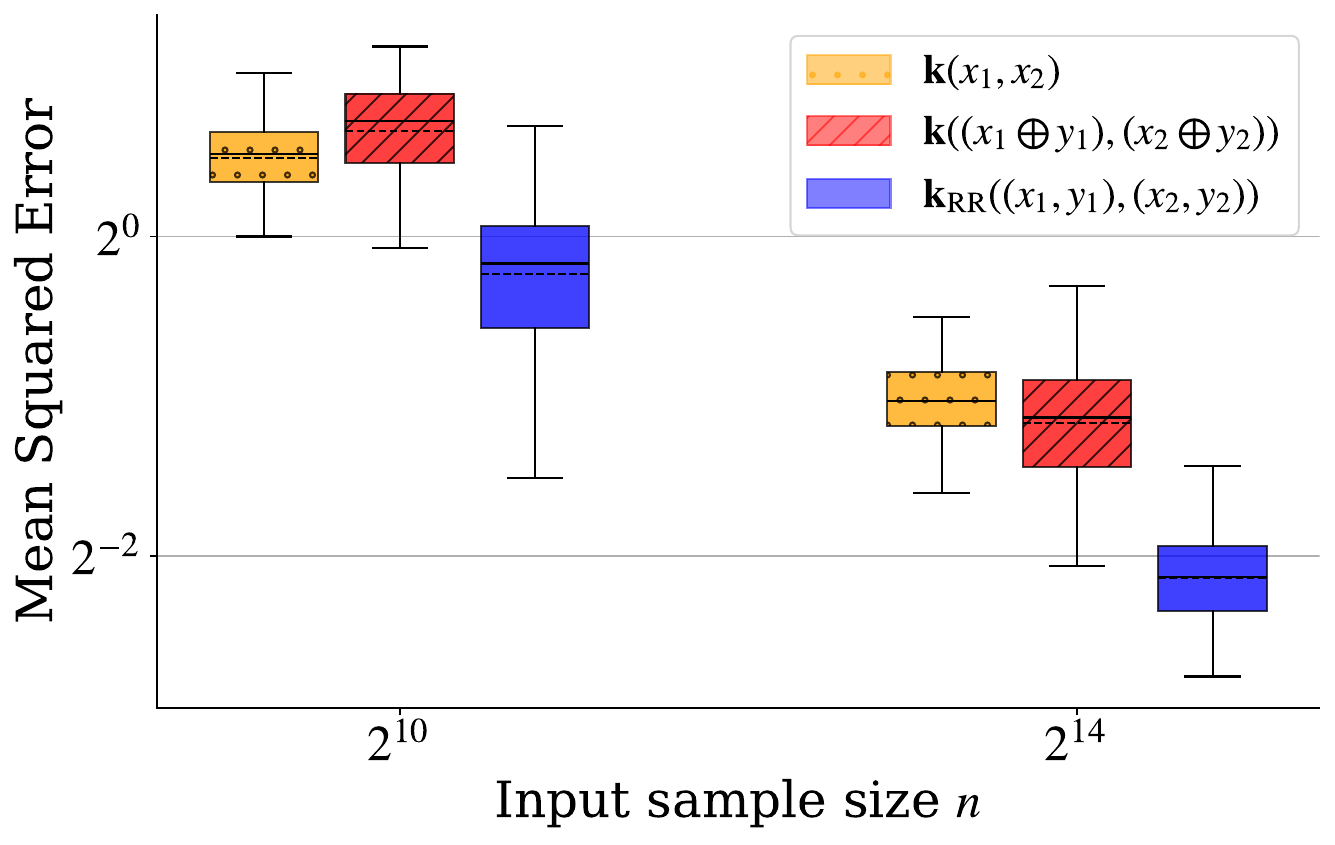} \\
        (a) NW ablation with \Wendland[0] base kernel &
        (b) KRR ablation with Gaussian base kernel
    \end{tabular}
    \caption{\tbf{MSE vs choice of kernels}. For exact settings and further discussion see \cref{sec:simulation}.
    }
    \label{fig:ablations}
\end{figure}

\subsubsection{Kernel-thinned kernel ridge regression (KT-KRR)}  
While with NW estimator, the closed-form expression was a ratio of averages, the KRR estimate~\cref{eq:krr} can not be expressed as an easy function of averages. However, notice that $L_{\inputcoreset}$ in \cref{eq:krr-objective} is an average of the function $\ell_f : \X \times \Y \to \reals$ defined as
\begin{align}
    \ell_f(x,y) \defeq f^2(x) -2 f(x) y + y^2 \qtext{for} f\in \rkhs(\kernel).
\end{align}
Thus, there may be hope of deriving a KT-powered KRR estimator by thinning $L_{\inputcoreset}$ with the appropriate kernel.
Assuming $f\in \rkhs(\kernel)$, we can verify that $f^2$ lies in the RKHS associated with kernel $\kernel^2(x_1,x_2)$ and that $-2 f(x) y$ lies in the RKHS associated with kernel $\kernel(x_1,x_2) \cdot y_1 y_2$. 
We now define the ridge regression kernel by
\begin{align}\label{eq:kernel-krr}
    \kernelrr((x_1,y_1),(x_2,y_2)) \defeq \kernel^2 (x_1,x_2) + \kernel(x_1,x_2) \cdot y_1 y_2
\end{align}
and we can verify that $f^2(x) - 2f(x) y$ lies in the RKHS $\rkhs(\kernelrr)$.\footnote{One might expect the ridge regression kernel to include a term that accounts for $y^2$. However, the generalization bounds turn out to be essentially the same regardless of whether we include this term when defining $\kernelrr$.}
When $\outputcoreset \defeq \ktcompresspp(\inputcoreset, \kernelrr, \delta)$, we call the resulting solution to \cref{eq:thinned-krr-solution} the {kernel-thinned KRR (\krrktname) estimator} with regularization parameter $\lambda'>0$, denoted $\krrkt$. 
We note that the kernel $\kernelrr$ also appears in \cite[Lem.~4]{grunewalder2022compressed}, except our subsequent analysis comes with generalization bounds for the \krrktname estimator. Like for NW, in \cref{fig:ablations}(b) we do a comparison for KRR-MSE across many kernel choices and conclude that the choice \cref{eq:kernel-krr} is indeed a superior choice compared to the base kernel $\kernel$ and the concatenated kernel~\cref{eq:baseline_2}.

\section{Main results}
We derive generalization bounds of our two proposed estimators. In particular, we bound the mean squared error (MSE) defined by $\twonorm{f-\fstar}^2 = \Esubarg{X}{(f(X) - \fstar(X))^2}$.
Our first assumption is that of a well-behaved density on the covariate space. This assumption mainly simplifies our analysis of Nadaraya-Watson and kernel ridge regression, but can in principle be relaxed.
\begin{assumption}[Compact support]\label{assum:compact-support}
Suppose that $\X \subset \ball_2(\Rin) \subset \reals^d$ for some $\Rin> 0$ and that the density $p$ satisfies $0< \pmin \leq p(x)\leq \pmax$ for all $x \in \X$.
\end{assumption}

\subsection{KT-NW}\label{sec:kt-nw}

For the analysis of the NW estimator, we define function complexity in terms of \Holder smoothness following prior work \cite{tsybakov2009nonparametric}.
\begin{definition}\label{def:L-beta-smooth}
For $L>0$ and $\beta \in (0, 1]$, a function $f:\X\to\real$ is $(\beta,L)$-\Holder if for all $x_1, x_2 \in \X$, $\abss{f(x_1)-f(x_2)} \leq L\norm{x_1-x_2}^\beta$.
\end{definition}

Our next assumption is that on the kernel. Whereas typically the NW estimator does not require a reproducing kernel, our KT-NW estimator requires that $\k$ be reproducing to allow for valid analysis. 

\begin{assumption}[Shift-invariant kernel] \label{assum:nw-kernel}
$\k$ is a reproducing kernel function (i.e., symmetric and positive semidefinite) defined by $\k(x_1, x_2) \defeq \kappa( \| x_1 - x_2\| / h)$, where $h>0$ and $\kappa:\R\to\R$ is bounded by 1, $L_{\kappa}$-Lipschitz, square-integrable, and satisfies:
\begin{align}
    \log(L_\kappa\cdot \kappa^\dagger(1/n)) = \bigO{\log n}, \qtext{where} \kappa^\dagger(u) \defeq \sup\braces{r: \kappa(r)\geq u}; \label{eq:kappa-dagger}
\end{align}
Additionally, there must exist constants $c_1,c_2>0$ such that
\begin{align}
    &2^{j\beta} \sup_{x\in \X} \int_{\snorm{z}\in [(2^{j-1}-\half)h,(2^j+\half)h]} \k(x,x-z) dz \leq c_1 \cdot \sup_{x\in \X} \int_{\snorm{z}\in [0,\half h]} \k(x,x-z) dz \qtext{and}\label{eq:nw-kernel-cond1}\\
    &(2^j + \half)^d 2^{j\beta} \kappa(2^{j-1} - 1) \leq c_2 \cdot \int_0^{\half} \kappa(u) u^{d-1} du
    \qtext{for all} j=1,2,\ldots.\label{eq:nw-kernel-cond2}
\end{align}
\end{assumption}
When $\k$ is a compact kernel, such as Wendland, Sinc, and B-spline, \cref{assum:nw-kernel} is easily satisfied. In \cref{app:nw-kernel-examples}, we prove that Gaussian and \Matern (with $\nu> d/2+1$) kernels also satisfy \cref{assum:nw-kernel}.
We now present our main result for the \textsc{KT-NW} estimator. \proofref{proof:nadaraya-kt}

\begin{theorem}[\textsc{KT-NW}] \label{thm:nadaraya-kt}
Suppose that \cref{assum:compact-support,assum:nw-kernel} hold and that $\fstar \in \Sigma(\beta,L_f)$ with $\beta\in (0,1]$. Then for any fixed $\delta \in (0, 1]$, the \textsc{KT-NW} estimator \cref{eq:nadaraya-kt} with $\nout=\sqrt n$ and bandwidth $h=n^{-\frac{1}{2\beta+2 d}}$ satisfies
\begin{talign}\label{eq:nw-kt}
    \statictwonorm{\nwkt - \fstar}^2  \leq C n^{ -\frac{\beta}{\beta + d}} \log^2 n,
\end{talign}
with probability at least $1-\delta$, for some positive constant $C$ that does not depend on $n$.
\end{theorem}

\citet{tsybakov2009nonparametric,belkin2019does} show that \textsc{Full-NW} achieves a rate of 
$\bigO{n^{ -\frac{2 \beta}{2\beta + d}}}$, which is minimax optimal for the $(\beta,L)$-\Holder function class.
Compared to the \textsc{ST-NW} rate of $n^{-\frac{\beta}{2\beta + d}}$, \nwktname achieves strictly better rates for all $\beta>0$ and $d>0$, while retaining \textsc{ST-NW}'s fast query time of $\bigO{\sqrt n}$. Note that our method \nwktname has a training time of $\bigO{n \log^3 n}$, which is not much more than simply storing the input points.

\subsection{KT-KRR}
\label{sec:kt-krr}

We present our main result for \krrktname using finite-rank kernels. This class of RKHS includes linear functions and polynomial function classes.

\begin{theorem}[KT-KRR for finite-dimensional RKHS] \label{thm:kt-krr-finite}
Assume $\fstar\in \rkhs(\kernel)$, \cref{assum:compact-support} is satisfied, and $\kernel$ has rank $m\in \naturals$. 
Let $\krrkt$ denote the \krrktname estimator with regularization parameter $\lambda' = \bigO{\frac{m \log \nout}{n\wedge \nout^2}}$.
Then with probability at least $1-2\delta - 2 e^{-\frac{\knorm{\fstar}^2}{c_1 (\knorm{\fstar}^2+\sigma^2) }}$, the following holds:
\begin{align}\label{eq:krr-kt-L2-bound-finite-simplified}
    \statictwonorm{\krrkt - \fstar}^2 &\leq  \frac{C m \cdot \log \nout}{\min(n, \nout^2)} \brackets{\knorm{\fstar} + 1}^2
\end{align}
for some constant $C$ that does not depend on $n$ or $\nout$.
\end{theorem}
\proofref{proof:kt-krr-finite}
Under the same assumptions, \citet[Ex.~13.19]{wainwright2019high} showed that the Full-KRR estimator $\krr$ achieves the minimax optimal rate of $O(m/n)$ in  $O(n^3)$ runtime. When $\nout = \sqrt{n} \log ^c n$, the KT-KRR error rates from  \cref{thm:kt-krr-finite} match this minimax rate in $\wtil O(n^{3/2})$ time, a (near)  quadratic improvement over the Full-KRR. On the other hand, standard thinning-KRR with similar-sized output achieves a quadratically poor MSE of order $\frac{m}{\sqrt n}$.

Our method and theory also extend to the setting of infinte-dimensional kernels. To formalize this, we first introduce the notion of kernel covering number.

\begin{definition}[Covering number]\label{def:cover}
For a kernel $\kernel:\mc Z\times \mc Z \to \R$ with $\ball_{\kernel} \defeq \braces{f\in \rkhs:\hnorm{f}\leq 1}$, a set $\mc A \subset \mc Z$ and $\eps>0$, the covering number $\coveringnumber_{\kernel}(\mc A, \eps)$ is the minimum cardinality of all sets $\mc C\subset \ball_{\kernel}$ satisfying $\mc \ball_{\kernel} \subset \bigcup_{h\in \mc C} \braces{ g\in\mc \ball_{\kernel}:\sup_{x\in \mc A} \abss{h(x)-g(x)}\leq \eps }$.
\end{definition}

We consider two general classes of kernels.
\begin{assumption}\label{assum:alpha-beta-kernel}
For some $\mathfrak C_d>0$, all $r>0$ and $\eps\in (0,1)$, and $\ball_2(r)=\braces{x\in \R^d: \twonorm{x}\leq r}$, a kernel $\kernel$ is 
\begin{talign}
    & \textsc{LogGrowth}(\alpha,\beta) \qtext{when} \log \coveringnumber_{\kernel}(\ball_2(r),\eps)\leq \mfk C_d \log(e/\eps)^\alpha (r+1)^\beta\qtext{with} \alpha,\beta>0\qtext{and}\\
    &\textsc{PolyGrowth}(\alpha,\beta) \qtext{when} \log \coveringnumber_{\kernel}(\ball_2(r),\eps)\leq \mfk C_d(1/\eps)^\alpha (r+1)^\beta \qtext{with} ~\alpha < 2.
\end{talign}
\end{assumption}
We highlight that the definitions above cover several popular kernels: \textsc{LogGrowth} kernels include finite-rank kernels and analytic kernels, like Gaussian, inverse multiquadratic (IMQ), and sinc \citep[Prop.~2]{dwivedi2022generalized}, while \textsc{PolyGrowth} kernels includes finitely-many continuously differentiable kernels, like \Matern and B-spline \citep[Prop.~3]{dwivedi2022generalized}. For clarity, here we present our guarantee for \textsc{LogGrowth} kernels and defer the other case to \cref{app:proof-krr-explicit-rates}.

\begin{theorem}[\krrktname guarantee for infinite-dimensional RKHS] \label{cor:krr-kt}
Suppose \cref{assum:compact-support} is satisfied and $\kernel$ is $\textsc{LogGrowth}(\alpha,\beta)$ (\cref{assum:alpha-beta-kernel}).
Then $\krrkt$ with $\lambda'=\bigO{1/\nout}$ satisfies the following bound with probability at least $1-2\delta - 2 e^{-\frac{\knorm{\fstar}^2 \log^\alpha n}{c_1 (\knorm{\fstar}^2+\sigma^2) }}$:
\begin{talign}\label{eq:krr-kt-log}
    \statictwonorm{\krrkt - \fstar}^2 \leq C
    \parenth{\frac{\log^{\alpha}n}{n} + \frac{\sqrt{\log^{\alpha}\nout}}{\nout}}
    \cdot \brackets{\knorm{\fstar} + 1}^2.
\end{talign}
for some constant $C$ that does not depend on $n$ or $\nout$.
\end{theorem}
\proofref{app:proof-krr-explicit-rates}
When $\nout = \sqrt n$, \krrstname achieves an excess risk rate of $n^{-1/2} \log^\alpha n$ for $\kernel$ satisfying $\textsc{LogGrowth}(\alpha,\beta)$. While \krrktname does not achieve a strictly better excess risk rate bound over \krrstname, we see that in practice, \krrktname still obtains an empirical advantage. Obtaining a sharper error rate for the infinite-dimensional kernel setting is an exciting venue for future work.

\section{Experimental results}\label{sec:results}

We now present experiments on simulated and real-world data. On real-world data, we compare our \krrktname estimator with several state-of-the-art KRR methods, including \Nystrom subsampling-based methods and KRR pre-conditioning methods.
All our experiments were run on a machine with 8 CPU cores and 100 GB RAM. 
Our code can be found at \url{https://github.com/ag2435/npr}. 

\subsection{Simulation studies}
\label{sec:simulation}

\begin{wrapfigure}{R}{0.3\textwidth}
  \begin{center}
    \includegraphics[width=0.3\textwidth]{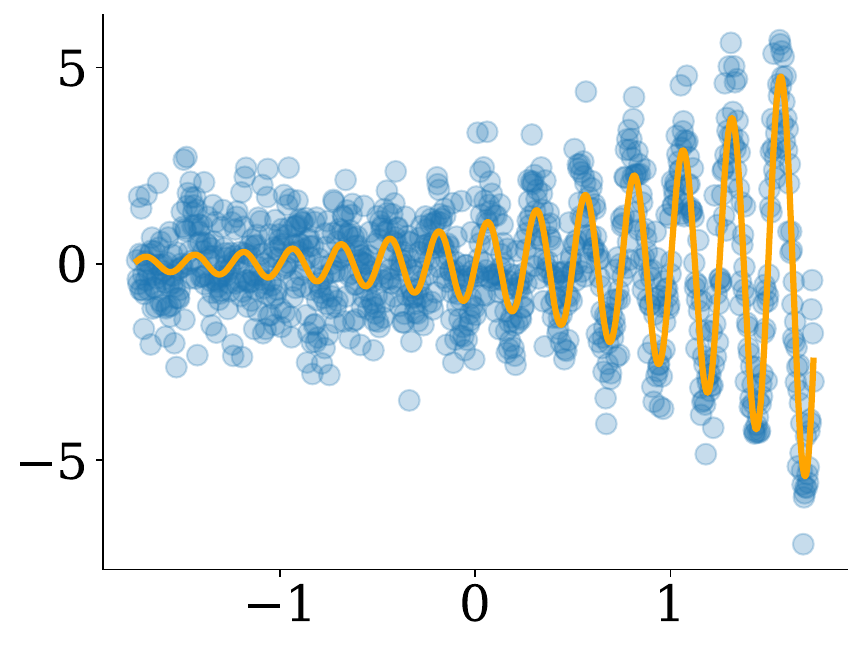}
  \end{center}
  \caption{\tbf{Simulated data.}}
  \label{fig:sinexp}
\end{wrapfigure}

We begin with some simulation experiments. For simplicity, let $\X = \R$ and $\P = \text{Unif}[-\sqrt{3}, \sqrt 3]$ so that $\Var[X]=1$. We set
\begin{talign}\label{eq:fstar}
    \fstar(x) = 8\sin(8\pi x) \exp(x) \qtext{and} \sigma=1
\end{talign}
and follow \cref{eq:process} to generate $\parenth{y_i}_{i=1}^n$ (see \cref{fig:sinexp}). We let the input sample size $n$ vary between $2^{8},2^{10},2^{12},2^{14}$ and set the output coreset size to be $\nout = \sqrt n$ in all cases. For NW, we use the \Wendland[0] kernel defined by 
\begin{talign}\label{eq:wendland-kernel}
    \kernel(x_1,x_2)\defeq (1-\frac{\twonorm{x_1-x_2}}{h})_{+}\qtext{for} h>0. 
\end{talign}
For KRR, we use the Gaussian kernel defined by
\begin{talign}\label{eq:gaussian-kernel}
    \kernel(x_1,x_2) \defeq \exp(-\frac{\twonorm{x_1-x_2}^2}{2 h^2}) \qtext{for} h>0.
\end{talign}
We select the bandwidth $h$ and regularization parameter $\lambdakt$ (for KRR) using grid search. Specifically, we use a held-out validation set of size $10^4$ and run each parameter configuration $100$ times to estimate the validation MSE since \krrktname and \krrstname are random.

\paragraph{Ablation study.}
In \cref{fig:ablations}, we compare thinning with our proposed meta-kernel $\kalg=\kernelnw$ to thinning with the baseline meta-kernels \cref{eq:baseline_1} and \cref{eq:baseline_2}. For our particular regression function \cref{eq:fstar}, thinning with \cref{eq:baseline_1} outperforms thinning with \cref{eq:baseline_2}. We hypothesize that the latter kernel is not robust to the scaling of the response variables. By inspecting \cref{eq:wendland-kernel}, we see that $\twonorm{(x_1, y_1) - (x_2, y_2)}$ is heavily determined by the $y_i$ values when they are large compared to the values of $x_i$---as is the case on the right side of \cref{fig:sinexp} (when $X>0$). Since $\P$ is a uniform distribution, thinning with \cref{eq:baseline_1} evenly subsamples points along the input domain $\X$, even though accuractely learning the left side of \cref{fig:sinexp} (when $X<0$) is not needed for effective prediction since it is primarily noise. 
Validating our theory from \cref{thm:nadaraya-kt}, the best performance is obtained when thinning with $\kernelnw$ \cref{eq:kernel-nw}, which avoids evenly subsampling points along the input domain and correctly exploits the dependence between $X$ and $Y$.

In \cref{fig:ablations}, we perform a similar ablation for KRR. Again we observe that thinning with $\kalg((x_1,y_1),(x_2,y_2)) = \kernel(x_1,x_2)$ outperforms thinning with $\kalg((x_1,y_1),(x_2,y_2)) = \kernel((x_1,y_1), (x_2,y_2))$, while thinning with $\kalg = \kernelrr$ achieves the best performance.

\paragraph{Comparison with \textsc{Full}, \textsc{ST}, \textsc{RPCholesky}.}
In \cref{fig:mse_runtime}(a), we compare the MSE of \nwktname to \nwfullname, \nwstname (a.k.a ``Subsample''), and \textsc{RPCholesky-NW} across four values of $n$.
This last method uses the pivot points from \textsc{RPCholesky} as the output coreset $\outputcoreset$.
At all $n$ we evaluated, \textsc{KT-NW} achieves lower MSE than \textsc{ST-NW} and \textsc{RPCholesky-NW}. \textsc{Full-NW} achieves the lowest MSE across the board, but it suffers from significantly worse run times, especially at test time. Owing to its $\bigO{n \log^3 n}$ runtime, \textsc{KT-NW} is significantly faster than \textsc{RPCholesky-NW} for training and nearly matches \textsc{ST-NW} in both training and testing time.
We hypothesize that \textsc{RPCholesky}---while it provides a good low-rank approximation of the kernel matrix---is not designed to preserve averages.

In \cref{fig:mse_runtime}(b), we compare the MSE of \krrktname to \krrfullname, \krrstname (a.k.a ``Subsample''), and the \textsc{RPCholesky-KRR} method from \citet[Sec.~4.2.2]{chen2022randomly}, which uses \textsc{RPCholesky} to select landmark points for the restricted KRR problem. We observe that \krrktname achieves lower MSE than \krrstname, but higher MSE than  \textsc{RPCholesky-KRR} and \krrfullname. In \cref{fig:mse_runtime}(b), we also observe that \krrktname is orders of magnitude faster than \krrfullname across the board, with runtime comparable to \krrstname and \textsc{RPCholesky-KRR} in both training and testing. 

\begin{figure*}[t]
    \centering
    \begin{tabular}{cc}
        \includegraphics[width=0.45\textwidth]{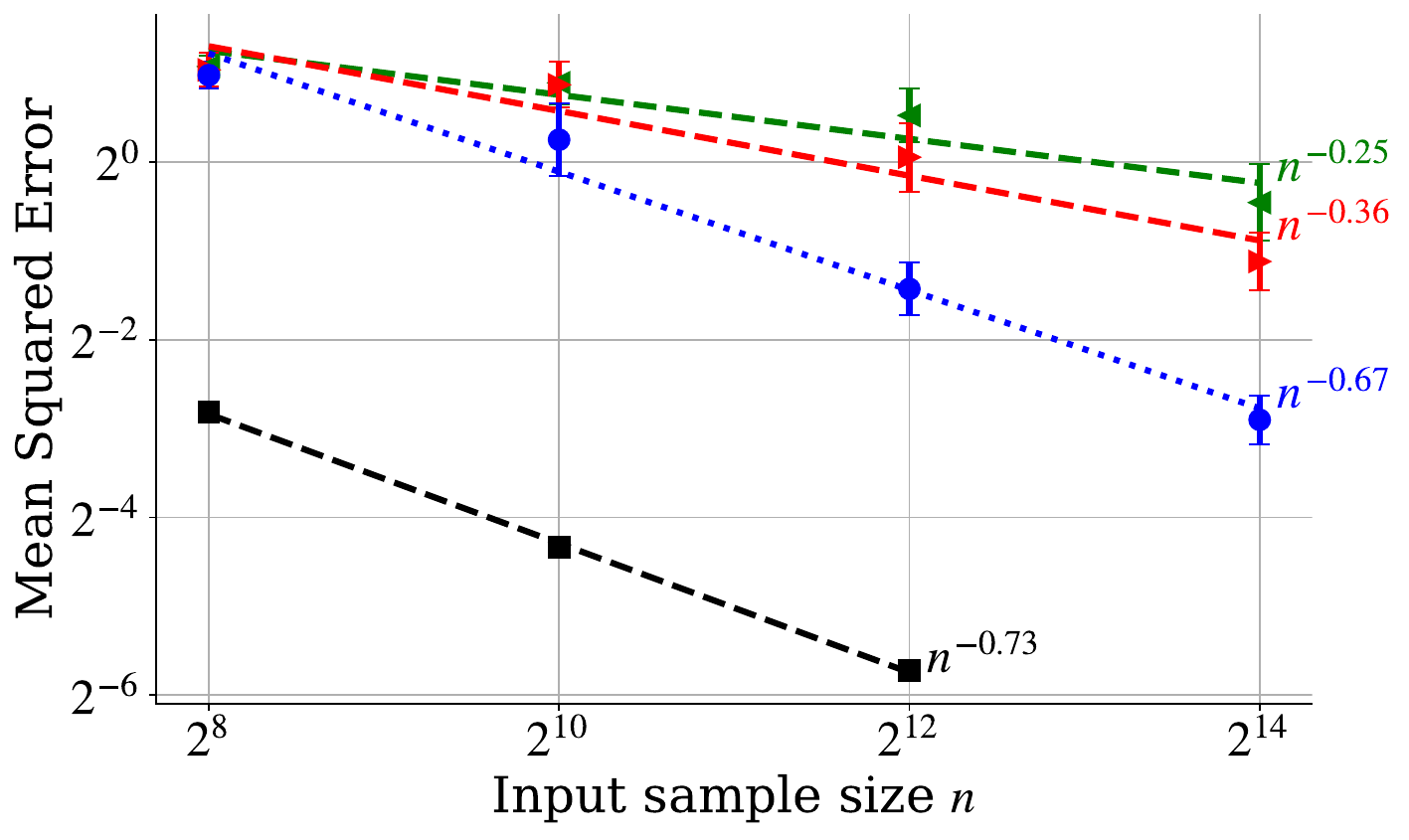} &
        \includegraphics[width=0.45\textwidth]{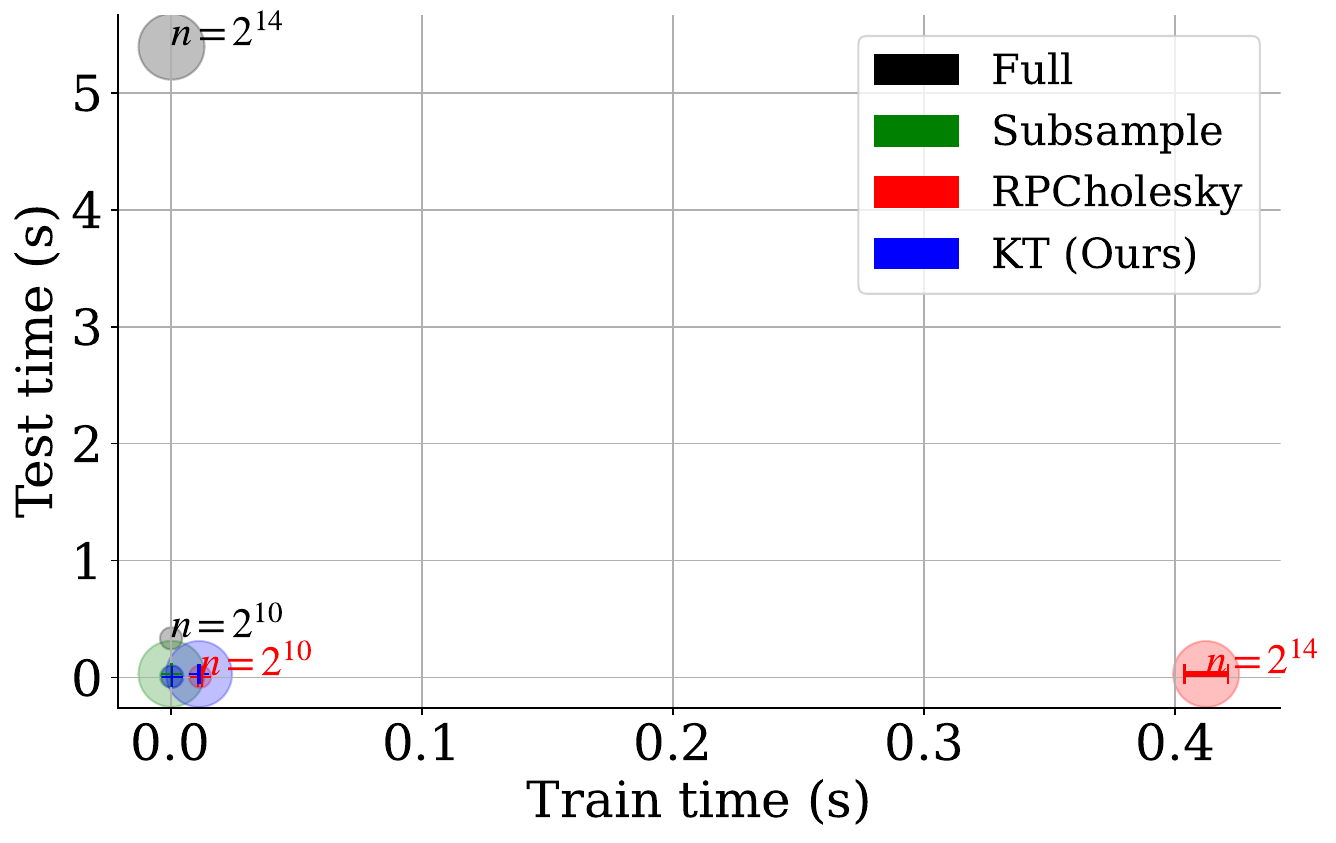}\\[-1mm]
    \end{tabular}
    \\[2mm]
    {(a) Nadaraya-Watson estimator with \Wendland[0] kernel \cref{eq:wendland-kernel}.}\\
    \rule{\textwidth}{0.5pt}\vspace{1em}
    \begin{tabular}{cc}
        \includegraphics[width=0.45\textwidth] {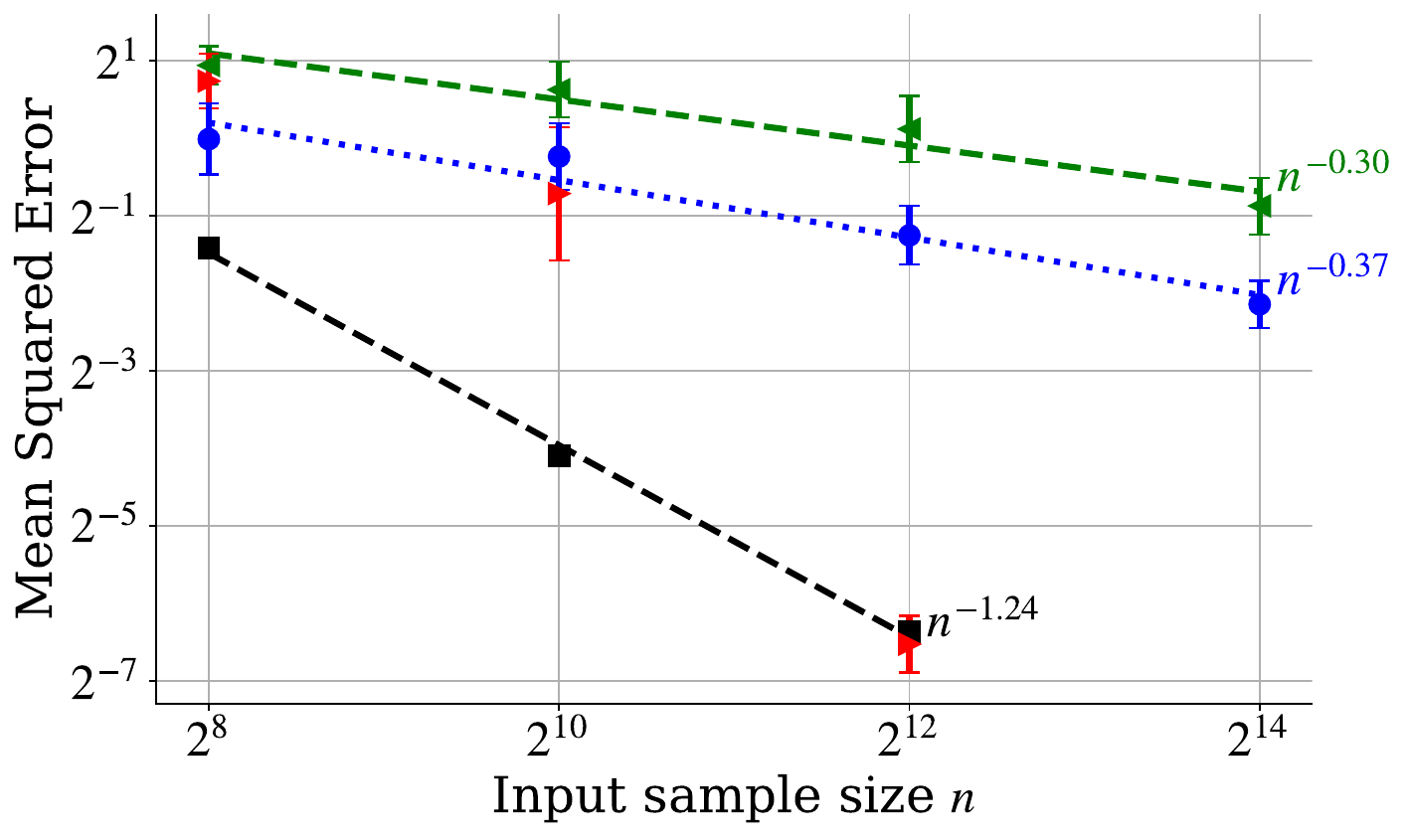} &
        \includegraphics[width=0.45\textwidth]{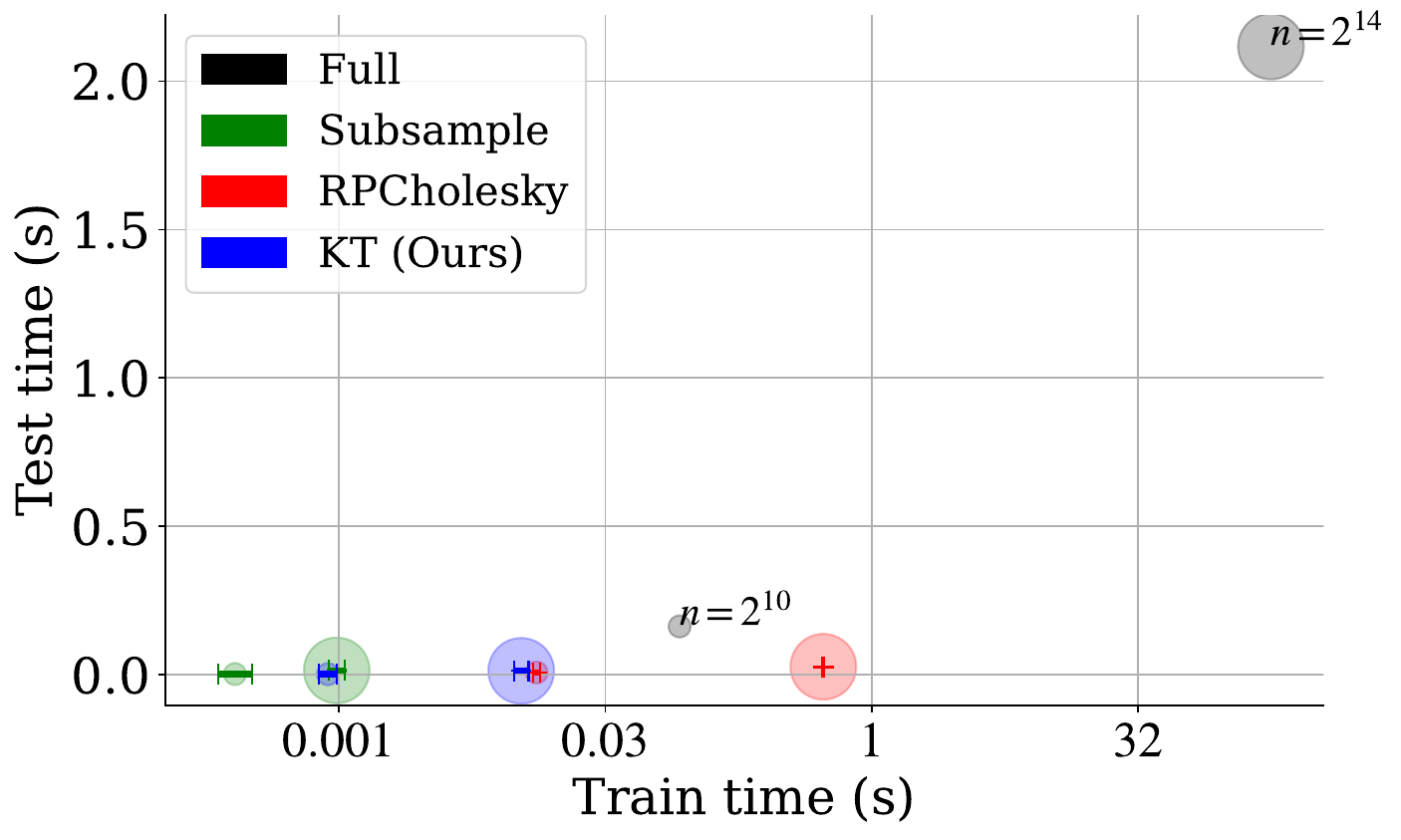} \\
    \end{tabular}
    \\[2mm]
    {(b) Kernel ridge regression estimator with Gaussian kernel \cref{eq:gaussian-kernel}.}
    
    \caption{\tbf{MSE and runtime comparison on simulated data.} Each point plots the mean and standard deviation across 100 trials (after parameter grid search).}
    \label{fig:mse_runtime}
\end{figure*}

\subsection{Real data experiments}\label{sec:real-data}

We now present experiments on real-world data using two popular datasets: the California Housing regression dataset from \citet{pace1997sparse} (\url{https://scikit-learn.org/1.5/datasets/real_world.html\#california-housing-dataset}; BSD-3-Clause license) and the SUSY binary classification dataset from \citet{baldi2014searching} (\url{https://archive.ics.uci.edu/dataset/279/susy}; CC-BY-4.0 license).

\paragraph{California Housing dataset ($d=8,N=2\times 10^4$).} 
\cref{tab:real}(a) compares the test MSE, train times, and test times.
We normalize the input features by subtracting the mean and dividing by the standard deviation and use a 80-20 train-test split. For all methods, we use the Gaussian kernel \cref{eq:gaussian-kernel} with bandwidth $h=10$. We use $\lambda=\lambdakt=10^{-3}$ for \krrfullname, \krrstname, and \krrktname and $\lambda=10^{-5}$ for \textsc{RPCholesky-KRR}. 
On this dataset, \krrktname lies between ST-KRR and \textsc{RPCholesky-KRR} in terms of test MSE. When $\nout=\sqrt n$, \textsc{RPCholesky} pivot selection takes $\bigO{n^2}$ time by \citet[Alg.~2]{chen2022randomly}, compared to \ktcompresspp, which compresses the input points in only $\bigO{n \log^3 n}$ time. This difference in big-O runtime is reflected in our empirical results, where we see \krrktname take 0.0153s versus 0.3237s for \textsc{RPCholesky-KRR}.

\paragraph{SUSY dataset ($d=18, N=5\times 10^6$).} 
\cref{tab:real}(b) compares our proposed method \krrktname (with $h=10,\lambdakt=10^{-1}$) to several large-scale kernel methods, namely RPCholesky preconditioning \citep{diaz2023robust}, FALKON \citep{rudi2017falkon}, and Conjugate Gradient (all with $h=10,\lambda=10^{-3}$) in terms of test classification error and training times. For the baseline methods, we use the Matlab implementation provided by \citet{diaz2023robust} (\url{https://github.com/eepperly/Robust-randomized-preconditioning-for-kernel-ridge-regression}).
In our experiment, we use $4\times 10^6$ points for training and the remaining $10^6$ points for testing. As is common practice for classification tasks, we use the Laplace kernel defined by  $ \kernel(x_1,x_2) \defeq \exp(-{\twonorm{x_1-x_2}}/{h}).$
All parameters are chosen with cross-validation.

We observe that \krrktname achieves test MSE between \krrstname and \textsc{RPCholesky} preconditioning with training time almost half that of \textsc{RPCholesky} preconditioning. Notably, our Cython implementation of \ktcompresspp thinned the four million training samples in only 1.7 seconds on a single CPU core---with further speed-ups to be gained from parallelizing on a GPU in the future.

\begin{table*}[t]
    \centering
    \begin{tabular}{c ccc}
        Method  &  MSE (\%) & Training time (s) & Prediction time (s) \\
        \hline
        \hline
        Full & $0.4137$ & $11.1095$ & $0.7024$\\
        \hline
        \krrstname & $0.5736 \pm 0.0018$ & $0.0018 \pm 0.0005$ & $0.0092 \pm 0.0006$ \\
        \hline
        \textsc{RPCholesky} & $0.3503 \pm 0.0001$ & $0.3237 \pm 0.0094$ & $0.0060 \pm 0.0008$ \\
        \hline
        \krrktname (Ours) & $0.5580 \pm 0.0015$ & $0.0153 \pm 0.0013$ & $0.0083 \pm 0.0003$ \\
        \hline 
    \end{tabular}
    \\[2mm]
    {(a) California Housing regression dataset.}\\
    \rule{\textwidth}{0.5pt}\vspace{1em}
    \begin{tabular}{c ccc}
        Method  &  Test Error (\%) & Training Time (s) &  \\
        \hline
        \hline
        RPCholesky & $19.99 \pm 0.00$ & $3.46 \pm 0.03$ & \\
        \hline
        FALKON & $19.99 \pm 0.00$ & $5.06 \pm 0.02$ & \\
        \hline
        CG & $20.35 \pm 0.00$ & $6.16 \pm 0.03$ &  \\
        \hline
        \krrstname & $22.71 \pm 0.30$ & $0.09 \pm 0.00$ & \\
        \hline
        \krrktname (Ours) & $22.00 \pm 0.21$ & $1.79 \pm 0.00$ \\
        \hline
    \end{tabular}
    \\[2mm]
    {(b) SUSY dataset.}
    
    \caption{\tbf{Accuracy and runtime comparison on real-world data.} Each cell represents mean ± standard error across 100 trials.}
    \label{tab:real}
\end{table*}

\section{Conclusions}\label{sec:conclusions}

In this work, we introduce a meta-algorithm for speeding up two estimators from non-parametric regression, namely the Nadaraya-Watson and Kernel Ridge Regression estimators. Our method inherits the favorable computational efficiency of the underlying Kernel Thinning algorithm and stands to benefit from further advancements in unsupervised learning compression methods. 

The KT guarantees provided in this work apply only when $\fstar\in \rkhs(\kernel)$ for some base kernel $\kernel$.
In practice, choosing a good kernel $\mathbf{k}$ is indeed a challenge common to all prior work.
Our framework is friendly to recent developments in kernel selection to handle this problem:
\citet[Cor.~1]{dwivedi2022generalized} provide integration-error guarantees for KT when $f^\star \notin \mathcal{H}(\mathbf{k})$.
Moreover, there are recent results on finding the best kernel (e.g., for hypothesis testing \cite[Sec.~4.2]{domingo2023compress}).
\citet{radhakrishnan2022feature} introduce the Recursive Feature Machine, which uses a parameterized kernel $\mathbf{k} _{M}(x_1,x_2) \defeq \exp(-{(x_1-x_2)^\top M (x_1-x_2)}/{(2h^2)})$, and propose an efficient method to learn the matrix parameter $M$ via the average gradient outer product estimator. An exciting future direction would be to combine these parameterized (or "learned") kernels with our proposed KT methods for non-parametric regression.

\section{Acknowledgements}

AG is supported with funding from the NewYork-Presbyterian Hospital.

\bibliographystyle{plainnat} %
\bibliography{refs}

\newpage\clearpage
\appendix
\section{Background on \ktcompresspp}
\label{app:ktcompresspp}

This section details the \ktcompresspp algorithm of \citet[Ex.~6]{shetty2022distribution}.
In a nutshell, \ktcompress (\cref{algo:ktcompress}) takes as input a point sequence of size $n$, a compression level $\ossymb$, a (reproducing) kernel functions $\kalg$, and a failure probability $\delta$. 
\ktcompresspp first runs the $\ktcompress(\ossymb)$ algorithm of \citet[Ex.~4]{shetty2022distribution} to produce an intermediate coreset of size $2^{\ossymb} \sqrt n$. Next, the KT algorithm is run on the intermediate coreset to produce a final output of size $\sqrt n$.

\ktcompress proceeds by calling the recursive procedure \compress, which uses KT with kernels $\kalg$ as an intermediate halving algorithm. 
The KT algorithm itself consists of two subroutines: (1) \ktsplit (\cref{algo:ktsplit}), which splits a given input point sequence into two equal halves with small approximation error in the $\kalg$ reproducing kernel Hilbert space and (2) \ktswap (\cref{algo:ktswap}), which selects the best approximation amongst the \ktsplit coresets and a baseline coreset (that simply selects every other point in the sequence) and then iteratively refines the selected coreset by swapping out each element in turn for the non-coreset point that most improves $\mmd_{\kalg}$ error.
As in \citet[Rem.~3]{shetty2022distribution}, we symmetrize the output of KT by returning either the KT coreset or its complement with equal probability.

Following \citet[Ex.~6]{shetty2022distribution}, we always default to $\ossymb = \ceil{\log_2 \log n + 3.1}$ so that \ktcompresspp has an overall runtime of $\bigO{n \log^3 n}$. For the sake of simplicity, we drop any dependence on $\ossymb$ in the main paper.

\begin{algorithm2e}[ht!]
\caption{{\ktcompresspp\ --\ } Identify coreset of size $\sqrt n$}
\label{alg:compresspp}
\SetAlgoLined
  \DontPrintSemicolon
\small
  \KwIn{point sequence $ \inputcoreset $ of size $n$, \osname $\ossymb$, kernel $\kalg$, failure probability $\delta$ } 
  $\cset_{\compresssub} \ \quad\gets\quad \ktcompress(\ossymb,  \inputcoreset, \delta) $  \qquad\quad // coreset of size $2^{\ossymb}\sqrt{n}$ \\
  $\cset_{\compressppsub}  \ \ \gets \quad \kt(\cset_{\compresssub},\kalg, \frac{\ossymb}{\ossymb + 2^{\ossymb}(\rounds+1)} \delta)$ \ \ \quad\ \  // coreset of size $\sqrt{n}$ \\
  \KwRet{ $\cset_{\compressppsub}$}
\end{algorithm2e}

\begin{algorithm2e}[h!]
\SetKwFunction{proc}{\textnormal{\textsc{\compress}}}
\SetKwFunction{proctwo}{\textnormal{KT}}
\caption{{\ktcompress\ --\ } Identify coreset of size $2^\ossymb\sqrt{n}$}
\label{algo:ktcompress}
\SetAlgoLined
    \DontPrintSemicolon
    \small{
      \KwIn{point sequence $ \inputcoreset$ of size $n$, \osname~$\ossymb$, kernel $\kalg$, failure probability $\delta$} 
      \BlankLine
        \KwRet{ \textup{  $\compress(\inputcoreset, \ossymb, \kalg, \frac{\delta}{n 4^{\ossymb+1} (\log_4 n - \ossymb)})$}}
            \\
        \hrulefill\\
        \SetKwProg{myproc}{function}{}{}
         \myproc{\proc{$\cset,\ossymb,\kalg,\delta$}:}{
         \lIf{  $|\cset| = 4^{\ossymb}$ }{
            \KwRet{ $\cset $ }{}
        }  
        Partition $\cset$ into four arbitrary subsequences $ \{ \cset_i \}_{i=1}^4 $ each of size $n/4$ \\ 
        \For{$i=1, 2, 3, 4$}
        {
            $ \wtil{ \cset_i } \leftarrow \compress( \cset_i, \ossymb, \kalg,  \delta)$   // run \compress recursively to return coresets of size $2^{\ossymb} \cdot \sqrt{\frac{|\cset|}{4}}$
        } 
        $ \wtil{\cset} \gets\textsc{Concatenate}( \wtil{\cset}_1, \wtil{\cset}_2,\wtil{\cset}_3, \wtil{\cset}_4)$ \quad  // combine the coresets to obtain a coreset of size   $2 \cdot 2^{\ossymb} \cdot \sqrt{|\cset|}$ \\
        }\KwRet{ \textup{  $\kt(\wtil{\cset}, \kalg, |\wtil{\cset}|^2\delta )$ 
        \qquad \ 
        // halve the coreset to size $2^{\ossymb} \sqrt{|\cset|}$}} via symmetrized kernel thinning\;
    \hrulefill\\
\myproc{\proctwo{$\cset,\kalg,\delta$}:}{
    // Identify kernel thinning coreset containing  $\floor{|\cset|/2}$ input points\\
    $\ktcoreset \gets \ktswap(\kalg,\ktsplit(\kalg,\cset, \delta))$
    }
    \KwRet{\textup{$\ktcoreset$ with probability $\half$ and the complementary coreset $\cset\setminus\ktcoreset$ otherwise}}\;
}
\end{algorithm2e}
\setcounter{algocf}{2}
\renewcommand{\thealgocf}{\arabic{algocf}a}
\begin{algorithm2e}[h!]
  \SetKwFunction{proctwo}{\texttt{get\_swap\_params}}
\caption{{\large\textsc{kt-split}\ --\ } Divide points into candidate coresets of size $\floor{n/2}$} 
  \label{algo:ktsplit}
  \SetAlgoLined\DontPrintSemicolon
  \small
  {
  \KwIn{kernel $\kersplit$, point sequence $\inputcoreset = (x_i)_{i = 1}^n$, failure probability $\delta$}
\BlankLine
  {$\coreset[1], \coreset[2] \gets \braces{}$}\quad // Initialize empty coresets: $\coreset[1],\coreset[2]$ have size $i$ after round $i$ \\ 
  {$\sgparam[] \gets 0$}\qquad\quad\ \ \ \quad // Initialize swapping parameter \\
  \For{$i=1, 2, \ldots, \floor{n/2}$}
    {%
    // Consider two points at a time \\
    $(\x, \x') \gets (\x_{2i-1}, \x_{2i})$ \\
	 \BlankLine
     // Compute swapping threshold $\cnew$ \\ 
     $ \cnew, \sigma \gets $\proctwo{$\sigma, \vmax[], \frac{\delta}{n}$}
     with\ $\vmax[]^2 
      \!=\! \kersplit(\x,\x)\!+\!\kersplit(\x',\x')\!-\!2\kersplit(x,x')$
    \BlankLine
    // Assign one point to each coreset after probabilistic swapping \\
    $\theta\gets  \sum_{j=1}^{2i-2}(\kersplit
	 (\x_j, \x)-\kersplit(\x_j,\x')) 
	 - 2\sum_{\z\in\coreset[1]}(\kersplit(\z, \x)-\kersplit(\z,\x'))$ \\[2pt]
    $(x, x') \gets (x', x)$ \textit{ with probability }
		    $\min(1, \half (1-\frac{\theta}{\cnew})_+)$\\[2pt]
     $\coreset[1]\texttt{.append}(\x); 
		        \quad \coreset[2]\texttt{.append}(\x')$
    }
    \KwRet{$(\coreset[1],\coreset[2])$\textup{, candidate coresets of size $\floor{n/2}$}}\\
    \hrulefill\\
    \SetKwProg{myproc}{function}{}{}
     \myproc{\proctwo{$\sigma, \vmax[], \delta$}:}{
     $
			\cnew 
			    \gets \max(\vmax[] \sigma\sqrt{\smash[b]{2\log(2/\delta)}}, \vmax[]^2)$ \\
     $\sigma^2 \gets \sigma^2
	        \!+\! \vmax[]^2(1 \!+\! ({\vmax[]^2}{}\! - \!2\cnew){\sigma^2}{/\cnew^2})_+$\\
     }
     \KwRet{$(\cnew, \sigma)$}\;
  }
\end{algorithm2e} 

\setcounter{algocf}{2}
\renewcommand{\thealgocf}{\arabic{algocf}b}
\begin{algorithm2e}[h!]
\caption{{\large\textsc{kt-swap}\ --\ } Identify and refine the best candidate coreset} 
  \label{algo:ktswap}
 \SetAlgoLined\DontPrintSemicolon
\small
{
    \KwIn{kernel $\kalg$, point sequence $\inputcoreset = (x_i)_{i = 1}^n$, candidate coresets $(\coreset[1],\coreset[2])$}
        \BlankLine
    $\coreset[0] 
	    \!\gets\! \texttt{baseline\_coreset}(\inputcoreset, \texttt{size}\!=\!\floor{n/2})
	 $ \qquad \quad \ // Compare to baseline (e.g., standard thinning)
	 \BlankLine
	 $\ktcoreset \!\gets\! \coreset[ \ell^\star]
	 \text{ for }
	 \ell^\star 
	    \!\gets\! \argmin_{\ell \in \braces{0, 1, 2}} \mmd_{\kalg}(\inputcoreset, \coreset[\ell])$ 
	    \ \ // \textup{Select best coreset} \\
	    \BlankLine
	    // Swap out each point in $\ktcoreset$ for best alternative in $\inputcoreset$ while ensuring no point is repeated in $\ktcoreset$
     \\[1pt]
	   \For{$i=1, \ldots, \floor{n/2}$}{
	   \BlankLine
	    $\ktcoreset[i] \gets \argmin_{z\in \sbraces{\ktcoreset[i]} \cup (\inputcoreset\backslash \ktcoreset) }\mmd_{\kalg}(\inputcoreset, \ktcoreset \text{ with } \ktcoreset[i] = z)$
	   }
    \KwRet{$\ktcoreset$\textup{, refined coreset of size $\floor{n/2}$}}
}
\end{algorithm2e}
\renewcommand{\thealgocf}{\arabic{algocf}}

\newcommand{\ktgoodgen}{\mc E_{\trm{KT}, \delta}}
Define the event
\begin{talign}\label{eq:ktgoodgen}
    \ktgoodgen \defeq \braces{\ktcompresspp ~\text{succeeds}}.
\end{talign}
\citet[Thm.~1,~Rmk.~4]{dwivedi2024kernel} show that
\begin{talign}
    \Parg{\ktgoodgen} \geq 1-\delta.
\end{talign}

We restate \citep[Thm.~4]{dwivedi2024kernel} in our notation:
\begin{lemma}[$\Linf$ guarantee for \ktcompresspp ]\label{lem:ktcompresspp-Linf}
Let $\mc Z \subset \R^d$ and consider a reproducing kernel $\kalg:\mc Z\times \mc Z \to \reals$.
Assume $n/\nout \in 2^{\naturals}$. Let $\ktcoreset \defeq \ktcompresspp(\inputcoreset,\ossymb,\kalg,\delta)$ and define $\Pin \defeq \frac{1}{\nin}\sum_{z\in \inputcoreset} \delta_z$ and $\Qout \defeq \frac{1}{\nout} \sum_{z\in \ktcoreset} \delta_z$. Then on event $\ktgoodgen$, the following bound holds:
\begin{talign}\label{eq:ktcompress-Linf}
    \infnorm{(\Pin - \Qout) \kalg} &\leq c \frac{\ininfnorm{\kalg}}{\nout} \err_{\kalg}(n,\nout, d,\delta, R) , \qtext{where} \\
    \err_{\kalg}(n,\nout, d,\delta, R) &\defeq \sqrt{\log\parenth{\frac{\nout \log_2(n/\nout)}{\delta}}} \times \label{eq:inflation-factor} \\
    &\qquad \brackets{ \sqrt{\log\parenth{\frac{1}{\delta}}} + \sqrt{d \log\parenth{1 + \frac{L_{\kalg}}{\infnorm{\kalg}}(R_{\kalg,n} + R)}}}, \\
    L_{\kalg} &\defeq \sup_{z_1,z_2,z_3 \in \mc Z} \frac{\abss{\kalg(z_1,z_2) - \kalg(z_1,z_3)}}{\stwonorm{z_2 - z_3}}, \qtext{and} \label{eq:L_kalg} \\
    R_{\kalg,n} &\defeq \inf\braces{r : \sup_{\substack{z_1,z_2 \in \mc Z \\ \stwonorm{z_1-z_2} \geq r}} \abss{\kalg(z_1,z_2)} \leq \frac{\infnorm{\kalg}}{n}},  \label{eq:R_kalg}
\end{talign}
for some universal positive constant $c$.
\end{lemma}
\begin{proof}
\todo{Changes over original proof:
\begin{enumerate}
    \item \ktcompresspp has slighly different sub-Gaussian constant from \ktsplit.
    \item Use of $\ininfnorm{\kalg}$ over $\infnorm{\kalg}$ in \cref{eq:ktcompress-Linf} confirmed that we can use in-version everywhere!
\end{enumerate}
}
The claim follows by replacing $\kernel$ \cite[Thm.~4]{dwivedi2024kernel} with $\kalg$, replacing the sub-Gaussian constant of KT with that of \ktcompresspp in \cite[Ex.~5]{shetty2022distribution}, and replacing $\infnorm{\kalg}$ with $\ininfnorm{\kalg}\defeq \sup_{z\in \inputcoreset} \kalg(z,z)$ throughout.
\end{proof}

We restate \citep[Thm.~2]{dwivedi2022generalized} in our notation:
\begin{lemma}[MMD guarantee for \ktcompresspp]\label{lem:ktcompresspp-mmd}
Let $\mc Z\subset \reals^d$ and consider a reproducing kernel $\kalg:\mc Z\times \mc Z \to \reals$.
Assume $n/\nout \in 2^{\naturals}$. Let $\ktcoreset \defeq \ktcompresspp(\kalg,\ossymb)(\inputcoreset)$ and define $\Pin \defeq \frac{1}{\nin}\sum_{z\in \inputcoreset} \delta_z$ and $\Qout \defeq \frac{1}{\nout} \sum_{z\in \ktcoreset} \delta_z$.
Then on event $\ktgoodgen$, the following bound holds:
\begin{talign}
    \sup_{\substack{h \in \rkhs(\kalg): \\ \kalgnorm{h} \leq 1} } \abss{ (\Pin - \Qout) h } &\leq \inf_{\substack{\eps\in (0,1) \\ \inputcoreset \subset \mc A}} \braces{2\eps + \frac{ 2\ininfnorm{\kalg}^{1/2}}{\nout} \errmmd_{\kalg}(n,\nout, \delta, \mc A, \eps)} \qtext{where}\\
    \errmmd_{\kalg}(n,\nout, \delta,R,\eps) &\defeq c \sqrt{\log\parenth{\frac{\nout \log(n/\nout)}{\delta}}  \cdot \brackets{\log\parenth{\frac{1}{\delta}} + \log \coveringnumber_{\kalg}(\mc A, \eps) }}. \label{eq:err-mmd}
\end{talign}
for some universal postive constant $c$.
\end{lemma}
\begin{proof}
\todo{
changes over original statement:
\begin{enumerate}
    \item Want an MMD bound for \ktcompresspp in terms of kernel covering number
    \item The MMD bounds of \citet{shetty2022distribution} are in terms of the inflation factor depending on kernel tail decay.
    \item The MMD bounds of \citep[Thm.~2]{dwivedi2022generalized} are for \ktsplit. We want the MMD bound for \ktcompresspp.
    \item Can we get away with replacing $\mc A$ with $\inputcoreset$ in the covering number term? When do we ever want $\mc A \neq \inputcoreset$? (just use $\inputcoreset$)
\end{enumerate}
}
The claim follows from replacing $\kernel$ in \citep[Thm.~2]{dwivedi2022generalized} with $\kalg$ and replacing the sub-Gaussian constant of KT with that of \ktcompresspp in \cite[Ex.~5]{shetty2022distribution}.
\end{proof}

\section{\pcref{thm:nadaraya-kt}}\label{proof:nadaraya-kt}

\newcommand{\error}{\text{Err}}

Our primary goal is to bound $\E_{\inputcoreset}[( \nwkt(x_0) - \fstar(x_0) )^2]$ for a fixed $x_0\in \X$. Once we have this bound, bounding $\stwonorm{\nwkt - \fstar}^2$ is as straightforward as integrating over $x_0\in \X$.

Consider the following decomposition:
\begin{talign}
    \Esubarg{\inputcoreset}{ \parenth{\nwkt(x_0) - \fstar(x_0)}^2 } &= \Esubarg{\inputcoreset}{ \parenth{\nwkt(x_0) - \nw(x_0) + \nw(x_0) - \fstar(x_0) }^2 } \\ 
    &\leq 2~\Esubarg{\inputcoreset}{ \parenth{\nwkt(x_0) - \nw(x_0)}^2} \label{eq:error-decomp-1}\\
    &\qquad + 2~\Esubarg{\inputcoreset}{ \parenth{ \nw(x_0) - \fstar(x_0) }^2 }. \label{eq:error-decomp-2}
\end{talign}

Define the random variables
\begin{talign}\label{eq:eta}
    \eta_i \defeq \indicator\braces{ \frac{\norm{X_i-x_0}}{h}\leq 1} \qtext{for} i=1,2,\ldots,n.
\end{talign}
Also define the event
\begin{talign}\label{eq:event-E}
    \event \defeq \braces{ \sum_{i=1}^n \eta_i > 0 }.
\end{talign}
Since $X_i$ are i.i.d. samples from $\P$, it follows that $\eta_i$ are i.i.d. Bernoulli random variables with parameter
\begin{talign}\label{eq:pover}
    \overline p \defeq \Parg{ \eta_i=1 } \geq c_0 \pmin h^d,
\end{talign}
where $c_0>0$ depends only on $d$ and $\kappa$ (see \cref{assum:nw-kernel}). 
Denote the denominator terms in $\nw$ and $\nwkt$ by
\begin{talign} \label{eq:denom}
    \denom(\cdot) \defeq \frac{1}{n} \sum_{i=1}^n \kernel \parenth{\cdot,x_i} \qtext{and} \denomkt(\cdot) \defeq \frac{1}{\nout}\sum_{j=1}^{\nout} \kernel \parenth{\cdot,x_i'},
\end{talign}
respectively, and the numerator terms in $\nw$ and $\nwkt$ by
\begin{talign}\label{eq:numer}
    \numer(\cdot) \defeq \frac{1}{n} \sum_{i=1}^n \kernel \parenth{ \cdot,x_i } y_i \qtext{and} \numerkt(\cdot) \defeq \frac{1}{\nout}\sum_{j=1}^{\nout} \kernel \parenth{\cdot,x_i'} y_i',
\end{talign}
respectively.

We now consider two cases depending on the event $\event$.

\textit{Case I:}\quad Suppose event $\event^c$ is satisfied. It follows from \cref{eq:denom} that $\denom(x_0) = 0$, in which case $\nw(x_0) = 0$. Since $\outputcoreset \subset \inputcoreset$, it necessarily follows that $\denomkt(x_0) = 0$ and $\nwkt(x_0) = 0$. Thus, we can bound \cref{eq:error-decomp-1,eq:error-decomp-2} by
\begin{talign}
    \Esubarg{\inputcoreset}{ \parenth{\nwkt(x_0) - \nw(x_0)}^2 \indic{\event^c}} &= 0 \qtext{and} \\
    \Esubarg{\inputcoreset}{ \parenth{\nw(x_0) - \fstar(x_0)}^2\indic{\event^c}} &= \Esubarg{\inputcoreset}{ \parenth{0-\fstar(x_0)}^2 ~ \indic{\event^c}} \\
    &\leq \parenth{\fstar}^2(x_0) \Parg{ \event^c } \\
    &\leq \parenth{\fstar}^2(x_0) (1-\overline p)^n \\
    &\leq \parenth{\fstar}^2(x_0) \exp\braces{ -C n h^d }
\end{talign}
for some positive constant $C$ that does not depend on $n$. Note that these are low-order terms compared to the rest of the calculations, so we may ignore them in the final bound.

\textit{Case II:}\quad Otherwise, we may assume event $\event$ is satisfied. Let us first bound \cref{eq:error-decomp-1}. On event $\ktgoodgen$ \cref{eq:ktgoodgen}, we claim that
\begin{talign}\label{eq:nadaraya-approximation}
    \Esubarg{\inputcoreset}{ \parenth{\nwkt(x_0) - \nw(x_0)}^2 \indic{\event}} \leq \frac{C d \log^2 n}{n h^{2 d}} \qtext{whenever} \overline p = \omega(\sqrt{\frac{d}{n}} \log n).
\end{talign}
We defer the proof to \cref{proof:nadaraya-approximation}.

Letting $X \defeq (X_1,\ldots,X_n)$ and $Y \defeq (Y_1,\ldots,Y_n)$ denote the $x$ and $y$ components of $\inputcoreset$, respectively, we can further decompose \cref{eq:error-decomp-2} by
\begin{talign}
    \Esubarg{\inputcoreset}{\parenth{\nw(x_0) - \fstar(x_0) }^2~ \indic{\mc E} } &= \Esubarg{X}{\Esubarg{Y \mid X}{ \parenth{\nw(x_0) - \Esubarg{Y \mid X}{\nw(x_0)} }^2} ~ \indic{\mc E}} \\
    &\qquad + \Esubarg{X}{\parenth{\Esubarg{Y \mid X}{\nw(x_0)} - \fstar(x_0)}^2~ \indic{\mc E}},
\end{talign}
where the first RHS term corresponds to the variance and the second RHS term corresponds to the bias.
We claim that
\begin{talign}
    \Esubarg{X}{\Esubarg{Y \mid X}{ \parenth{\nw(x_0) - \Esubarg{Y \mid X}{\nw(x_0)} }^2} ~ \indic{\mc E}} &\leq \sigma_\xi^2 \parenth{n \exp\left\{-C n h^d\right\} + \frac{C}{n h^d}} \qtext{and} \label{eq:nadaraya-variance} \\
    \Esubarg{X}{\parenth{\Esubarg{Y \mid X}{\nw(x_0)} - \fstar(x_0)}^2~ \indic{\mc E}} &\leq C\cdot L_f^2 h^{2\beta} \log^2 n, \label{eq:nadaraya-bias}
\end{talign}
for some constant $C>0$ that does not depend on either $n$ or $h$. We defer the proofs to \cref{proof:nadaraya-variance,proof:nadaraya-bias}.
Combining \cref{eq:nadaraya-approximation,eq:nadaraya-variance,eq:nadaraya-bias}, we have
\begin{talign}\label{eq:combined}
    \Esubarg{\inputcoreset}{ \parenth{\nwkt(x_0)- \fstar(x_0)}^2 \indic{\event}} 
    &\leq \underbrace{\frac{C d \log^2 n}{n h^{2 d}}}_{\text{KT bound}} + \underbrace{ 2\sigma_\xi^2 \parenth{n e^{-C n h^d} + \frac{C}{n h^d} } }_{\text{Variance bound}} + \underbrace{2 C L_f^2 h^{2\beta} \log^2 n}_{\text{Bias bound}}.
\end{talign}
Note that $h^d \leq 1$, so the $\frac{C d \log^2 n}{n h^{2 d}}$ term dominates the $\frac{C}{n h^d}$ term. Thus, the optimal choice of bandwidth $h$ comes from balancing 
\begin{talign}\label{eq:nw-kt-h}
    \frac{C}{n h^{2 d}} \sim 2 L_f^2 h^{2\beta} \implies h = cn^{-\frac{1}{2\beta + 2 d}}.
\end{talign}

Finally, we must verify our growth rate assumption on $\overline p$ in \cref{eq:nadaraya-approximation} is satisfied. Since $\beta>0$, we have
\begin{talign}
    \overline p \sgrt{\cref{eq:pover}} c_0 \pmin h^d \seq{\cref{eq:nw-kt-h}} c_0' n^{-\frac{d}{2\beta + 2d}} \implies \lim_{n\to \infty} \frac{\overline p}{\sqrt{\frac{d}{n}} \log n} =\infty.
\end{talign}
Plugging \cref{eq:nw-kt-h} into \cref{eq:combined} yields the advertised bound \cref{eq:nw-kt}.

\subsection{Proof of claim~\cref{eq:nadaraya-approximation}}
\label{proof:nadaraya-approximation}

We first provide a generic result for approximating the numerator and denominator terms defined in \cref{eq:denom,eq:numer}.

\begin{lemma}[Simultaneous $\Linf$ bound using \ktcompresspp with $\kernelnw$] \label{lem:l-infty}
Suppose $\kernel$ satisfies \cref{assum:nw-kernel}. Given $\inputcoreset$, the following bounds hold on the event $\ktgoodgen$: 
\begin{talign} 
    \infnorm{\denom - \denomkt} &\leq c_p \sqrt{\frac{d}{n}} (\log n + \log(1/\delta)) \label{eq:denom-bound} \\
    \sinfnorm{\numer - \numerkt} &\leq c_p \sqrt{\frac{d}{n}} (\log n + \log(1/\delta)), \label{eq:numer-bound}
\end{talign}
where $c_a, c_p>0$ are constants that do not depend on $d$ or $n$.
\end{lemma}
\proofref{proof:l-infty}
In the sequel, we will simply treat the $\log(1/\delta)$ term as a constant, meaning the $\log n$ terms dominate in the expressions.

With this lemma in hand, let us prove the claim \cref{eq:nadaraya-approximation}. 
Define the following events:
\newcommand{\denomktzero}{\mc A}
\newcommand{\denomktnonzero}{\mc B}
\newcommand{\denomlarge}{\mc C}
\begin{talign}
    \denomktzero \defeq \braces{\denomkt(x_0) = 0} \qtext{} \denomktnonzero \defeq \braces{\denomkt(x_0) \neq 0} \qtext{} \denomlarge \defeq \braces{\denom(x_0) \geq \frac{\overline p}{2}}.
\end{talign}
On event $\event$, consider the following decomposition:
\begin{talign}
    \Esubarg{\inputcoreset}{ \parenth{\nwkt(x_0)- \nw(x_0)}^2 \indic{\ktgoodgen}} &= \Esubarg{\inputcoreset}{\parenth{\nwkt(x_0)- \nw(x_0)}^2~ \indic{\ktgoodgen \cap \denomlarge^c} } \label{eq:decomp_kt_1} \\
    &\qquad + \Esubarg{\inputcoreset}{\parenth{\nwkt(x_0)- \nw(x_0)}^2 ~\indic{\ktgoodgen \cap \denomktzero\cap \denomlarge}} \label{eq:decomp_kt_2} \\
    &\qquad + \Esubarg{\inputcoreset}{\parenth{\nwkt(x_0)- \nw(x_0)}^2 ~\indic{\ktgoodgen\cap \denomktnonzero\cap \denomlarge}}. \label{eq:decomp_kt_3}
\end{talign}

\paragraph{Bounding \cref{eq:decomp_kt_1}.}
Note that almost surely, we have 
\begin{talign}
    | \nw(x_0) | \leq \ymax \qtext{and} | \nwkt(x_0) | \leq \ymax.
\end{talign}
Thus, we have
\begin{talign}
    \Esubarg{\inputcoreset}{\parenth{\nwkt(x_0)- \nw(x_0)}^2~ \indic{\denomlarge^c} } &\leq 4\ymax^2 ~\Parg{ n~ \denom(x_0) < \frac{n\overline p}{2} } \\
    &\seq{(i)} 4\ymax^2 \Parg{\sum_{i=1}^n \eta_i - n\overline p < \frac{n\overline p}{2} - n\overline p } \\
    &\sless{(ii)} c_0 \exp\left\{-c_1 n h^d\right\},
\end{talign}
where (i) follows from subtracting $n\overline p$ from both sides of the probability statement and (ii) follows from concentration of Bernoulli random variables (see \cref{proof:nadaraya-variance}).

\paragraph{Bounding \cref{eq:decomp_kt_2}.}
Note that on event $\ktgoodgen\cap \denomlarge$, we have
\begin{talign}
    \denomkt(x_0) &\geq \denom(x_0) - \infnorm{\denom - \denomkt} \\
    &\sgrt{(i)} \frac{\overline p}{2} - c_p \sqrt{\frac{d}{n}} \log n \\
    &\sgrt{\cref{eq:nadaraya-approximation}} c_1 \overline p \sgrt{\cref{eq:pover}} c_2  \pmin h^d > 0.\label{eq:denomkt-lower}
\end{talign}
where step (i) follows from applying \cref{eq:denom-bound} and substituting $\denom \geq \frac{\overline p}{2}$. Hence the events $\ktgoodgen$ and $\denomktzero \cap \denomlarge$ are mutually exclusive with probability $1$, thereby yielding
\begin{talign}
    \Esubarg{\inputcoreset}{\parenth{\nwkt(x_0)- \nw(x_0)}^2 ~\indic{\ktgoodgen \cap \denomktzero \cap \denomlarge }} = 0.
\end{talign}

\paragraph{Bounding \cref{eq:decomp_kt_3}.}
On the event $\denomktnonzero \cap \denomlarge$, we have $\nwkt(x_0) = \frac{\numerkt(x_0)}{\denomkt(x_0)}$ and $\nw(x_0) = \frac{\numer(x_0)}{\denom(x_0)}$, which yields
\begin{talign}
     \nwkt(x_0)- \nw(x_0) &= \frac{\numerkt}{\denomkt} - \frac{\numer(x_0)}{\denom(x_0)} = \frac{\numerkt(x_0) \cdot \denom(x_0) - \numer(x_0) \cdot \denomkt(x_0)}{\denom(x_0) \cdot \denomkt(x_0)}\\
     &= \frac{(\numerkt(x_0) - \numer(x_0))\cdot \denom(x_0) + \numer(x_0) \cdot (\denom(x_0) - \denomkt(x_0))}{\denom(x_0)\cdot \denomkt(x_0)} \\
     &\leq \frac{\abss{\numerkt(x_0) - \numer(x_0)}\cdot \denom(x_0) + \numer(x_0) \cdot \abss{\denom(x_0) - \denomkt(x_0)}}{\denom(x_0)\cdot \denomkt(x_0)}
\end{talign}
We can invoke \cref{eq:denom-bound} and \cref{eq:numer-bound} to bound  $|\denom(x_0) - \denomkt(x_0)|$ and $|\numerkt(x_0) - \numer(x_0)|$ respectively. Thus, we have
\begin{talign}
    \Esubarg{\inputcoreset}{\parenth{\nwkt(x_0)- \nw(x_0)}^2 ~\indic{\ktgoodgen \cap \denomktnonzero \cap \denomlarge}} &\leq \parenth{ \frac{c_a\sqrt{\frac{d}{n}}\log n \cdot \denom(x_0) + c_p \sqrt{\frac{d}{n}}\log n\cdot \numer(x_0)}{\denom(x_0)\cdot \denomkt(x_0)} }^2 \\
    &\leq \frac{2d\cdot \log^2 n}{n} \brackets{\parenth{\frac{c_a}{\denomkt(x_0)}}^2 + \parenth{\frac{\numer(x_0)}{\denom(x_0)}}^2 \parenth{\frac{c_p}{\denomkt(x_0)}}^2} \\
    &\sless{(i)} \frac{2d \cdot \log^2 n}{n} \brackets{\frac{c_a^2 + \ymax^2 c_p^2}{\denomkt(x_0)}}^2\\
    &\sless{(ii)} \frac{C d \log^2 n}{n h^{2 d}},
\end{talign}
for some positive constant $C$ that does not depend on $n$, where step (i) uses the fact that $\frac{\numer(x_0)}{\denom(x_0)} \leq \ymax$ and step (ii) uses the lower bound on $\denomkt(x_0)$ from \cref{eq:denomkt-lower}.
Combining \cref{eq:decomp_kt_1,eq:decomp_kt_2,eq:decomp_kt_3}, we have
\begin{talign}
    \Esubarg{\inputcoreset}{\parenth{\nwkt(x_0)- \nw(x_0)}^2 \indic{\ktgoodgen}} \leq c_0~ \exp\left\{-c_1 n h^d\right\} + \frac{C d \log n}{n h^{2 d}}.
\end{talign}
Note that the second term dominates so that we may drop the first term with slight change to the value of the constant $C$ in the bound \cref{eq:nadaraya-approximation}.

\subsubsection{\pcref{lem:l-infty}}
\label{proof:l-infty}

We first decompose $\kernelnw$ as
\begin{talign}
    \kernelnw((x_1,y_1),(x_2,y_2)) &= \kernel_1((x_1,y_1),(x_2,y_2)) + \kernel_2((x_1,y_1),(x_2,y_2)), \qtext{where}\\
    \kernel_1((x_1,y_1),(x_2,y_2)) &\defeq \kernel(x_1,x_2) \qtext{and} \label{eq:kernel_1} \\
    \kernel_2((x_1,y_1),(x_2,y_2)) &\defeq \kernel(x_1,x_2) \cdot y_1 y_2.\label{eq:kernel_2}
\end{talign}
and note that 
\begin{talign}\label{eq:kernelnw-direct-sum}
    \rkhs(\kernelnw) = \rkhs(\kernel_1) \oplus \rkhs(\kernel_2).
\end{talign}
This fact will be useful later for proving simultaneous $\Linf$ approximation guarantees for $\numer$ and $\denom$.

Given that $\kernel$ satisfies \cref{assum:nw-kernel}, we want to show that $\kernelnw$ defined by \cref{eq:kernel-nw} satisfies the Lipschitz and tail decay properties, so that we may apply \cref{lem:ktcompresspp-Linf}.
Note that
\begin{talign}\label{eq:kernelnw-infnorm}
    \infnorm{\kernelnw} = \infnorm{\kernel} (1 + \ymax^2).
\end{talign}
We claim that kernel $\kernelnw$ satisfies
\begin{talign}
    L_{\kernelnw} &\leq L_\kernel + \ymax \parenth{\infnorm{\kernel} + L_{\kernel} \ymax} \qtext{and} \label{eq:L_kernelnw}\\
    \radius_{\kernelnw,n} &\leq R_{\kernel,n} + 2\ymax \label{eq:R_kernelnw}
\end{talign}

By \citep[Rmk.~8]{dwivedi2024kernel}, we have
\begin{talign}\label{eq:L_k-R_k}
    \frac{L_\kernel}{\infnorm{\kernel}} \leq \frac{L_\kappa}{h}\qtext{and} \radius_{\kernel,n} \leq h \kappa^\dagger(1/n),
\end{talign}
where $\kappa^\dagger$ is defined by \cref{eq:kappa-dagger}. 
Applying \cref{eq:kernelnw-infnorm} and \cref{eq:L_kernelnw}, we have
\begin{talign}
    \frac{L_{\kernelnw}}{\infnorm{\kernelnw}} &\leq \frac{L_\kernel + \ymax \parenth{\infnorm{\kernel} + L_{\kernel} \ymax}}{\infnorm{\kernel} (1 + \ymax^2)} \\
    &\leq \frac{L_\kernel}{\infnorm{\kernel}} + \frac{1}{\ymax} + \frac{L_\kernel}{\infnorm{\kernel}} \\
    &\sless{\cref{eq:L_k-R_k}} \frac{2L_\kappa}{h} + \frac{1}{\ymax}.
\end{talign}
Finally, we have
\begin{talign}
    \frac{L_{\kernelnw}\radius_{\kernelnw,n}}{\infnorm{\kernelnw}} &\leq \parenth{\frac{2L_\kappa}{h} + \frac{1}{\ymax}} \parenth{h \kappa^\dagger(1/n) + 2\ymax} \\
    &= 2L_\kappa \kappa^\dagger(1/n) + \frac{4L_\kappa \ymax}{h} + \frac{h \kappa^\dagger(1/n)}{\ymax} + 2 \\
    &\leq 4 \max\braces{1, L_\kappa \ymax} \cdot \frac{\kappa^\dagger (1/n)}{h} \\
    &\leq 4 \max\braces{1, L_\kappa \ymax} \cdot c' n^\alpha
    ,\label{eq:subpoly-1}
\end{talign}
where the last inequality follows from \cref{assum:nw-kernel} for some universal positive constant $c'$.

Since \cref{assum:compact-support} is satisfied, $R$ is constant. Applying \cref{eq:subpoly-1} to $\err_{\kernelnw}(n,\nout, d,\delta, R)$ as defined by \cref{eq:inflation-factor}, we have the bound
\begin{talign}
    \err_{\kernelnw}(n,\nout, d,\delta, R) \leq c'' \sqrt{\log\parenth{\frac{\nout }{\delta}}} \brackets{ \sqrt{\log\parenth{\frac{8}{\delta}}} + 5\sqrt{d \log n}}
\end{talign}
for some positive constant $c''$. Substituting this into \cref{eq:ktcompress-Linf}, we have
\begin{talign}
    \infnorm{(\Pin - \Qout) \kernelnw} &\leq c_1 \frac{\infnorm{\kernel} (1 + \ymax^2)}{\nout} \sqrt{\log\parenth{\frac{\nout }{\delta}}} \brackets{ \sqrt{\log\parenth{\frac{8}{\delta}}} + 5\sqrt{d \log n}} \\
    &\leq c_2 \frac{\infnorm{\kernel} (1+\ymax^2)}{\nout} \sqrt d (\sqrt{\log n} + \sqrt{\log(1/\delta)})^2,
\end{talign}
for some positive constants $c_1,c_2$.
By definition, 
\begin{talign}
    \infnorm{(\Pin - \Qout) \kernelnw} = \sup_{z\in \mc Z} \inner{(\Pin - \Qout) \kernelnw}{\kernelnw(\cdot, z)}_{\kernelnw}.
\end{talign}
Define $\kernel_1$ and $\kernel_2$ by \cref{eq:kernel_1} and \cref{eq:kernel_2}, respectively, and note that $\kernel_1(\cdot,z),\kernel_2(\cdot,z) \in \rkhs(\kernelnw)$ for all $z\in \mc Z$. We want to show that
\begin{talign}
    \sup_{z\in \mc Z} \inner{(\Pin - \Qout) \kernelnw}{\kernel_1(\cdot, z)}_{\kernelnw} &\leq c_2 \frac{\infnorm{\kernel} (1+\ymax^2)}{\nout} \sqrt d (\sqrt{\log n} + \sqrt{\log(1/\delta)})^2 \qtext{and} \\
    \sup_{z\in \mc Z} \inner{(\Pin - \Qout) \kernelnw}{\kernel_2(\cdot, z)}_{\kernelnw} &\leq c_2 \frac{\infnorm{\kernel} (1+\ymax^2)}{\nout} \sqrt d (\sqrt{\log n} + \sqrt{\log(1/\delta)})^2,
\end{talign}
which would imply \cref{eq:denom-bound} and \cref{eq:numer-bound} (after simplifying all terms besides $n$, $d$, and $\delta$).

The first inequality follows from replacing all occurrences of the test function $\kernelnw(\cdot,(x,y))$ in the proof of \cref{lem:ktcompresspp-Linf} with the function $\kernel_1(\cdot,x)$ and noting that $\inner{\kernelnw(\cdot, (x_i,y_i))}{\kernel_1(\cdot, (x,y))}_{\kernelnw} = \inner{\kernel_1(\cdot,x_i)}{\kernel_1(\cdot,x)}_{\kernel_1}$ from the fact that $\rkhs(\kernelnw) = \rkhs(\kernel_1) \oplus \rkhs(\kernel_2)$ \cref{eq:kernelnw-direct-sum}.

The second inequality follows from replacing all occurrences of the test function $\kernelnw(\cdot,(x,y))$ in the proof of \cref{lem:ktcompresspp-Linf} with the function $\kernel_2(\cdot,x)$ and noting that $\inner{\kernelnw(\cdot, (x_i,y_i))}{\kernel_2(\cdot, (x,y))}_{\kernelnw} = \inner{\kernel_2(\cdot,(x_i,y_i))}{\kernel_2(\cdot,(x,y))}_{\kernel_2}$, again from the fact that $\rkhs(\kernelnw) = \rkhs(\kernel_1) \oplus \rkhs(\kernel_2)$ \cref{eq:kernelnw-direct-sum}.

\paragraph{Proof of claim~\cref{eq:L_kernelnw}.}
We leverage the fact that the Lipschitz constants defined by \cref{eq:L_kalg} satisfies the following additivity property.
Letting $\mc Z = \inputcoreset$, we have
\begin{talign}
    L_{\kernelnw} &= \sup_{z_1,z_2,z_3\in \mc Z} \frac{\abss{\kernelnw(z_1,z_2) - \kernelnw(z_1,z_3)}}{\statictwonorm{z_2-z_3}} \\
    &\leq \sup_{z_1,z_2,z_3\in \mc Z} \frac{\abss{\kernel_1 (z_1,z_2) - \kernel_1(z_1,z_3)}}{\statictwonorm{z_2-z_3}} + \sup_{z_1,z_2,z_3\in \mc Z} \frac{\abss{\kernel_2 (z_1,z_2) - \kernel_2(z_1,z_3)}}{\statictwonorm{z_2-z_3}} \\
    &= L_{\kernel_1} + L_{\kernel_2}.
\end{talign}
We proceed to bound $L_{\kernel_1}$ and $L_{\kernel_2}$ separately. Note that
\begin{talign}
    L_{\kernel_1}  = L_\kernel.
\end{talign}
Applying the definition \cref{eq:L_kalg} to $L_{\kernel_2}$, we have
\begin{talign}
    L_{\kernel_2} &= \sup_{\substack{z_1=(x_1,y_1) \\ z_2=(x_2,y_2) \\ z_3=(x_3,y_3)}} \frac{\abss{\kernel(x_1,x_2) y_1 y_2 - \kernel(x_1,x_3) y_1 y_3}}{\sqrt{\norm{x_2-x_3}^2+ \norm{y_2-y_3}^2}} \\
    &= \sup_{\substack{z_1=(x_1,y_1) \\ z_2=(x_2,y_2) \\ z_3=(x_3,y_3)}} \frac{\abss{y_1}\cdot \abss{\kernel(x_1,x_2) y_2 - \kernel(x_1,x_2) y_3+\kernel(x_1,x_2) y_3-\kernel(x_1,x_3) y_3}}{\sqrt{\norm{x_2-x_3}^2+ \norm{y_2-y_3}^2}} \\
    &= \sup_{\substack{z_1=(x_1,y_1) \\ z_2=(x_2,y_2) \\ z_3=(x_3,y_3)}} \frac{\abss{y_1}\cdot \abss{\kernel(x_1,x_2) (y_2 - y_3) +(\kernel(x_1,x_2) - \kernel(x_1,x_3)) y_3}}{\sqrt{\norm{x_2-x_3}^2+ \norm{y_2-y_3}^2}} \\
    &\leq \sup_{\substack{z_1=(x_1,y_1) \\ z_2=(x_2,y_2) \\ z_3=(x_3,y_3)}} \frac{\abss{y_1}\cdot \abss{\kernel(x_1,x_2) (y_2 - y_3)}}{\sqrt{\norm{x_2-x_3}^2+ \norm{y_2-y_3}^2}} 
    + \sup_{\substack{z_1=(x_1,y_1) \\ z_2=(x_2,y_2) \\ z_3=(x_3,y_3)}} \frac{\abss{y_1}\cdot \abss{(\kernel(x_1,x_2) - \kernel(x_1,x_3)) y_3}}{\sqrt{\norm{x_2-x_3}^2+ \norm{y_2-y_3}^2}} \\
    &\leq \ymax \infnorm{\kernel} + L_{\kernel} \ymax^2. \label{eq:L-k-star}
\end{talign}
Putting together the pieces yields the claimed bound.

\paragraph{Proof of claim~\cref{eq:R_kernelnw}.}

We aim to show that $R_{\kernelnw,n}$ is not much larger than $R_{\kernel,n}$. Note that $\kernelnw$ can be rewritten as
\begin{talign}
    \kernelnw((x_1,y_1),(x_2,y_2)) = (1+y_1 y_2)\kernel(x_1,x_2).
\end{talign}

We define the sets
\begin{talign}
    \Gamma &\defeq \braces{ r:\sup_{\substack{x_1,x_2: \\\stwonorm{x_1-x_2}\geq r}} \abss{\kernel (x_1,x_2)}\leq \frac{\infnorm{\kernel}}{n} }\qtext{and} \label{eq:gamma}\\
    \Gamma^\star &\defeq \braces{ r^\star:\sup_{\substack{(x_1,y_1),(x_2,y_2): \\ \norm{x_1-x_2}^2 + \norm{y_1-y_2}^2 \geq (r^\star)^2}} \abss{\kernel (x_1,x_2)\cdot y_1 y_2 }\leq \frac{\infnorm{\kernel_\star}}{n} },\label{eq:gamma-star}
\end{talign}
noting that 
\begin{talign}
    \radius_{\kernel,n} = \inf \Gamma\qtext{and} \radius_{\kernelnw,n} = \inf \Gamma^\star
\end{talign}
by definition \cref{eq:R_kalg}.

Suppose $r\in \Gamma$. Then for any $(x_1,y_1),(x_2,y_2)$ such that 
\begin{talign}
    \norm{x_1-x_2}^2 + \norm{y_1-y_2}^2 \geq r^2 + 4\ymax^2,
\end{talign}
it must follow that $\norm{x_1-x_2}^2 \geq r^2$ (since $\norm{y_1-y_2}^2\leq 4\ymax^2$ by triangle inequality). Since $r$ satisfies \cref{eq:gamma}, it must follow that
\begin{talign}
    \abss{\kernel (x_1,x_2)\cdot (y_1 y_2 + 1) } \leq \abss{\kernel(x_1,x_2)} (\ymax^2 + 1) \leq \frac{\infnorm{\kernel}}{n} (\ymax^2 + 1) \seq{\cref{eq:kernelnw-infnorm}} \frac{\infnorm{\kernelnw}}{n},
\end{talign}
meaning $\sqrt{r^2 + 4\ymax^2}\in \Gamma^\star$, where recall $\Gamma^\star$ is defined by \cref{eq:gamma-star}. Thus, we have
\begin{talign}\label{eq:R-k-star}
    R_{\kernelnw,n}\leq \sqrt{\radius_{\kernel,n} + 4\ymax^2} \leq  R_{\kernel,n} + 2\ymax
\end{talign}
as desired.

\subsection{Proof of claim~\cref{eq:nadaraya-variance}}
\label{proof:nadaraya-variance}

Suppose event $\event$ \cref{eq:event-E} is satisfied. Define the shorthand for the variance:
\begin{talign}
    &\sigma^2(x_0; X) \defeq \Esubarg{Y \mid X}{ \parenth{\nw(x_0) - \Esubarg{Y \mid X}{\nw(x_0)} }^2}.
\end{talign}
Conditioned on $X_1=x_1,X_2=x_2,\ldots, X_n=x_n$, we have
\begin{talign}\label{eq:cond-mean}
    \Esubarg{Y \mid X}{\what{f}(x_0)}  &= \E_{Y_1\mid X_1,\ldots,Y_n\mid X_n}\left[ \frac{\sum_{i=1}^n Y_i \kernel(X_i, x_0)}{ \sum_{i=1}^n \kernel(X_i, x_0)} \right] = \frac{\sum_{i=1}^n \fstar(X_i) \kernel(X_i, x_0)}{ \sum_{i=1}^n \kernel(X_i, x_0)},
\end{talign}
where we have used the fact that $\Earg{Y\mid X=\cdot} = \fstar(\cdot)$ by assumption \cref{eq:process}.

Note that on event $\event$, we have
\begin{talign}
    \sigma^2(x_0;X) &\seq{\cref{eq:cond-mean}} \Esubarg{Y\mid X}{ \left( \frac{\sum_{i=1}^n Y_i \kernel(X_i,x_0)}{\sum_{i=1}^n \kernel(X_i,x_0)} - \frac{\sum_{i=1}^n \fstar(X_i) \kernel(X_i, x_0)}{ \sum_{i=1}^n \kernel(X_i, x_0)}\right)^2 } \\
    &= \Esubarg{Y \mid X}{\left(\frac{\sum_{i=1}^n \noise_i \kernel(X_i, x_0)}{ \sum_{i=1}^n \kernel(X_i, x_0)} \right)^2} \\
    &= \Var[\noise_1]\cdot \sum_{i=1}^n \frac{\kernel(X_i,x_0)^2}{\left(\sum_{i=1}^n \kernel(X_i,x_0)\right)^2},
\end{talign}
where recall $\noise_1,\ldots,\noise_n$ are i.i.d. random variables with $\Var[\noise_i]=\sigma^2$ by \cref{eq:process}. Taking the expectation w.r.t. $X_1,\ldots,X_n$ and leveraging symmetry, we have
\begin{talign}\label{eq:VarX}
    \Esubarg{X}{\sigma^2(x_0;X) ~\indic{\mc E} } = n \sigma^2 \cdot \sigma_X^2, \qtext{where} \sigma_X^2 \defeq \Esubarg{X}{\frac{\kernel^2(X_1,x_0)}{\parenth{\sum_{i=1}^n \kernel(X_i,x_0)}^2}}.
\end{talign}
$\sigma_X^2$ can be bounded by
\begin{talign}
    \sigma_X^2 &\sless{} \Esubarg{X}{ \frac{\kernel^2(X_1,x_0)}{\parenth{\sum_{i=1}^n \kernel(X_i,x_0)}^2} \indic{ \sum_{i=1}^n \eta_i \leq \frac{n\overline p}{2} } } + \parenth{\frac{2}{n\overline p}}^2 \Esubarg{X}{\kernel^2(X_1,x_0)} \\
    &\sless{(i)} \Parg{\sum_{i=1}^n \eta_i \leq \frac{n\overline p}{2}}+ \parenth{\frac{2}{n\overline p}}^2 \int_{\X} \kernel^2 (x_1,x_0) p(d x_1),
\end{talign}
where step (i) follows from the fact that $\frac{\kernel^2(X_1,x_0)}{\parenth{\sum_{i=1}^n \kernel(X_i,x_0)}^2} \leq 1$.
Using Bernstein's inequality \citep[Prop.~2.14]{wainwright2019high}, the first term can be bounded by
\begin{talign}
    \Parg{\sum_{i=1}^n \eta_i \leq \frac{n\overline p}{2}} &= \Parg{ \sum_{i=1}^n \eta_i - n\overline p \leq -\frac{n\overline p}{2} } \\
    &\sless{} \exp\braces{-\frac{(n\overline p)^2}{2\parenth{n\overline p(1-\overline p)+n\overline p/3}}} \leq \exp\braces{-c_1 n h^d}
\end{talign}
for some universal positive constant $c$.
Applying the fact that $p$ is bounded by \cref{assum:compact-support} and $\kappa$ is square-integrable by \cref{assum:nw-kernel}, we can bound the second term by
\begin{talign}
    \parenth{\frac{2}{n\overline p}}^2 \int_{\X} \kernel^2 (x_1,x_0) p(d x_1) &\sless{(i)} \parenth{\frac{2}{n\overline p}}^2 \pmax h^d \int_{\R^d} \kappa^2(u) d u \\
    &\sless{(ii)} \frac{4}{(n\overline p)^2} \parenth{c_1 h^d} \leq \frac{c_2}{n^2 h^d},
\end{talign}
for some positive constant $c_2$ that does not depend on either $h$ or $n$. Substituting these expressions into \cref{eq:VarX} yields a bound on $\Esubarg{X}{\sigma^2(x_0;X) ~\indic{\mc E}}$ as desired.

\subsection{Proof of claim~\cref{eq:nadaraya-bias}}
\label{proof:nadaraya-bias}

Define the following shorthand for the bias:
\begin{talign}
    b(x_0; X) \defeq \Esubarg{Y \mid X}{\nw(x_0)} - \fstar(x_0).
\end{talign}

We state a more detailed version of the claim.

\begin{lemma}[Bias of Nadaraya-Watson]\label{lem:nadaraya-bias}
Suppose \cref{assum:compact-support,assum:nw-kernel} are satisfied and the event $\event$ \cref{eq:event-E} holds.
If $\fstar\in\Sigma(\beta,L_f)$ for $\beta\in (0,1]$, $L_f>0$, then with high probability, the following statements hold true uniformly for all $x_0\in \X$:
\begin{enumerate}[label=(I.\alph*)]
    \item If $\kernel$ is compactly supported, then $b(x_0;X) \leq L_f h^{\beta}$.\label{item:beta-0-1-compact}
    \item Otherwise, if $\k$ is non-compactly supported, we have $b(x_0;X) \leq c\cdot L_f h^{\beta} \log n$ for some positive constant $c$ that does not depend on $n$. \label{item:beta-0-1-non-compact}
\end{enumerate}
\end{lemma}

For completeness, we first state the proof of claim~\cref{item:beta-0-1-compact} from \citet[Lem.~2]{belkin2019does}.
On the event $\event$ \cref{eq:event-E}, we have
\begin{talign}\label{eq:beta-0-1-compact-1}
    b(x_0; X) \seq{\cref{eq:cond-mean}} \frac{\sum_{i=1}^n (\fstar(X_i) - \fstar(x_0)) \kernel(X_i,x_0)}{\sum_{i=1}^n \kernel(X_i,x_0)} \sless{(i)} \frac{\sum_{i=1}^n L_f \norm{X_i}^\beta \kernel(X_i,x_0)}{\sum_{i=1}^n \kernel(X_i,x_0)} \sless{(ii)} L_f h^\beta,
\end{talign}
where step (i) follows from our assumption that $\fstar \in \Sigma(\beta,L_f)$ and step (ii) follows from our assumption that $\kernel$ is compactly supported, so $\kernel(x,x_0)=0$ whenever $\norm{x-x_0} > h$. We now proceed to proving claim~\cref{item:beta-0-1-non-compact}.

\newcommand{\vol}{\mrm{Vol}}
\newcommand{\mass}{\mrm{Mass}}

\paragraph{Step I: Decomposing the bias.} For any $x_0\in \X$, define the random variables 
\begin{talign}
    A_0(x_0) &\defeq \frac{1}{n} \sum_{i=1}^n \k(X_i,x_0) \indic{\| X_i - x_0\| \leq h} \qtext{and} \label{eq:A0} \\
    A_j(x_0) &\defeq \frac{1}{n} \sum_{i=1}^n \k(X_i,x_0) \indic{2^{j-1} h < \| X_i - x_0\| \leq 2^j h} \qtext{for} j=1,2,\ldots,T,\label{eq:Aj} 
\end{talign}
where
\begin{talign}\label{eq:T}
    T \defeq \ceil{\frac{\log R}{\log h}} = \bigO{\log n}
\end{talign}
when $p_X(\cdot)$ has compact support under \cref{assum:compact-support} and $h$ is defined by \cref{eq:nw-kt-h}. We call these random variables the \textit{kernel empirical means}. We can directly verify that
\begin{talign}
    \sum_{j=0}^T A_0(x_0) = \frac{1}{n}\sum_{i=1}^n \k(X_i,x_0) \qtext{for all} x_0\in \X.
\end{talign}

With these definitions in hand, we may rewrite $b(x_0;X)$ in terms of $A_0(x_0), A_1(x_0), \ldots, A_T(x_0)$ as follows:
\begin{talign}
    b(x_0;X) &= \frac{\frac{1}{n} \sum_{i=1}^n \k(X_i,x_0) (\fstar(X_i) - \fstar(x_0))}{\frac{1}{n} \sum_{i=1}^n \k(X_i,x_0)} \\
    &= \frac{\frac{1}{n} \sum_{i=1}^n \k(X_i,x_0) \indic{\|X_i-x_0\| \leq h} (\fstar(X_i) - \fstar(x_0))}{A_0(x_0) + \sum_{j=1}^{T} A_j(x_0)} \\
    &\qquad + \frac{ \sum_{j=1}^T \frac{1}{n} \sum_{i=1}^n \k(X_i,x_0) \indic{2^{j-1}h < \|X_i-x_0\| \leq 2^j h} (\fstar(X_i) - \fstar(x_0)) }{ A_0(x_0) + \sum_{j=1}^{T} A_j(x_0) } \\
    &\sless{(i)} \frac{\frac{1}{n} \sum_{i=1}^n \k(X_i,x_0) \indic{\|X_i-x_0\| \leq h} L_f \|X_i - x_0 \|^\beta }{A_0(x_0) + \sum_{j=1}^{T} A_j(x_0)} \\
    &\qquad + \frac{ \sum_{j=1}^T \frac{1}{n} \sum_{i=1}^n \k(X_i,x_0) \indic{2^{j-1}h < \|X_i-x_0\| \leq 2^j h} L_f \|X_i - x_0 \|^\beta }{ A_0(x_0) + \sum_{j=1}^{T} A_j(x_0) } \\
    &\sless{(ii)} \frac{\frac{1}{n} \sum_{i=1}^n \k(X_i,x_0) \indic{\|X_i-x_0\| \leq h} L_f h^\beta }{A_0(x_0) + \sum_{j=1}^{T} A_j(x_0)} \label{eq:sum-0}\\
    &\qquad + \frac{ \sum_{j=1}^T \frac{1}{n} \sum_{i=1}^n \k(X_i,x_0) \indic{2^{j-1}h < \|X_i-x_0\| \leq 2^j h} L_f (2^j h)^\beta }{ A_0(x_0) + \sum_{j=1}^{T} A_j(x_0) }, \label{eq:sum-1-T}
\end{talign}
where step (i) follows from the fact that $\fstar(X_i) - \fstar(x_0) \leq L_f \|X_i-x_0\|^\beta$ by \cref{def:L-beta-smooth} and step (ii) follows from the fact that $\|X_i-x_0\| \leq 2^j h$ under the indicator functions. Applying \cref{eq:A0} to \cref{eq:sum-0} and \cref{eq:Aj} to \cref{eq:sum-1-T}, we have
\begin{talign}\label{eq:bias-sum}
    b(x_0;X) &\leq \frac{A_0(x_0) L_f h^\beta}{A_0(x_0) + \sum_{j=1}^T A_j(x_0)} + \frac{\sum_{j=1}^T A_j(x_0) L_f (2^j h)^\beta}{A_0(x_0) + \sum_{j=1}^T A_j(x_0)} \leq L_f h^\beta \parenth{\sum_{i=0}^T \frac{A_i(x_0) 2^{i\beta}}{\sum_{j=0}^T A_j(x_0)}}.
\end{talign}

For any $z\in \X$, define the following shorthands:
\begin{talign}
    a_0(z) &\defeq \Esubarg{X}{\k(X,z)\indic{\snorm{X-z}\leq h}}
    \qtext{and} \label{eq:a0} \\
    a_j(z) &\defeq \Esubarg{X}{\k(X,x)\indic{2^{j-1}h < \snorm{X-z} \leq 2^j h}} \qtext{for} j=1,2,\ldots,T. \label{eq:aj}
\end{talign}
We can directly verify that
\begin{talign}
    \Esubarg{X}{A_0(z)} = a_0(z) \qtext{and} \Esubarg{X}{A_j(z)} = a_j(z) \qtext{for all} j=1,2,\ldots, T.
\end{talign}

Under \cref{assum:compact-support}, $a_i(z)$ can be computed directly from the decay properties of $\k$ as follows:
\begin{talign}
    a_0(z) &= \int_{\snorm{x-z}\leq h} \k(x,z) p_X(dx) \geq \pmin \int_{\snorm{x-z}\leq h} \k(x,z) dx \\
    a_j(z) &= \int_{2^{j-1}<\snorm{x-z}\leq 2^j h} \k(x,z) p_X(dx) \leq \pmax \int_{2^{j-1}<\snorm{x-z}\leq 2^j h} \k(x,z) dx \qtext{for} j=1,2,\ldots,T.
\end{talign}

\begin{remark}
Note that for fixed $x_0\in \X$, $A_i(x_0)$, Hoeffding's inequality implies that $A_i(x_0)$ concentrates around $a_i(x_0)$ for all $i=0,1,\ldots,T$ with high probability. Assuming that $\k$ decays sufficiently rapidly (so that $a_j(x_0) 2^{j\beta} = \bigO{a_0(x_0)}$ for all $j\in [T]$), we have
\begin{talign}
    b(x_0;X) \leq L_f h^\beta \parenth{\sum_{i=1}^T O(1)} = \bigO{L_f h^\beta \log n}.
\end{talign}
It remains to show that this bound holds uniformly over all $x_0\in \X$ with high probability.
\end{remark}

\paragraph{Step II: Uniform concentration of empirical kernel means.}

Fix $\eps \in (0,\frac{h}{2}]$ and let $\cover_0$ be a $\eps$-cover of $\X$ w.r.t. the Euclidean norm (see \cite[Def.~5.1]{wainwright2019high} for the detailed definition). We will optimize the choice of $\eps$ at the end of the proof. By \citep[Lem.~5.7]{wainwright2019high}, we have
\begin{talign}\label{eq:cover0-size}
    |\cover_0| \leq (1+ \frac{2R}{\eps})^d
\end{talign}

We now introduce a slightly modified version of \cref{eq:a0,eq:aj} defined as
\begin{talign}
    a_0^{\eps}(z) &\defeq \Esubarg{X}{\k(X,z) \indic{\snorm{X-z}\leq h - \eps}} \qtext{and} \label{eq:a0-eps}\\
    a_j^{\eps}(z) &\defeq \Esubarg{X}{\k(X,z) \indic{2^{j-1}h - \eps < \snorm{X-z} \leq 2^j h + \eps}}\qtext{for} j=1,2,\ldots,T,\label{eq:aj-eps}
\end{talign}
respectively. Using the fact that $\eps\leq h/2$, \cref{assum:compact-support,assum:nw-kernel}, and a change of variables, we have
\begin{talign}
    a_0^\eps(z) &\geq \pmin \int_{\snorm{x-z}\leq \frac{h}{2}} \k(x,z) dx \seq{} \pmin h^d \omega_{d-1} \int_{0}^{\half} \kappa(u) u^{d-1} du, \qtext{and} \\
    a_j^\eps(z) &\leq \pmax \int_{(2^{j-1}-\half)h < \snorm{x-z}\leq (2^j +\half)h} \k(x,x_0') dx \seq{} \pmax h^d \omega_{d-1} \int_{2^{j-1}-\half}^{2^j +\half} \kappa(u) u^{d-1} du \qtext{for} j\in [T],
\end{talign}
where $\omega_{d-1} \defeq 2\pi^{d/2}/ \Gamma(d/2)$ denotes the surface area of the unit sphere in $\reals^d$.
Moreover, we consider the following concentration events:
\begin{talign}
    \mc A_0^\eps(z) &\defeq \braces{ \frac{1}{n}\sum_{i=1}^n \k(X_i,z) \indic{\snorm{X_i-z}\leq h-\eps} \geq a_0^\eps(z) - \Delta_0}\qtext{and} \label{eq:A0-eps}\\
    \mc A_j^\eps(z) &\defeq \braces{ \frac{1}{n}\sum_{i=1}^n \k(X_i,z) \indic{2^{j-1}h-\eps < \snorm{X_i-z}\leq 2^j h + \eps} \leq a_j^\eps(z) + \Delta_j } \qtext{for} j\in [T],\label{eq:Aj-eps}
\end{talign}
where $\Delta_0,\Delta_1,\ldots,\Delta_T$ are positive scalars to be chosen later.

For any scalars $0<a<b$ and $z\in \X$, further denote the empirical mass contained within the ball or annulus centered at point $z$ by
\begin{talign}
    \mass_{z}(a) &\defeq \frac{1}{n} \sum_{i=1}^n \indic{\snorm{X_i-z}\leq a}
    \qtext{and} \\
    \mass_{z}(a,b) &\defeq \frac{1}{n}\sum_{i=1}^n \indic{a < \snorm{X_i-z} \leq b}.
\end{talign}
We now introduce a lemma that bounds the empirical kernel mean restricted to each ball and annuli:

\begin{lemma}[Uniform concentration of empirical kernel means]
\label{lem:k-empirical}
Suppose event $\bigcap_{z\in \cover_0} \bigcap_{i=0}^T \mc A_i^\eps(z)$ is satisfied. Then for every $x_0\in \X$, there exists $x_0'\in \cover_0$ such that $\snorm{x_0-x_0'}\leq \eps$ and the following bounds hold simultaneously:
\begin{talign}
    A_0(x_0) &\geq a_0^{\eps}(x_0') -\Delta_0 - \mass_{x_0'}(h-\eps) \cdot L_{\k} \eps \qtext{and} \label{eq:k-empirical-lower} \\
    A_j(x_0) &\leq a_j^{\eps}(x_0') + \Delta_j + \mass_{x_0'}(2^{j-1}h -\eps, 2^j h + \eps) \cdot \kappa(\frac{2^{j-1}h - 2\eps}{h}) \qtext{for} j\in [T]. \label{eq:k-empirical-upper}
\end{talign}
Here, $L_\k$ and $\kappa(\cdot)$ are properties of $\k$ from \cref{assum:nw-kernel}.
\end{lemma}
\proofref{proof:k-empirical}
This result suggests that we can control the error between the empirical kernel mean, $A_i(x_0)$, and its (shifted) population counterpart, $a_j^\eps(x_0')$, by choosing $\eps$ to be small and leveraging the decay property of $\kappa$. To bound the $\mass_{x_0'}(h-\eps)$ and $\mass_{x_0'}(2^{j-1}h -\eps, 2^j h + \eps)$ terms, we define the following events:
\begin{talign}
    \mc M_0(x_0') &\defeq \braces{\mass_{x_0'}(h-\eps) \leq \Esubarg{X}{\mass_{x_0'}(h-\eps)} + t_0}, \qtext{and} \label{eq:mass0-event} \\
    \mc M_j(x_0') &\defeq \braces{\mass_{x_0'}(2^{j-1}h-\eps,2^j h + \eps) \leq \Esubarg{X}{\mass_{x_0'}(2^{j-1}h-\eps,2^j h + \eps)} + t_j} \qtext{for} j \in [T],\label{eq:massj-event}
\end{talign}
where $t_0,t_1,\ldots,t_T$ are positive scalars to be chosen later.
For any scalars $0<a<b$, we may apply \cref{assum:compact-support} to obtain
\begin{talign}\label{eq:vol-upper-bound}
    \Esubarg{X}{\mass_{x_0'}(a)} \leq \pmax \vol(a) \qtext{and}
    \Esubarg{X}{\mass_{x_0'}(a,b)} \leq \pmax \vol(b),
\end{talign}
where $\vol(r)\defeq \pi^d / \Gamma(d/2+1) r^d$ denotes the volume of the Euclidean ball $\ball_2(r)\subset \reals^d$.

\paragraph{Step III: Putting things together.} For any $z\in \cover_0$, note that
\begin{talign}
    &\k(X_i,z)\indic{\snorm{X_i-z}\leq h-\eps}  \in [0,\infnorm{\k}] \qtext{and} \\
    &\k(X_i,z) \indic{2^{j-1}h-\eps < \snorm{X_i-z}\leq 2^j h + \eps}\in [0,\infnorm{\k}]
\end{talign}
and therefore these random variables are sub-Gaussian with parameter $\sigma = \half \infnorm{\k}$. Thus, we may bound \cref{eq:A0-eps} and \cref{eq:Aj-eps} using Hoeffding's bound \citep[Prop.~2.5]{wainwright2019high}, obtaining
\begin{talign}\label{eq:conc-Ai-eps}
    \Parg{\mc A_i^\eps(z)^c }  \leq \exp\braces{-\frac{2n \Delta_0^2}{\infnorm{\k}^2}} 
    \qtext{for all} i=0,1,\ldots,T \qtext{and} z\in \cover_0,
\end{talign}
Again by Hoeffding's bound \citep[Prop.~2.5]{wainwright2019high}, we have
\begin{talign}
    \Parg{\mc M_i(z)^c } \leq \exp\braces{- 2n t_i^2} \qtext{for all} i=0,1,\ldots,T \qtext{and} z\in \cover_0.
\end{talign}

Next, it remains to choose the values of $\eps$, $\Delta_0,\ldots,\Delta_T$ and $t_0,\ldots,t_T$. Consider
\begin{talign}
    \Delta_0 &= \half \cdot \pmin h^d \omega_{d-1} \int_{0}^{\half} \kappa(u) u^{d-1} du, \label{eq:Delta0} \\
    \Delta_j &= 2^{-j \beta} \cdot \pmax h^d \omega_{d-1} \int_{0}^{\half} \kappa(u) u^{d-1} du \qtext{for} j=1,2,\ldots,T, \label{eq:Deltaj} \\
    t_0 &= \half \pmax \vol(h-\eps), \qtext{and} \label{eq:t0} \\
    t_j &= \half \pmax \vol(2^j h + \eps) \qtext{for} j=1,2,\ldots,T. \label{eq:tj}
\end{talign}
Combining \cref{eq:k-empirical-lower,eq:mass0-event,eq:vol-upper-bound,eq:Delta0,eq:t0}, we have on the event $\mc E_0 \defeq \bigcap_{z\in \cover_0} (\mc A_0^\eps(z)\cap \mc M_0(z))$ the following bound:
\begin{talign}
    A_0(x_0) &\geq \half \pmin h^d \omega_{d-1} \int_{0}^{\half} \kappa(u) u^{d-1} du - \pmax \frac{\pi^d}{\Gamma(\frac{d}{2}+1)}(h-\eps)^d\cdot L_k \eps \\
    &\geq h^d \omega_{d-1} \parenth{\half \pmin \int_0^{\half} \kappa(u) u^{d-1} du - \pmax \frac{\pi^d}{\Gamma(\frac{d}{2}+1) \omega_{d-1}}L_{\k}\eps} \\
    &\sgrt{(i)} \frac{1}{4} \pmin h^d \omega_{d-1} \int_0^{\half} \kappa(u) u^{d-1} du,
\end{talign}
where step (i) follows from setting
\begin{talign}\label{eq:eps}
    \eps = \min\braces{\frac{1}{4} \brackets{\pmin \int_0^{\half} \kappa(u) u^{d-1} du} \parenth{\pmax \frac{\pi^d}{\Gamma(\frac{d}{2}+1) \omega_{d-1}}L_{\k}}\inv, \frac{h}{2}}.
\end{talign}

Similarly, for each $j=1,2,\ldots,T$, we may combine \cref{eq:k-empirical-upper,eq:massj-event,eq:vol-upper-bound,eq:Deltaj,eq:tj} to obtain a bound on $A_j(x_0)$. Hence, on the $\mc E_j \defeq \bigcap_{z\in \cover_0} (\mc A_j^\eps(z)\cap \mc M_j(z))$, we have
\begin{talign}
    A_j(x_0) &\leq \pmax h^d \omega_{d-1} \parenth{\int_{2^{j-1}-\half}^{2^j +\half} \kappa(u) u^{d-1} du + 2^{-j \beta} \int_{0}^{\half} \kappa(u) u^{d-1} du} + \pmax \frac{\pi^d}{\Gamma(\frac{d}{2}+1)}(2^j h + \eps)^d\cdot \kappa(\frac{2^{j-1}h - 2\eps}{h}) \\
    &\sless{(i)} \pmax h^d \omega_{d-1} \parenth{ \int_{2^{j-1}-\half}^{2^j +\half} \kappa(u) u^{d-1} du + 2^{-j \beta} \int_{0}^{\half} \kappa(u) u^{d-1} du + \frac{\pi^d (2^j + \half)^d}{\Gamma(\frac{d}{2}+1) \omega_{d-1}} \kappa(2^{j-1} - 1)}
\end{talign}
where step (i) employs the fact that $\eps\leq h/2$. Note that for $j=1,2,\ldots,T$:
\begin{talign}\label{eq:Aj-A0}
    \frac{A_j(x_0) 2^{j\beta}}{A_0(x_0)} &\leq \frac{4\pmax}{\pmin} \parenth{\frac{2^{j\beta} \int_{2^{j-1}-\half}^{2^j +\half} \kappa(u) u^{d-1} du + \int_{0}^{\half} \kappa(u) u^{d-1} du + \frac{\pi^d (2^j + \half)^d 2^{j\beta}}{\Gamma(\frac{d}{2}+1) \omega_{d-1}} \kappa(2^{j-1} - 1)}{ \int_0^{\half} \kappa(u) u^{d-1} du}} = \bigO{1},
\end{talign}
where the last inequality comes from the fact that
\begin{talign}
    \frac{2^{j\beta} \int_{2^{j-1}-\half}^{2^j +\half} \kappa(u) u^{d-1} du}{\int_0^{\half} \kappa(u) u^{d-1} du} = \bigO{1} 
    \qtext{and}
    \frac{\frac{\pi^d (2^j + \half)^d 2^{j\beta}}{\Gamma(\frac{d}{2}+1) \omega_{d-1}} \kappa(2^{j-1} - 1)}{\int_0^{\half} \kappa(u) u^{d-1} du} = \bigO{1}
\end{talign}
by \cref{assum:nw-kernel}. Substituting this bound into \cref{eq:bias-sum}, we have
\begin{talign}
    b(x_0;X) \sless{\cref{eq:bias-sum}} L_f h^\beta \parenth{\sum_{i=0}^T \frac{A_i(x_0) 2^{j\beta}}{\sum_{j=0}^T A_j(x_0)}} \leq L_f h^\beta \parenth{\sum_{i=0}^T \frac{A_i(x_0) 2^{j\beta}}{ A_0(x_0) }} \sless{\cref{eq:Aj-A0}} L_f h^\beta T C \sless{\cref{eq:T}} \bigO{L_f h^\beta \log n}
\end{talign}
as desired. Finally, it remains to control the probability that this bound on $b(x_0;X)$ holds. Applying union-bound, we have
\begin{talign}
    \Parg{\mc E_0^c\cup \ldots \cup \mc E_T^c} &\sless{\cref{eq:cover0-size},\cref{eq:conc-Ai-eps}} (1+\frac{2R}{\eps})^d \sum_{i=0}^T \parenth{\exp\braces{-\frac{2n \Delta_i^2}{\infnorm{\k}^2}} + \exp\braces{-2n t_i^2}} \\
    &\leq (1+\frac{2R}{\eps})^d  \sum_{j=0}^T \parenth{\exp\braces{-c_1 n h^{2d} 4^{-j\beta}} + \exp\braces{-c_2 n h^{jd}}} \\
    &\leq (1+\frac{2R}{\eps})^d (T+1) \parenth{\exp\braces{-c_1 n h^{2d} 4^{-T\beta}} + \exp\braces{-c_2 n h^{Td}}},
\end{talign}
for some universal positive constants $c_1,c_2$ that do not depend on $n$. Note that by \cref{eq:eps}, we have $(1+\frac{2R}{\eps})^d =\bigO{(1 + \frac{1}{h})^d}$. Setting $h= cn^{-\frac{1}{2\beta + 2 d}}$ as in \cref{eq:nw-kt-h} and recalling that $T=\bigO{\log n}$, we have $\Parg{\mc E_0^c\cup \ldots \cup \mc E_T^c} \ll 1$ for sufficiently large $n$. This completes the proof of \cref{lem:nadaraya-bias}.

\subsubsection{\pcref{lem:k-empirical}}
\label{proof:k-empirical}

Fix any $x_0\in \X$. By definition of $\cover_0$, there exists $x_0'\in \cover_0$ such that $\snorm{x_0-x_0'}\leq \eps$. Consider any $x\in \ball_2(h-\eps;x_0')$\footnote{We use the notation $\ball_2(r;x)$ to denote the Euclidean ball of radius $r$ centered at $x\in \reals^d$.}. By triangle inequality, $\snorm{x-x_0} \leq \snorm{x-x_0'} + \snorm{x_0' - x_0} \leq (h-\eps) + \eps$. Therefore, we have
\begin{talign}\label{eq:ball-h-eps-contained}
    \ball_2(h-\eps;x_0') \subset \ball_2(h;x_0).
\end{talign}

We now introduce the following lemma, which is a finite-sample version of \citep[Lem.~8]{dwivedi2024kernel}:
\begin{lemma}[$\Linf$ bound on $\ltwo$ kernel error (finite sample)]
\label{lem:Linf-L2-kernel-error}
Consider any bounded kernel $\k\in L^{2,\infty}$, points $x_0,x_0'\in \X$, and function $g\in L^2(\P_n)$. For any $r,a,b\geq 0$ with $a+b=1$, points $x_1,\ldots,x_n\in \X$, and set $A\subset \X$, we have
\begin{talign}
    &\abss{\frac{1}{n} \sum_{i=1}^n g(x_i)\parenth{\k(x_i,x_0) - \k(x_i,x_0')} \boldone_A(x_i)} \\
    &\qquad \leq \snorm{g}_{A,L^2(\P_n)} \cdot \brackets{ \snorm{\k(\cdot,x_0) - \k(\cdot,x_0')}_{A,\infty} \mass_A^{\half}(r) + 2\tail[\k,A](ar) + 2 \snorm{\k}_{A,L^{2,\infty} } \indic{\snorm{x_0-x_0'}\geq br}}.
\end{talign}
where 
\begin{talign}
\snorm{f}_{A,L^2(\P_n)} &\defeq \parenth{\frac{1}{n}\sum_{i=1}^n f^2 (x_i) \boldone_A(x_i)}^{\half} \\
\snorm{f}_{A,\infty} &\defeq \max_{i\in [n]} \abss{f(x_i)\boldone_A(x_i)} \\
\mass_A(r) &\defeq \frac{1}{n}\sum_{i=1}^n \indic{\snorm{x_i-x_0'}\leq r, x_i\in A} \\
\tail[A,\k](r) &\defeq \sup_{x'\in \braces{x_0,x_0'}}\parenth{\frac{1}{n}\sum_{i=1}^n \k^2(x_i,x') \indic{\snorm{x_i-x'}\geq r, x_i\in A}}^{\half} \label{eq:tail} \\
\snorm{\k}_{A,L^{2,\infty}} &\defeq \sup_{x'\in \braces{x_0,x_0'}} \snorm{\k(\cdot,x')}_{A,L^2(\P_n)}.
\end{talign}
\end{lemma}
\proofref{proof:Linf-L2-kernel-error}
We now use this lemma to prove each of the claims in \cref{lem:k-empirical}.

\paragraph{Proof of claim~\cref{eq:k-empirical-lower}.}

We consider the following basic decomposition:
\begin{talign}\label{eq:k-empirical-decomp}
    &\frac{1}{n} \sum_{i=1}^n \k(X_i,x_0) \indic{\|X_i-x_0\| \leq h} \\
    &\qquad = \frac{1}{n} \sum_{i=1}^n \k(X_i,x_0') \indic{\|X_i-x_0'\|\leq h-\eps} \label{eq:k-empirical}\\
    &\qquad \qquad - \parenth{\frac{1}{n}\sum_{i=1}^n \k(X_i,x_0') \indic{\|X_i-x_0'\| \leq h-\eps} - \frac{1}{n}\sum_{i=1}^n \k(X_i,x_0)\indic{\|X_i-x_0\|\leq h}} \label{eq:k-empirical-diff}.
\end{talign}
Note that the term \cref{eq:k-empirical} can be directly bounded using the definition of \cref{eq:A0-eps}. Next, we bound the term \cref{eq:k-empirical-diff}. Define
\begin{talign}\label{eq:g1}
    g \defeq \boldone
    \qtext{and}
    A \defeq \ball_2(h-\eps; x_0'),
\end{talign}
so that $\snorm{g}_{A,L^2(\P_n)} = \mass_{x_0'}^{\half}(h-\eps)$. By property \cref{eq:ball-h-eps-contained}, we have
\begin{talign}\label{eq:indicator-bound}
    \indic{\staticnorm{X_i-x_0'} \leq h-\eps} \leq \indic{\staticnorm{X_i-x_0}\leq h} \qtext{for all} i=1,2,\ldots,n.
\end{talign}
Observe that
\begin{talign} \label{eq:L2-kernel-error}
    &\frac{1}{n} \sum_{i=1}^n g(X_i)\parenth{\k(X_i,x_0) - \k(x_i,x_0')} \boldone_A(x_i)\\
    &= \frac{1}{n} \sum_{i=1}^n \parenth{\k(X_i, x_0') - \k(X_i, x_0)} \indic{\staticnorm{X_i-x_0'} \leq h-\eps} \\
    &= \frac{1}{n} \sum_{i=1}^n \k(X_i,x_0') \indic{\staticnorm{X_i-x_0'}\leq h-\eps} - \frac{1}{n} \sum_{i=1}^n \k(X_i,x_0) \indic{\staticnorm{X_i-x_0'} \leq h-\eps} \label{eq:fact-1}\\
    &\sgrt{\cref{eq:indicator-bound}} \frac{1}{n} \sum_{i=1}^n \k(X_i,x_0') \indic{\staticnorm{X_i-x_0'}\leq h-\eps} - \frac{1}{n} \sum_{i=1}^n \k(X_i,x_0) \indic{\staticnorm{X_i-x_0}\leq h},
\end{talign}

To upper bound \cref{eq:L2-kernel-error}, we now apply \cref{lem:Linf-L2-kernel-error} with $r = h$, $a=\half$, and $b=\half$ to obtain
\begin{talign}
    &\abss{\frac{1}{n} \sum_{i=1}^n g(x_i)\parenth{\k(x_i,x_0) - \k(x_i,x_0')} \boldone_A(x_i)} \\
    &\qquad \leq \mass_{x_0'}^{\half}(h-\eps) \cdot \brackets{ \snorm{\k(\cdot,x_0)-\k(\cdot,x_0')}_{A,\infty} \mass_A^{\half}(h) + 2\tail[\k,A](\frac{h}{2}) + 2 \snorm{\k}_{A,L^{2,\infty}} \indic{\snorm{x_0-x_0'} > \frac{h}{2}}} \\
    &\qquad \leq  \mass_{x_0'}^{\half}(h-\eps) \cdot \brackets{ 
    L_{\k}\eps\cdot \mass_{x_0'}^{\half}(h-\eps)}, \label{eq:L2-kernel-error-bound}
\end{talign}
where the last inequality uses the fact that
\begin{talign}
    \snorm{\k(\cdot,x_0)-\k(\cdot,x_0')}_{A,\infty} \leq L_{\k} \snorm{x_0-x_0'} \leq L_{\k} \eps
\end{talign}
from the Lipschitz property of $\k$ (see \cref{assum:nw-kernel}) and the fact that
\begin{talign}
    \mass_A^{\half}(h) = \mass_{x_0'}^{\half}(h-\eps),
    \qtext{} 
    \tail[\k,A](\frac{h}{2}) = 0,
    \qtext{and}
    \snorm{x_i-x_0}\leq \eps \leq \frac{h}{2}.
\end{talign}
Substituting the bounds \cref{eq:A0-eps,eq:L2-kernel-error,eq:L2-kernel-error-bound} into the decomposition \cref{eq:k-empirical-decomp} yields the desired claim.

\paragraph{Proof of claim~\cref{eq:k-empirical-upper}.}
We consider the basic decomposition:
\begin{talign}\label{eq:k-empirical-annulus-decomp}
    &\frac{1}{n}\sum_{i=1}^n \k(X_i,x_0) \indic{2^{j-1}h - \eps <\snorm{X_i-x_0}\leq 2^j h + \eps} \\
    &\qquad = \frac{1}{n}\sum_{i=1} \k(X_i,x_0') \indic{2^{j-1}h - \eps < \snorm{X_i-x_0'}\leq 2^j h + \eps} \label{eq:k-empirical-annulus} \\
    &\qquad \qquad + (\frac{1}{n}\sum_{i=1}^n \k(X_i,x_0')\indic{2^{j-1} h < \snorm{X_i-x_0} \leq 2^j h} \\
    &\qquad \qquad - \frac{1}{n}\sum_{i=1}^n \k(X_i,x_0')\indic{2^{j-1}h -\eps < \snorm{X_i-x_0'} \leq 2^j h + \eps}) \label{eq:k-empirical-annulus-diff}
\end{talign}
Note that the term \cref{eq:k-empirical-annulus} can be directly bounded using the definition of \cref{eq:Aj-eps}. Next, we bound the term \cref{eq:k-empirical-annulus-diff}. By triangle inequality, we have
\begin{talign}
    \snorm{x-x_0'} \leq \snorm{x-x_0} + \snorm{x_0-x_0'} \leq 2^j h + \eps,\qtext{for all} x\in \ball_2(2^j h; x_0),
\end{talign}
which implies that $\ball_2(2^j h; x_0) \subset \ball_2(2^j h + \eps; x_0')$. Moreover, we can verify that $\ball_2(2^{j-1}h-\eps; x_0') \subset \ball_2(2^{j-1} h; x_0)$. Hence,
\begin{talign}\label{eq:annulus-indicator-bound}
    \indic{2^{j-1} h< \snorm{X_i-x_0}\leq 2^j h} 
    \leq 
    \indic{2^{j-1}h-\eps < \snorm{X_i-x_0'}\leq 2^j h + \eps } \qtext{for all} i=1,2,\ldots,n.
\end{talign}
Now define 
\begin{talign}\label{eq:g-A}
    g \defeq 1 
    \qtext{and} 
    A \defeq \ball_2(2^j h + \eps; x_0') \setminus \ball_2(2^{j-1} h - \eps; x_0'),
\end{talign}
so that $\snorm{g}_{A,\ltwo(\P_n)} = \mass_{x_0'}^{\half}(2^{j-1}h -\eps, 2^j h + \eps)$.
Hence, \cref{eq:k-empirical-annulus-diff} can be bounded by
\begin{talign}
    &\frac{1}{n}\sum_{i=1}^n \k(X_i,x_0')\indic{2^{j-1} h<\snorm{X_i-x_0} \leq 2^j h} - \frac{1}{n}\sum_{i=1}^n \k(X_i,x_0')\indic{2^{j-1} h-\eps < \snorm{X_i-x_0'} \leq 2^j h + \eps} \\
    &\qquad \sless{\cref{eq:annulus-indicator-bound}} \frac{1}{n}\sum_{i=1}^n \k(X_i,x_0') \indic{2^{j-1} h-\eps< \snorm{X_i-x_0'}\leq 2^j h + \eps} \\
    &\qquad \qquad - \frac{1}{n}\sum_{i=1}^n \k(X_i,x_0')\indic{2^{j-1} h -\eps< \snorm{X_i-x_0'} \leq 2^j h + \eps} \\
    &\qquad \sless{\cref{eq:g-A}} \frac{1}{n} \sum_{i=1}^n g(X_i)\parenth{\k(X_i,x_0') - \k(X_i,x_0')} \boldone_A(x_i) \\
    &\qquad \sless{} \mass_{x_0'}^{\half}(2^{j-1}h -\eps, 2^j h + \eps) \cdot \big[ \snorm{\k(\cdot,x_0) - \k(\cdot,x_0')}_{A,\infty} \mass_A^{\half}(2(2^j h + \eps)) \\
    &\qquad \qquad + 2 \tail[\k,A](2^j h + \eps) + 2 \snorm{\k}_{A,L^{2,\infty}} \indic{\snorm{x_0-x_0'}\geq 2^j h + \eps} \big] \label{eq:k-empirical-annulus-diff-bound}
\end{talign}
where the last inequality follows from applying \cref{lem:Linf-L2-kernel-error} with $r = 2(2^j h + \eps)$ and $a=b=\half$. 
By triangle inequality, we have 
\begin{talign}
    \snorm{X_i-x_0} \geq \snorm{X_i-x_0'} - \snorm{x_0'-x_0} \sgrt{(i)} (2^{j-1}h -\eps) - \eps \geq 2^{j-1}h - 2\eps,
\end{talign}
where step (i) follows from the definition of $A$ \cref{eq:g-A}. Since $\kappa$ is monotonically decreasing by \cref{assum:nw-kernel}, we have $\kappa(\frac{\snorm{X_i-x_0}}{h}) \leq \kappa(\frac{2^{j-1}h - 2\eps}{h})$. Thus, we have
\begin{talign}
    \snorm{\k(\cdot,x_0) - \k(\cdot,x_0')}_{A,\infty} 
    &\leq \max\braces{\snorm{\k(\cdot,x_0)}_{A,\infty}, \snorm{\k(\cdot,x_0')}_{A,\infty}} \\
    &\leq \max\braces{ \kappa(\frac{2^{j-1}h - 2\eps}{h}) , \kappa(\frac{2^{j-1}h-\eps}{h}) } \\
    &\sless{(i)} \kappa(\frac{2^{j-1}h - 2\eps}{h}), \label{eq:k-infty}
\end{talign}
where step (i) again uses the fact that $\kappa$ is monotonically decreasing. Moreover, note that
\begin{talign}
    \mass_A(2(2^j h + \eps)) &= \mass_{x_0'}(2^{j-1}h-\eps, 2^j h + \eps),\\
    \tail[A,\k](2^j h + \eps) &= 0, \qtext{and}\\
    \snorm{x_0-x_0'} &\leq \eps < 2^j h + \eps.
\end{talign}
Substituting the bounds \cref{eq:Aj-eps,eq:k-infty,eq:k-empirical-annulus-diff-bound} into the decomposition \cref{eq:k-empirical-annulus-decomp} yields the desired claim.

\subsubsection{\pcref{lem:Linf-L2-kernel-error}}
\label{proof:Linf-L2-kernel-error}

Define the restrictions
\begin{talign}
    g_r(x) = g(x) \boldone_{\ball_2(r;x_0')}(x),
    \qtext{}
    g_r^{(c)} = g - g_r,
    \qtext{}
    \k_r(x,z) \defeq \k(x,z)\cdot \boldone_{\ball_2(r;x_0')}(z), 
    \qtext{and} 
    \k_r^{(c)} \defeq \k -  \k_r,
\end{talign}
so that $\k(z, x_0) = \k_r(z, x_0) + \k_r^{(c)}(z,x_0)$ and $\k(z, x_0') = \k_r(z, x_0') + \k_r^{(c)}(z,x_0')$. We now apply the triangle inequality and \Holder's inequality to obtain
\begin{talign}
    &\abss{\frac{1}{n}\sum_{i=1}^n g(x_i) (\k(x_i,x_0) - \k(x_i,x_0')) \boldone_A(x_i)} \\
    &\qquad = \abss{\frac{1}{n}\sum_{i=1}^n g_r(x_i) (\k_r(x_i,x_0)-\k_r(x_i,x_0'))\boldone_A(x_i) + \frac{1}{n}\sum_{i=1}^n g_r^{(c)}(x_i) (\k_r^{(c)}(x_i,x_0)-\k_r^{(c)}(x_i,x_0')) \boldone_A(x_i)} \\
    &\qquad \leq \abss{\frac{1}{n}\sum_{i=1}^n g_r(x_i) (\k_r(x_i,x_0)-\k_r(x_i,x_0')) \boldone_A(x_i)} \\
    &\qquad \qquad + \abss{\frac{1}{n}\sum_{i=1}^n g_r^{(c)}(x_i) (\k_r^{(c)}(x_i,x_0)-\k_r^{(c)}(x_i,x_0')) \boldone_A(x_i)} \\
    &\qquad \leq \norm{g_r\cdot \boldone_A}_{\lone(\P_n)}\cdot \infnorm{(\k(\cdot,x_0) - \k(\cdot,x_0'))\boldone_A(\cdot)} \label{eq:Linf-L2-kernel-error-step-1a}\\
    &\qquad\qquad + \abss{\frac{1}{n}\sum_{i=1}^n g_r^{(c)}(x_i) (\k(x_i,x_0) - \k(x_i,x_0')) \boldone_A(x_i)}, \label{eq:Linf-L2-kernel-error-step-1b}
\end{talign}
where $\snorm{\cdot}_{\lone(\P_n)}$ is defined by $\snorm{f}_{\lone(\P_n)} \defeq \frac{1}{n}\sum_{i=1}^n \abss{f(x_i)}$.
To bound the term \cref{eq:Linf-L2-kernel-error-step-1a}, we apply Cauchy-Schwarz to $g_r\in L^1(\P_n) \cap L^2(\P_n)$ to obtain
\begin{talign}
    \snorm{g_r\cdot \boldone_A}_{L^1(\P_n)} \leq \snorm{g_r\cdot \boldone_A}_{L^2(\P_n)} \cdot \sqrt{\mass_A(r)} \leq \snorm{g}_{A,L^2(\P_n)} \cdot \sqrt{\mass_A(r)}.
\end{talign}
Next, we bound the term \cref{eq:Linf-L2-kernel-error-step-1b}. For any $x^\star \in \braces{x_0,x_0'}$ and $x_i$ satisfying $\norm{x_i-x_0'}\geq r$ and scalars $a,b\in [0,1]$ such that $a+b=1$, either $\norm{x^\star-x_i}\geq ar$ or $\norm{x^\star-x_0'} > br$. Hence,
\begin{talign}
    &\abss{\frac{1}{n}\sum_{i=1}^n g_r^{(c)} (\k(x_i,x_0) - \k(x_i,x_0')\boldone_A(x_i)} \\
    &\qquad \leq \abss{\frac{1}{n}\sum_{i=1}^n \indic{\snorm{x_i - x_0} \geq r} g_r^{(c)}(x_i) \sum_{x^\star\in \braces{x_0,x_0'}} \k(x_i,x^\star) \boldone_A(x_i) (\delta_{x_0}(x^\star) - \delta_{x_0'}(x^\star))} \\
    &\qquad \leq \underbrace{\abss{\frac{1}{n} \sum_{i=1}^n \sum_{x^\star\in \braces{x_0,x_0'}} \indic{\norm{x_i-x^\star}\geq ar} g_r^{(c)}(x_i) \k(x_i,x^\star) \boldone_A(x_i)(\delta_{x_0}(x^\star)-\delta_{x_0'}(x^\star))}}_{\defeq T_1} \label{eq:Linf-L2-kernel-error-step-2a}\\
    &\qquad \qquad+ \underbrace{\abss{ \frac{1}{n}\sum_{i=1}^n \sum_{x^\star\in \braces{x_0,x_0'}} \indic{\norm{x^\star-x_0} > br} g_r^{(c)}(x_i) \k(x_i,x^\star) \boldone_A(x_i) (\delta_{x_0}(x^\star)-\delta_{x_0'}(x^\star))}}_{\defeq T_2}. \label{eq:Linf-L2-kernel-error-step-2b},
\end{talign}
where we define $\delta_{x_0}(z) \defeq \indic{z=x_0}$ and $\delta_{x_0'}(z) \defeq \indic{z=x_0'}$.

\paragraph{Bounding \cref{eq:Linf-L2-kernel-error-step-2a}.}
Exchanging the order of the summations and applying Cauchy-Schwarz, we have
\begin{talign}
    T_1 &= \abss{\sum_{x^\star\in \braces{x_0,x_0'}} \frac{1}{n} \sum_{i=1}^n \indic{\snorm{x^\star-x_i} \geq ar} g_r^{(c)}(x_i) \k(x^\star, x_i) \boldone_A(x_i) (\delta_{x_0}(x^\star) - \delta_{x_0'}(x^\star))} \\
    &\leq \sum_{x^\star\in \braces{x_0,x_0'}} \abss{\frac{1}{n} \sum_{i=1}^n \indic{\snorm{x_i-x^\star} \geq ar} g_r^{(c)}(x_i) \k(x^\star, x_i) \boldone_A(x_i) (\delta_{x_0}(x^\star) - \delta_{x_0'}(x^\star))} \\
    &\leq \sum_{x^\star\in \braces{x_0,x_0'}} \parenth{\frac{1}{n}\sum_{i=1}^n \indic{\snorm{x_i-x^\star}\geq ar}g_r^{(c)}(x_i)^2 \boldone_A(x_i) }^{\half} \\
    &\qquad \cdot \parenth{\frac{1}{n}\sum_{i=1}^n \indic{\snorm{x_i-x^\star}\geq ar} \k^2(x_i, x^\star) \boldone_A(x_i)}^{\half} \cdot \abss{\delta_{x_0}(x^\star) - \delta_{x_0'}(x^\star)} \\
    &\leq \sum_{x^\star\in \braces{x_0,x_0'}} \parenth{\frac{1}{n}\sum_{i=1}^n \indic{\snorm{x_i-x^\star}\geq ar}g_r^{(c)}(x_i)^2 \boldone_A(x_i) }^{\half} \\
    &\qquad \cdot \sup_{x'\in \braces{x_0,x_0'}}\parenth{\frac{1}{n}\sum_{i=1}^n \indic{\snorm{x_i-x'}\geq ar} \k^2(x_i, x') \boldone_A(x_i)}^{\half} \cdot \abss{\delta_{x_0}(x^\star) - \delta_{x_0'}(x^\star)} \\
    &= \sum_{x^\star\in \braces{x_0,x_0'}} \snorm{g\cdot \boldone_A}_{L^2(\P_n)} \cdot \tail[\k](ar) \cdot \abss{\delta_{x_0}(x^\star) - \delta_{x_0'}(x^\star)} \\
    &=2 \snorm{g\cdot \boldone_A}_{L^2(\P_n)} \cdot \tail[\k](ar)
\end{talign}

\paragraph{Bounding \cref{eq:Linf-L2-kernel-error-step-2b}.}
Rearranging the terms and applying Cauchy-Schwarz, we have
\begin{talign}
    T_2 &= \abss{\sum_{x^\star \in \braces{x_0,x_0'}} \indic{\snorm{x^\star-x_0'} > br} \frac{1}{n}\sum_{i=1}^n g_r^{(c)}(x_i) \k(x_i,x^\star) \boldone_A(x_i) (\delta_{x_0}(x^\star) - \delta_{x_0'}(x^\star)) } \\
    &\leq \sum_{x^\star \in \braces{x_0,x_0'}} \indic{\snorm{x^\star-x_0'} > br} \cdot \abss{\frac{1}{n}\sum_{i=1}^n g_r^{(c)}(x_i) \k(x_i,x^\star) \boldone_A(x_i)} \cdot \abss{\delta_{x_0}(x^\star) - \delta_{x_0'}(x^\star)} \\
    &\leq \sum_{x^\star\in \braces{x_0,x_0'}} \indic{\snorm{x^\star-x_0'} > br} \cdot \snorm{g\cdot \boldone_A}_{L^2(\P_n)} \sup_{x'\in \braces{x_0,x_0'}} \snorm{\k(\cdot,x')\cdot \boldone_A}_{L^2(\P_n)}\\ &\qquad \cdot \abss{\delta_{x_0}(x^\star) - \delta_{x_0'}(x^\star)} \\
    &= \snorm{g\cdot \boldone_A}_{L^2(\P_n)} \sup_{x'\in \braces{x_0,x_0'}} \snorm{\k(\cdot,x')\cdot \boldone_A}_{L^2(\P_n)} \cdot \indic{\snorm{x_0-x_0'} > br}.
\end{talign}

\section{\pcref{thm:kt-krr-finite}}
\label{proof:kt-krr-finite}

We rely on the localized Gaussian/Rademacher analysis of KRR from prior work \citep{wainwright2019high}. 
Define the \emph{Gaussian critical radius} $\vareps_n > 0$ to be the smallest positive solution to the inequality
\begin{talign}\label{eq:critical-gaussian-cond}
    \widehat{\mc G}_n(\vareps;\ball_{\rkhs}(3)) \leq \frac{R}{2\sigma} \vareps^2, \qtext{where} \widehat{\mathcal G}_n(\vareps;\modelclass) \defeq \Esubarg{w}{ \sup_{\substack{ f\in \modelclass: \\ \norm{f}_n\leq \vareps}} \abss{ \frac{1}{n}\sum_{i=1}^n w_i f(x_i) }},
\end{talign}
$\ball_\rkhs(3)$ is the $\knorm{\cdot}$-ball of radius 3 and  $w_i\distiid \Gsn(0,1)$. 

\begin{assumption}\label{assum:kball}
Assume that $\knorm{\fstar} \in \ball_\kernel(R)$ and $\krrkt \in \ball_\kernel(c_{\dagger} R)$, for some constant $c_{\dagger} > 0$.
\end{assumption}
\newcommand{\unifb}{B}
Note that for any $g\in \ball_\kernel(c_{\dagger} R)$, we have
\begin{talign}\label{eq:unif-bounded}
    \infnorm{g} \leq \sup_{x\in \X} \inner{g}{\kernel(\cdot, x)}_\kernel \sless{(i)} \sup_{x\in \X} \knorm{g} \knorm{\kernel(\cdot,x)}\leq \knorm{g} \sqrt{\infnorm{\kernel}} \leq c_{\dagger} R \sqrt{\infnorm{\kernel}} \defeq \unifb,
\end{talign}
where step (i) follows from Cauchy-Schwarz. Thus, the function class $\ball_\kernel(c_{\dagger} R)$ is $B$-uniformly bounded.
Now define the \emph{Rademacher critical radius} $\delta_n>0$ to be the smallest positive solution to the inequality
\begin{talign}\label{eq:critical-rademacher-cond}
    \mc R_n(\delta;\rkhs)\leq \delta^2,\qtext{where} \mathcal R_n(\delta;\modelclass) \defeq \Esubarg{x,\nu}{ \sup_{\substack{f\in \modelclass: \\ \twonorm{f} \leq \delta}} \abss{ \frac{1}{n}\sum_{i=1}^n \nu_i f(x_i)} }
\end{talign}
and $\nu_i= \pm 1$ each with probability $1/2$.

Finally, we use the following shorthand to control the KT approximation error term, 
\begin{talign}\label{eq:kt-complexity}
    \ktcomplexity &\defeq \frac{\ktcomplexityconstant^2}{\nout} (2 + \errmmd_{\kernel}(n,\nout,\delta, \Rin, \frac{\ktcomplexityconstant}{\nout})), \qtext{where}  \\
    \mfk R_{\mrm{in}} &\defeq \max_{x\in \inputcoreset} \twonorm{x} \qtext{and} \ktcomplexityconstant \defeq \ininfnorm\kernel + \ymax^2\label{eq:radius}
\end{talign}
and $\errmmd_\kernel$ is an \textit{inflaction factor}
defined in \cref{eq:err-mmd} that scales with the covering number $\coveringnumber_{\kernel}$ (see \cref{def:cover}). 
With these definitions in place, we are ready to state a detailed version of \cref{thm:kt-krr-finite}:

\begin{theorem}[KT-KRR for finite-dimensional RKHS, detailed] \label{thm:kt-krr-finite-detailed}
Suppose the kernel operator associated with $\kernel$ and $\P$ has eigenvalues $\mu_1\geq \ldots \geq\mu_m>0$ (by Mercer's theorem). Define $C_m  \defeq 1/\mu_m$.
Let $\vareps_n$ and $\delta_n$ denote the solutions to \cref{eq:critical-rademacher-cond} and \cref{eq:critical-gaussian-cond}, respectively. Further assume \footnote{Note that when $\kernel$ is finite-rank, this condition is automatically satisfied.}
\begin{talign}
    n\delta_n^2 > \log ( 4 \log (1/\delta_n) ).
    \label{eq:delta_cond}
\end{talign}
Let $\krrkt$ denote the \krrktname estimator with regularization parameter 
\begin{talign}\label{eq:lambdakt}
    \lambdakt \geq 2 \critthresh^2 \qtext{where} \critthresh \defeq \vareps_n \vee \delta_n \vee 4 \sqrt C_m (\knorm{\fstar} +1)\ktcomplexity.
\end{talign}
Then with probability at least $1-2\delta - 2 e^{-\frac{n\delta_n^2}{c_1 (b^2+\sigma^2) }}$, we have
\begin{talign}\label{eq:krr-kt-L2-bound-finite}
    \statictwonorm{\krrkt - \fstar}^2 \leq c \braces{\critthresh^2 + \lambdakt} \knorm{\fstar}^2 + c\delta_n^2.
\end{talign}
where recall $\delta$ is the success probability of \ktcompresspp \cref{eq:ktgoodgen}.
\end{theorem}
\proofref{proof:kt-krr-finite-detailed}
We set $\lambdakt = 2\critthresh^2$, so that \cref{eq:krr-kt-L2-bound-finite} becomes
\begin{talign}\label{eq:krr-kt-L2-bound-finite-1}
    \statictwonorm{\krrkt - \fstar}^2 \leq 3c\critthresh^2 \knorm{\fstar}^2 + c\delta_n^2.
\end{talign}
It remains to bound the quantities $\vareps_n$ \cref{eq:critical-gaussian-cond}, $\delta_n$ \cref{eq:critical-rademacher-cond}, and $\ktcomplexity$ \cref{eq:kt-complexity}. We claim that
\begin{talign}
    \vareps_n &\leq c_0 \frac{\sigma}{R} \sqrt{\frac{m}{n}} \label{eq:gsn-crit-radius-finite} \\
    \delta_n &\leq c_1 b \sqrt{\frac{m}{n}} \label{eq:rad-crit-radius-finite} \\
    \ktcomplexity &\leq c_2 \frac{\sqrt{ m \cdot \log \nout \cdot \log(1/\delta)}}{\nout}. \label{eq:kt-complexity-finite} 
\end{talign}
for some universal positive constants $c_0,c_1, c_2$.
Now set
\begin{talign}
    R=\knorm{\fstar} \qtext{} \delta = e^{-1/R^4}.
\end{talign}
Thus, we have
\begin{talign}\label{eq:critthresh-explicit}
    \critthresh \leq c' (\frac{\sigma}{\knorm{\fstar}} \vee b \vee \frac{4 \sqrt{C_m}}{\knorm{\fstar}}) \frac{\sqrt m}{\sqrt n \wedge \nout}.
\end{talign}
for some universal positive constant $c'$. Substituting this into \cref{eq:krr-kt-L2-bound-finite-1} leads to the advertised bound \cref{eq:krr-kt-L2-bound-finite-simplified}.

\paragraph{Proof of claim~\cref{eq:gsn-crit-radius-finite}.}
For finite rank kernels, $\hat \mu_j = 0$ for $j>m$. Thus, we have $\sqrt{\frac{2}{n}}\sqrt{\sumn[j] \min\braces{\vareps^2, \hat \mu_j}} = \sqrt{\frac{2}{n}} \sqrt{m\vareps^2}$. From the critical radius condition \cref{eq:critical-gaussian-cond}, we want $\sqrt{\frac{2}{n}} \sqrt{m\vareps^2} \leq \frac{R}{4\sigma} \vareps^2$, so we may set $\vareps_n \simeq \frac{\sigma}{R} \sqrt{\frac{m}{n}}$.

\paragraph{Proof of claim~\cref{eq:rad-crit-radius-finite}.}
By similar logic as above, we have $\sqrt{\frac{2}{n}}\sqrt{\sumn[j] \min\braces{\delta^2, \mu_j}} = \sqrt{\frac{2}{n}} \sqrt{m \delta^2}$. From the critical radius condition \cref{eq:critical-rademacher-cond}, we want $\sqrt{\frac{2}{n}} \sqrt{m \delta^2} \leq \frac{1}{b} \delta^2$, so we may set $\delta_n \simeq b \sqrt 2 \sqrt{\frac{m}{n}}$.

\paragraph{Proof of claim~\cref{eq:kt-complexity-finite}.}
Consider the linear operator $T: \mc H \to \reals^m$ that maps a function to the coefficients in the vector space spanned by $\braces{\phi_i}_{i=1}^m$. Note that
\begin{talign}
    \norm{T} = \frac{\infnorm{Tf}}{\knorm{f}}
     \leq \sqrt{\sinfnorm{\kernel}}
\end{talign}
Since the image of $T$ has dimension $m$, we have $\rank{T}\leq m$. Moreover, $\sinfnorm{\kernel} \leq \mu_1 \cdot \Rin^2$.
Now we can invoke \citep[Eq.~14]{steinwart2021closer} with $\eps = \ktcomplexityconstant / \nout$ to obtain
\begin{talign}
    \coveringnumber_\kernel(\ball_2^d(\Rin),\ktcomplexityconstant / \nout) \leq \coveringnumber(T, \ktcomplexityconstant/\nout) \leq (1+ \mu_1 \Rin^2 \nout / \ktcomplexityconstant)^m.
\end{talign}
Taking the log on both sides and substituting this bound into \cref{eq:kt-complexity}, we have
\begin{talign}
    \ktcomplexity &= \frac{\ktcomplexityconstant^2}{\nout} (2 + \errmmd_{\kernel}(n,\nout,\delta, \Rin,\frac{\ktcomplexityconstant}{\nout})) \\
    &\leq \frac{\ktcomplexityconstant^2}{\nout} (2 + \sqrt{\log\parenth{\frac{\nout \log(n/\nout)}{\delta}}  \cdot \brackets{\log\parenth{\frac{1}{\delta}} + \log \coveringnumber_{\kernel}(\ball_2^d(\Rin), \frac{\ktcomplexityconstant}{\nout}) }}) \\
    &\leq \frac{\ktcomplexityconstant^2}{\nout} (2 + \sqrt{\log\parenth{\frac{\nout \log(n/\nout)}{\delta}}  \cdot \brackets{\log\parenth{\frac{1}{\delta}} + 
    m \log\parenth{1+\frac{2\norm{T} \nout}{\ktcomplexityconstant}}}}) \\
    &\leq c \frac{\sqrt{m \cdot \log \nout \cdot \log(1/\delta)}}{\nout}
\end{talign}
for some positive constant $c$ that doesn't depend on $m,\nout,\delta$.

\section{\pcref{thm:kt-krr-finite-detailed}}
\label{proof:kt-krr-finite-detailed}

\newcommand{\hadarkhs}{\rkhs\odot \rkhs}
\newcommand{\twonnorm}[1]{\Vert{#1}\Vert_{n, 2}}
\newcommand{\twonoutnorm}[1]{\Vert{#1}\Vert_{\nout, 2}}
\newcommand{\twooutnorm}[1]{\Vert{#1}\Vert_{\nout}}
\newcommand{\twonormlowerbound}{\mc E_{\trm{lower}}}
\newcommand{\unifconc}{\mc E_{\trm{conc}} }
\newcommand{\badeventA}{\mc A_{\mrm{KT}}}
\newcommand{\badeventB}{\mc B_{\mrm{KT}}}
\newcommand{\goodevent}{\mc E_{\trm{good}}}

We rescale our observation model \cref{eq:process} by $\knorm{\fstar}$, so that the noise variance is $\parenth{\sigma/\knorm{\fstar}}^2$ and our new regression function satisfies $\knorm{\fstar} = 1$. Our final prediction error should then be multiplied by $\knorm{\fstar}^2$ to recover a result for the original problem. For simplicity, denote 
\begin{talign}\label{eq:wtil-sigma}
    \widetilde \sigma = \sigma/\knorm{\fstar}.
\end{talign}
For notational convenience, define an event
\begin{talign}
    \mc E = \bigbraces{\stwonorm{ \krrkt - \fstar }^2 \leq c(\xi_n^2 + \lambda') },
\end{talign}
and our goal is to show that $\mc E$ occurs with high-probability in terms of $\Prob$, the probability regarding all the randomness.
For that end, we introduce several events that are used throughout, 
\begin{talign}
    \unifconc \defeq \biggbraces{ \sup_{g \in \rkhs} \big| \staticnnorm{g} - \twonorm{g} \big| \leq   \frac{\delta_n}{2} } \qtext{and} \twonormlowerbound \defeq \braces{\statictwonorm{\krrkt - \fstar} > \delta_n},
\label{eq:unif_conc_event}
\end{talign}
where $\delta_n$ is defined in \cref{eq:critical-rademacher-cond} and $\rkhs$ is the RKHS generated by $\kernel$ hence star-shaped. Further, we introduce two technical events $\badeventA(u), \badeventB$ defined in \cref{eq:control,eq:badeventB} respectively, which are proven to occur with small probability, and define a shorthand
\begin{talign}\label{eq:goodevent}
    \goodevent \defeq \badeventA^c(\xi_n) \cap \badeventB^c \cap \unifconc\cap \ktgoodgen.
\end{talign}
Equipped with these shorthands, observe the following inequality,
\begin{talign}
    \Prob( \mc E ) &= \Prob( \mc E \cap \twonormlowerbound ) + \Prob( \mc E \cap \twonormlowerbound^c )\\
    &\geq \Prob( \mc E \cap \twonormlowerbound) + \Prob( \twonormlowerbound^c )
    \label{eq:first_prob_chain}
\end{talign}
where the second inequality is because $\twonormlowerbound^c \subseteq \twonormlowerbound^c \cap \mc E$ due to the assumption $\lambda' \geq 2\xi_n^2 \geq 2\delta_n^2$.

If we are able to show the set inclusion $\sbraces{ \goodevent \cap \twonormlowerbound } \subseteq \sbraces{ \mc E\cap \twonormlowerbound }$ and that $\Prob( \goodevent^c )$ is small,
we are able refine \cref{eq:first_prob_chain} to the following
\begin{talign}
    \Prob( \mc E ) \geq \Prob( \goodevent \cap \twonormlowerbound ) + \Prob(\twonormlowerbound^c) \geq 1 - \Prob( \goodevent^c ) - \Prob(\twonormlowerbound^c) +\Prob(\twonormlowerbound) = 1 - \Prob(\goodevent^c),
\end{talign}
where the last quantity $1 - \Prob(\goodevent^c)$ would be large.

To complete this proof strategy, we claim the set inclusion
\begin{talign}\label{eq:claim_first_ing}
\sbraces{ \goodevent \cap \twonormlowerbound } \subseteq \sbraces{ \mc E\cap \twonormlowerbound }    
\end{talign}
to hold and prove it in \cref{proof:first_ingredient} and further claim 
\begin{talign}\label{eq:claim_second_ing}
    \Prob( \goodevent^c ) \leq c''\bigbraces{\delta + e^{-c' n \delta_n^2/( B_{\rkhs}^2 \wedge \widetilde \sigma^2 )} }
\end{talign}
which verify in \cref{proof:second_ingredient}.

Putting the pieces together, claims \cref{eq:claim_first_ing,eq:claim_second_ing} collectively implies
\begin{talign}
    \Prob( \mc E ) \geq 1 - c''\bigbraces{\delta + e^{-c' n \delta_n^2/( B_{\rkhs}^2 \wedge \widetilde \sigma^2 )} }
\end{talign}
as desired.

\subsection{Proof of claim \cref{eq:claim_first_ing}}\label{proof:first_ingredient}

There are several intermediary steps we take to show the set inclusion of interest \cref{eq:claim_first_ing}. We introduce the shorthand 
\begin{talign}
    \deltakt \defeq \krrkt - \fstar.
\end{talign}

By invoking Propositions and basic inequalities to come, we successively show the following chain of set inclusions 
\begin{talign}
    \goodevent \cap \twonormlowerbound &\subseteq \goodevent \cap \twonormlowerbound \cap \sbraces{ \snoutnorm{ \deltakt }^2 \leq c (\xi_n^2 + \lambda') } \label{step_a}\\
    &\subseteq \goodevent \cap \twonormlowerbound \cap \sbraces{ \snnorm{ \deltakt }^2 \leq c (\xi_n^2 + \lambda') } \label{step_b}\\
    &\subseteq \goodevent \cap \twonormlowerbound \cap \mc E \label{step_c}\\
    &\subseteq \twonormlowerbound \cap \mc E \label{step_d}.
\end{talign}
Note that step \cref{step_d} is achieved trivially by dropping $\goodevent$. Further note that \cref{step_b} is the crucial intermediary step after which we may apply uniform concentration across $n$ independent samples. Proof of \cref{step_b} leverages on the Proposition to come~(\cref{prop:rel-guarantee}) that allows $\snnorm{\cdot}$ and $\snoutnorm{\cdot}$ to be exchangeable for finite rank kernels. 

\paragraph{Recovering step \cref{step_a}} Since $\krrkt$ and $\fstar$ are optimal and feasible, respectively for the central optimization problem of interest
\begin{talign}
    \min_{f\in \rkhs(\kernel)} \frac{1}{\nout} \sum_{i=1}^{\nout} \parenth{y_i'- f (x_i')}^2 + \lambda' \knorm{f}^2,
\end{talign}
we have the basic inequality
\begin{talign}
    \frac{1}{\nout} \sum_{i=1}^{\nout} \parenth{y_i'-\krrkt (x_i')}^2 + \lambdakt \knorm{\krrkt}^2 &\leq \frac{1}{\nout} \sum_{i=1}^{\nout} \parenth{y_i'-\fstar (x_i')}^2 + \lambdakt \knorm{\fstar}^2,
    \label{eq:basic_ineq}
\end{talign}

With some algebra
, may refine \cref{eq:basic_ineq} to 
\begin{talign}
    \half \snoutnorm{\deltakt}^2 &\leq \Big|\frac{1}{\nout} \sum_{i=1}^{\nout} \noise_i' \deltakt(x_i')\Big| + \lambda \braces{ \knorm{\fstar}^2 - \knorm{\krrkt}^2 } .
    \label{eq:basic_ineq_root}
\end{talign}
where $\deltakt = \krrkt-\fstar$.Suppose that $\snoutnorm{\deltakt} < \critthresh$, then we trivially recover \cref{step_a} by adding $\lambdakt>0$. Thus, we assume that $\snoutnorm{\deltakt} \geq \critthresh$.

Under the assumption $\snoutnorm{\deltakt} \geq \critthresh$, which is without loss of generality, we utilize the basic inequality \cref{eq:basic_ineq_root} and control its stochastic component 
\begin{talign}
    \bigg|\frac{1}{\nout} \sum_{i=1}^{\nout} \noise_i' \deltakt(x_i')\bigg|,
\end{talign}
with a careful case work to follow, which is technical by nature. 

\underline{Case where $\knorm{ \krrkt } \leq 2$}: Under such case, we introduce a technical event
\begin{talign}
    \badeventA (u) &\defeq \bigg\{ \exists g \in \mc F \setminus \ball_{2}(\delta_n) \cap \{ \staticnoutnorm{g}  \geq u \}~\text{such that}~\Big| \frac{1}{\nout} \sum_{i = 1}^{\nout} \noise_i' g(x_i') \Big| \geq 3 \staticnoutnorm{g} u  \bigg\},
\end{talign}
for any star-shaped function class $\mc F\subset \mc H$. Since $\knorm{\fstar} =1$, triangle inequality implies $\knorm{\deltakt} \leq \knorm{\krrkt} + \knorm{\fstar} \leq 3$. Moreover, on the event $\twonormlowerbound$~($\subseteq \goodevent$), we have $\stwonorm{\deltakt} > \delta_n$. Thus, we may apply $\deltakt$ to the event $\badeventA^c(\critthresh)$ with $\mc F = \ball_\rkhs(3)$ (i.e., the $\rkhs$-ball of radius 3) to attain
\begin{talign}
    \Big|\frac{1}{\nout} \sum_{i=1}^{\nout} \noise_i' \deltakt(x_i')\Big| \leq c_0 \critthresh \snoutnorm{\deltakt}~ \text{on the event} ~ \badeventA^c(\critthresh) \cap \twonormlowerbound.
    \label{eq:first_stoch_comp_upper}
\end{talign} 
Upper bounding the stochastic component of the basic inequality \cref{eq:basic_ineq_root} by \cref{eq:first_stoch_comp_upper} and dropping the $-\knorm{\krrkt}^2$ term in \cref{eq:basic_ineq_root}, we have
\begin{talign}
    \half \snoutnorm{\deltakt}^2 \leq c_0 \critthresh \snoutnorm{\deltakt} + \lambdakt.
\end{talign}
As a last step under the case $\knorm{ \krrkt } \leq 2$, apply the quadratic formula (specifically, if $a,b\geq 0$ and $x^2 - ax - b\leq 0$, then $x\leq a^2 + b$) to obtain
\begin{talign}
    \snoutnorm{\deltakt}^2 \leq 4 c_0^2 \critthresh^2 + 2\lambdakt.
\end{talign}

\underline{Case where $\knorm{ \krrkt } > 2$}: Under such case, by assumption we have $\knorm{\krrkt} >2 >1 \geq \knorm{\fstar}$. Thus, we may derive the following
\begin{talign}
    \knorm{\fstar} - \knorm{\krrkt}  < 0 \qtext{and}\knorm{\fstar} + \knorm{\krrkt}  >1,
\end{talign}
which further implies the following inequality
\begin{talign}
    \knorm{\fstar}^2 - \knorm{\krrkt}^2 = \sbraces{ \knorm{\fstar} - \knorm{\krrkt} }  \sbraces{ \knorm{\fstar} +\knorm{\krrkt} } \leq \knorm{\fstar} - \knorm{\krrkt} .
    \label{eq:basic_sec_case}
\end{talign}
Further writing $\krrkt = \fstar + \deltakt$ and noting that $\knorm{\deltakt} - \knorm{\fstar} \leq \knorm{\krrkt}$ holds through triangle inequality, we may further refine \cref{eq:basic_sec_case} as
\begin{talign}
    \knorm{\fstar} - \knorm{\krrkt} \leq 2 \knorm{\fstar} - \knorm{\deltakt} \leq 2 - \knorm{\deltakt},
\end{talign}
so that the basic inequality in \cref{eq:basic_ineq_root} reduces to 
\begin{talign}\label{eq:mod-basic-ineq_good}
    \half \snoutnorm{\deltakt}^2 \leq \Big|\frac{1}{\nout} \sum_{i=1}^{\nout} \noise_i' \deltakt(x_i')\Big| + \lambdakt \sbraces{ 2 - \knorm{\deltakt} }.
\end{talign}
We again introduce a technical event that controls the stochastic component of \cref{eq:mod-basic-ineq_good}, which is
\begin{talign}\label{eq:badeventB}
    \badeventB \defeq \bigg\{ \exists g \in \mc F \setminus \ball_{2}(\delta_n) \cap \sbraces{\knorm{g} \geq 1} &: \\
    \Big| \frac{1}{\nout} \sum_{i = 1}^{\nout} \noise_i' g(x_i') \Big| &> 4 \xi_n \snoutnorm{g} + 2\xi_n^2 \knorm{g} + \frac{1}{4} \snoutnorm{ g }^2\bigg\},
\end{talign}
for a star-shaped function class $\mc F \subset \rkhs$.

By triangle inequality, we have $\knorm{\deltakt} \geq \knorm{\krrkt} - \knorm{\fstar}> 1$, and on event $\twonormlowerbound$~($\subset \goodevent$), we have $\stwonorm{\deltakt} > \delta_n$. Thus, we may apply $g = \deltakt$ to the event $\badeventB^c$, and the resulting refined basic inequality is
\begin{talign}
    \half \snoutnorm{\deltakt}^2 &\leq 4\xi_n \staticnoutnorm{\deltakt} +(2\critthresh^2 - \lambdakt) \knorm{\deltakt} + 2\lambdakt\\
    &\leq 4\xi_n \staticnoutnorm{\deltakt}+ 2\lambdakt \quad \text{on the event}~ \badeventB^c \cap \twonormlowerbound
\end{talign}
where the second inequality is due to the assumption that $\lambdakt \geq 2\critthresh^2$.
We apply the quadratic formula (specifically, if $a,b\geq 0$ and $x^2 - ax - b\leq 0$, then $x\leq a^2 + b$) to obtain
\begin{talign}
    \snoutnorm{\deltakt}^2 \leq 4 c_0^2 \critthresh^2 + 2\lambdakt.
\end{talign}
Putting the pieces together, we have shown
\begin{talign}
    \snoutnorm{\deltakt}^2 \leq c (\critthresh^2 +\lambda') \quad \text{on the event} ~ \badeventA^c(\critthresh) \cap \badeventB^c \cap \twonormlowerbound,
\end{talign}
which is sufficient to recover \cref{step_a}.

\paragraph{Recovering step \cref{step_b}}

We now upgrade events
\begin{talign}
    \sbraces{\snoutnorm{\deltakt}^2 \leq c (\critthresh^2 +\lambda')} \implies \sbraces{\snnorm{\deltakt}^2 \leq c' (\critthresh^2 +\lambda')}
\end{talign}
by exploiting the events $\unifconc \cap \ktgoodgen$~(subset of $\goodevent$) that were otherwise not used when recovering \cref{step_a}. For this end, the following result is a crucial ingredient, which shows that $\noutnorm{\cdot}$ and $\nnorm{\cdot}$ are essentially exchangeable with high-probability, 

\begin{proposition}[Multiplicative guarantee for \ktcompresspp with $\kernelrr$]\label{prop:rel-guarantee}
Let $C_m \defeq 1/\mu_m$ and suppose $\delta_n$ satisfies \cref{eq:delta_cond}. Then on event $\ktgoodgen \cap \unifconc$, where $\ktgoodgen$ and $\unifconc$ are defined in \cref{eq:ktgoodgen} and \cref{eq:unif_conc_event} respectively, we have
\begin{talign}
     (1 - 4 C_m\cdot\ktcomplexity) \staticnnorm{g} \leq \staticnoutnorm{g} \leq (1 + 4 C_m\cdot\ktcomplexity) \staticnnorm{g} \label{eq:rel_gaurantee}
\end{talign}
uniformly over all $g \in \rkhs$ such that $\twonorm{g} > \delta_n$.
\end{proposition}
\proofref{proof:rel-guarantee}

An immediate consequence of \cref{prop:rel-guarantee} is that 
\begin{align}
    \sbraces{\snoutnorm{\deltakt}^2 \leq c (\critthresh^2 +\lambda')} \implies \sbraces{\snnorm{\deltakt}^2 \leq c' (\critthresh^2 +\lambda')}
\end{align}
on the event $\ktgoodgen \cap \unifconc \cap \twonormlowerbound,$, which is sufficient to recover \cref{step_b}.

\paragraph{Recovering step \cref{step_c}}

Our last step is to show 
\begin{talign}
    \sbraces{\snnorm{\deltakt}^2 \leq c' (\critthresh^2 +\lambda')} \implies \sbraces{\stwonorm{\deltakt}^2 \leq c'' (\critthresh^2 +\lambda')}.
\end{talign}
Such result can be immediately shown on the event $\unifconc$ by observing that $\deltakt \in \mc H$, by our assumption that $\fstar \in \mc H$ and by the definition
\begin{talign}
    \krrkt \in \argmin_{f\in \rkhs(\kernel)} L_{\nout}(f) + \lambda' \knorm{f}^2.
\end{talign}

\subsection{Proof of claim \cref{eq:claim_second_ing}}\label{proof:second_ingredient}

It suffices to show  the appropriate bounds for the following four probability terms 
\begin{talign}
    \Prob( \badeventA(\critthresh) ), \quad \Prob( \badeventB ), \quad \Prob(\unifconc^c) ,  \quad \Prob(\ktgoodgen^c).
\end{talign}
Fix the shorthand 
\begin{talign}
    B_\rkhs \defeq \sinfnorm{\kernel}^2 R^2 < \infty.
\end{talign}
We know from \cite{dwivedi2024kernel} that $\Prob( \ktgoodgen^c |\inputcoreset ) \leq \delta$ and then we may apply \cite[Thm.~14.1]{wainwright2019high} to obtain a high probability statement,
\begin{talign}\label{eq:unifconc-prob-bound}
    \Parg{\unifconc^c} \leq e^{ -c' n\delta_n^2/ B_\rkhs^2 }.
\end{talign}

Now we present two Lemmas that bound the $\Prob( \cdot \mid \inputcoreset )$ probability of events $\badeventA(\critthresh)$ and $\badeventB$,

\begin{lemma}[Controlling bad event when $\knorm{\krrkt} \leq 2$]\label{lem:case-I}
Suppose $u\geq \critthresh$.
Then for some constant $c > 0$,
\begin{talign}\label{eq:case-I-bound}
\Parg{\badeventA (u) \mid \inputcoreset}\leq \delta +  e^{-c n\delta_n^2/B_{\rkhs}^2} + e^{-c n u^2/\widetilde \sigma^2}
\end{talign}
where $\widetilde \sigma = \sigma/\knorm{\fstar}$.
\end{lemma}
\proofref{proof:case-I} Note that by plugging in $\critthresh$ into \cref{eq:case-I-bound} results in a probability that depends on $\inputcoreset$~(as $\critthresh$ depends on $\inputcoreset$). By invoking the definition of $\critthresh$, we may further refine the probability bound of $\badeventA(\critthresh)$ by
\begin{talign}\label{eq:case-I-bound-refined}
    \Parg{\badeventA (\critthresh) \mid \inputcoreset}\leq \delta +  e^{-c n\delta_n^2/B_{\rkhs}^2} + e^{-c n \delta_n^2/\widetilde \sigma^2}
\end{talign}

\begin{lemma}[Controlling bad event when $\knorm{\krrkt} > 2$]\label{lem:case-II}
For some constants $c ,  c' > 0$,
\begin{talign}\label{eq:case-II-bound}
    \Parg{\badeventB\mid\inputcoreset} \leq  \delta + e^{-cn\delta_n^2/B_{\rkhs}^2} + ce^{-n \critthresh^2/(c'\widetilde \sigma^2)}
\end{talign}
where $\widetilde \sigma = \sigma/\knorm{\fstar}$.
\end{lemma}
\proofref{proof:case-II} It is notable that $\critthresh$ in the probability bound of \cref{eq:case-II-bound} contains a term $\vareps_n$ defined in \cref{eq:critical-gaussian-cond} that is a function of $\inputcoreset$. Invoking the definition of $\critthresh$, we observe the probability upper bound \cref{eq:case-II-bound} can be refined to
\begin{talign}\label{eq:case-II-bound-refined}
    \Parg{\badeventB\mid\inputcoreset} \leq  \delta + e^{-cn\delta_n^2/B_{\rkhs}^2} + ce^{-n \delta_n^2/(c'\widetilde \sigma^2)},
\end{talign}
which does not depend on $\inputcoreset$.

Putting the pieces together, we have the following probability bound for some constants $c, c' >0$, 
\begin{talign}
    \Prob( \goodevent^c \mid \inputcoreset ) &\leq \Prob( \badeventA(\critthresh) \mid \inputcoreset ) + \Prob( \badeventB\mid \inputcoreset ) + \Prob(\unifconc^c\mid \inputcoreset) + \Prob(\ktgoodgen^c\mid \inputcoreset) \\
    &\sless{\cref{eq:unifconc-prob-bound}\cref{eq:case-I-bound-refined}\cref{eq:case-II-bound-refined}} c \bigbraces{\delta + e^{-c' n \delta_n^2/( B_{\rkhs}^2 \wedge \widetilde \sigma^2 )} }
\end{talign}
thereby implying $\Prob( \goodevent^c  ) \leq c''\bigbraces{\delta + e^{-c' n \delta_n^2/( B_{\rkhs}^2 \wedge \widetilde \sigma^2 )} } $ for some constant $c''$.

\subsection{\pcref{prop:rel-guarantee}}
\label{proof:rel-guarantee}

Fix $g\in \rkhs$. Denote $\inner{g}{h} = \int g(x) h(x) dx$ as the inner product in the $L^2$ sense. By Mercer's theorem \cite[Cor.~12.26]{wainwright2019high}, the $\kernel$-norm of $g$ has a basis expansion $\knorm{g}^2 = \sum_{i=1}^{m} \inner{g}{\phi_i}^2/\lambda_i$ so that
\begin{talign}\label{eq:knorm-twonorm}
\knorm{g}^2
\leq \sum_{i=1}^{m} \inner{g}{\phi_i}^2/\lambda_{m}
= C_m \twonorm{g}^2
\qtext{since} C_m = 1/\lambda_m.
\end{talign}

The assumption $\twonorm{ g } \geq \delta_n$ implies that on the event $\mc E_{\trm{conc}}$ \cref{eq:unif_conc_event}, we have
\begin{talign}
     \half \delta_n \leq \twonorm{g} - \half\delta_n \leq \staticnnorm{g}
     \label{eq:unif_conc_result}
\end{talign}
Moreover, $g$ must be a non-zero function.
Note that $g^2 \in \rkhs(\kernelrr)$ (see \cref{proof:loss-diff}). 
Thus, we may apply \cref{lem:loss-diff} to $f_1=f_2 = g$ and $a=1,b=0$ to obtain
\begin{talign}
    \abss{\staticnnorm{g}^2 - \staticnoutnorm{g}^2} = \bigg| \frac{1}{n}\sum_{i = 1}^{n} g^2(x_i) - \frac{1}{\nout} \sum_{i = 1}^{\nout} g^2(x_i') \bigg| \leq \knorm{g}^2 \cdot  \ktcomplexity.
    \label{eq:gen_KT_hada}
\end{talign}
The LHS can be expanded as 
\begin{talign}
    \big| \staticnnorm{g}^2 - \staticnoutnorm{g}^2 \big| = \big| \staticnnorm{g} - \staticnoutnorm{g} \big| \cdot \underbrace{\big| \staticnnorm{g} + \staticnoutnorm{g} \big|}_{> 0\text{ by \cref{eq:unif_conc_result}}}. 
\end{talign}
Thus, we may rearrange \cref{eq:gen_KT_hada} and combine with \cref{eq:knorm-twonorm} to obtain
\begin{talign}
    \big| \staticnnorm{g} - \staticnoutnorm{g} \big| \leq \frac{C_m\twonorm{g}^2}{\staticnnorm{g} + \staticnoutnorm{g}}\cdot\ktcomplexity.
    \label{eq:gen_KT_hada_improved}
\end{talign}

On event $\unifconc$, we have 
\begin{talign}\label{eq:twonorm-g-upper-bound}
    \twonorm{g}^2 \sless{\cref{eq:unif_conc_event}} (\half\delta_n + \staticnnorm{g})^2 \sless{(i)} \frac{\delta_n^2}{4} +  \delta_n  \staticnnorm{g} + \staticnnorm{g}^2.
\end{talign}
Thus, we have 
\begin{talign}
    \frac{\twonorm{g}^2}{\staticnnorm{g} + \staticnoutnorm{g} } &\sless{\cref{eq:twonorm-g-upper-bound}} \frac{\delta_n^2}{4\abss{\staticnnorm{g} + \staticnoutnorm{g}}} + \frac{\delta_n \staticnnorm{g}}{\staticnnorm{g} + \staticnoutnorm{g}} + \frac{\staticnnorm{g}^2}{\staticnnorm{g} + \staticnoutnorm{g}}\\
    &\sless{\cref{eq:unif_conc_result}} \frac{\delta_n^2}{2\delta_n} + \delta_n + \staticnnorm{g} \cdot \frac{\staticnnorm{g}}{\staticnnorm{g} + \staticnoutnorm{g}} \\
    &\leq \frac32 \delta_n +  \staticnnorm{g} \\
    &\sless{\cref{eq:unif_conc_result}} 3 \staticnnorm{g} +  \staticnnorm{g} = 4 \staticnnorm{g}.
    \label{eq:last_step}
\end{talign}

Using \cref{eq:last_step} to refine \cref{eq:gen_KT_hada_improved}, we have on event $\ktgoodgen \cap \unifconc$:
\begin{talign}
    \big| \staticnnorm{g} - \staticnoutnorm{g} \big|   \leq 4 C_m \staticnnorm{g} \cdot \ktcomplexity.
    \label{eq:gen_KT_hada_final}
\end{talign}
With some algebra, this implies with probability at least $1 - \delta - \exp(-c' n\delta_n^2/B_{\mc F}^2)$:
\begin{talign}
    (1 - 4 C_m\cdot \ktcomplexity) \staticnnorm{g} \leq \staticnoutnorm{g} \leq (1 + 4 C_m\cdot \ktcomplexity) \staticnnorm{g}
\end{talign}
uniformly over all non-zero $g \in \rkhs$ such that $\twonorm{g} > \delta_n$.

\subsection{\pcref{lem:case-I}}\label{proof:case-I}

Recall $\ktgoodgen$ and $\unifconc$ defined by \cref{eq:ktgoodgen} and \cref{eq:unif_conc_event}. Also recall that $\ktgoodgen \cap \unifconc$ combined with the assumption $\twonorm{g} \geq \delta_n$ invokes the event \cref{eq:rel_gaurantee}. 
Our aim is to show 
\begin{talign}\label{eq:control}
 \badeventA(u)\cap \ktgoodgen \cap \unifconc \subseteq \sbraces{ Z_n(2u) \geq 2 u^2 } ,\qtext{where} Z_n(t) \defeq \sup_{\substack{g\in \modelclass: \\ \snnorm{g} \leq t}} \abss{\frac{\wtil \sigma}{n}\sumn w_i g(x_i)}
\end{talign}
so that we have a probability bound
\begin{talign}
    \Prob( \badeventA(u) ) \leq \Prob( \ktgoodgen^c ) + \Prob( \unifconc^c ) + \Prob( Z_n(2u) \geq 2 u^2 ).
\end{talign}
The first RHS term can be bounded by $\delta$ (see \cref{eq:ktgoodgen}). The second RHS term can bounded by \cref{eq:unifconc-prob-bound}. The third term can be bounded by
\begin{talign}
    \Prob( Z_n(2u)\geq 2 u^2 ) = \Parg{Z_n(u) \geq u^2/2 + u^2/2} &\sless{(i)} \Parg{Z_n(u) \geq u \vareps_n /2 + u^2/2}
    \sless{(ii)} e^{-\frac{n u^2}{8 \wtil \sigma^2}},
\end{talign}
where (i) follows from our assumption that $u\geq \vareps_n$ and (ii) follows from applying generic concentration bounds on $Z_n(u)$ (see \citep[Thm.~2.26,~Eq.~13.66]{wainwright2019high}).
Putting together the pieces yields our desired probability bound \cref{eq:case-I-bound}.

\paragraph{Proof of claim~\cref{eq:control}.}

Consider the event $\badeventA(u)\cap \ktgoodgen \cap \unifconc$. The norm equivalence established on the event $\ktgoodgen \cap \unifconc$ in \cref{prop:rel-guarantee} is an important ingredient throughout. 

Let $g\in \rkhs$ be the function that satisfies three conditions: $\twonorm{g} \geq \delta_n$
, $\staticnoutnorm{g} \geq u$, and
\begin{talign}
    \bigg| \frac{1}{\nout} \sum_{i = 1}^{\nout} \noise_i' g(x_i')\bigg| \geq 3 \staticnoutnorm{g}u.
\end{talign}
Define the normalized function
\begin{talign}\label{eq:g-til}
    \widetilde g = u\cdot g/ \staticnoutnorm{g}
\end{talign}
so that it satisfies $\staticnoutnorm{\widetilde g} = u$ and also
\begin{talign}
    \bigg| \frac{1}{\nout} \sum_{i = 1}^{\nout} \noise_i' \widetilde g(x_i') \bigg| \geq 3 u^2.
    \label{eq:zn_inter}
\end{talign}
By triangle inequality, the LHS of \cref{eq:zn_inter} can be further upper bounded by
\begin{talign}\label{eq:triangle-ineq-I}
    \bigg| \frac{1}{\nout} \sum_{i = 1}^{\nout} \noise_i' \widetilde g(x_i') \bigg| \leq \bigg|\frac{1}{n} \sum_{i=1}^n \noise_i \widetilde g(x_i)\bigg| + \frac{u}{\staticnoutnorm{g}} \abss{\frac{1}{n}\sum_{i=1}^n \noise_i g(x_i) - \frac{1}{\nout} \sum_{i=1}^{\nout} \noise_i' g(x_i')}.
\end{talign}

Recall the chosen $g$ satisies $\staticnoutnorm{g} \geq u$. Observe that 
\begin{talign}
    \noise_i g(x_i) \seq{\cref{eq:process}} (y_i-\fstar(x_i)) g(x_i) = -\fstar(x_i) g(x_i) + y_i g(x_i),
\end{talign}
so we may apply \cref{lem:loss-diff} with $f_1 = \fstar, f_2 = g$ and $a=-1,b=1$. Thus, on the event $\ktgoodgen \cap \mc E_{\mrm{conc}}$, we have
\begin{talign}\label{eq:gsn-complexity-kt}
    \abss{\frac{1}{n}\sum_{i=1}^n \noise_i g(x_i) - \frac{1}{\nout} \sum_{i=1}^{\nout} \noise_i' g(x_i')} \leq \knorm{g} (\knorm{\fstar} + 1) \cdot  \ktcomplexity.
\end{talign}

Thus, we may rearrange \cref{eq:triangle-ineq-I} and combine with \cref{eq:zn_inter,eq:gsn-complexity-kt} to obtain
\begin{talign}
    \abss{\frac{1}{n} \sum_{i=1}^n \noise_i \widetilde g(x_i)} \geq 3 u^2- \frac{u}{\staticnoutnorm{g}} \knorm{g} (\knorm{\fstar} + 1) \cdot  \ktcomplexity
\end{talign}

Note that
\begin{talign}
    \frac{\knorm{g}}{\snoutnorm{g}} = \frac{\knorm{g}}{\stwonorm{g}} \cdot \frac{\stwonorm{g}}{\snnorm{g}} \cdot \frac{\snnorm{g}}{\snoutnorm{g}}.
\end{talign}
We tackle each term in turn. First, $\frac{\knorm{g}}{\stwonorm{g}} \sless{\cref{eq:knorm-twonorm}} \sqrt C_m$. Since we assume $\stwonorm{g} \geq \delta_n$, we have $\frac{\stwonorm{g}}{\snnorm{g}} \leq \frac{\delta_n/2 + \snnorm{g}}{\snnorm{g}} \sless{\cref{eq:unif_conc_result}} 2$ on event $\unifconc$; and $\frac{\snnorm{g}}{\snoutnorm{g}} \sless{\cref{eq:nnorm-noutnorm}} 2$ on event $\unifconc \cap \ktgoodgen$. Taken together, 
\begin{talign}\label{eq:knorm-nnoutnorm}
    \frac{\knorm{g}}{\snoutnorm{g}} \leq 4\sqrt C_m \qtext{on event} \unifconc \cap \ktgoodgen.
\end{talign}

As $u \geq \xi_n \geq 4 \sqrt C_m (\knorm{\fstar} + 1) \ktcomplexity$ by assumption,
we have therefore found $\widetilde g$ with norm $\staticnoutnorm{\widetilde g} = u$ satisfying
\begin{talign}
    \bigg|\frac{1}{n} \sum_{i=1}^n \noise_i \widetilde g(x_i)\bigg| \geq 3u^2 - u^2 = 2 u^2.
\end{talign}
We may further show that 
\begin{talign}
    \staticnnorm{\wtil g} = \frac{u}{\staticnoutnorm{g}} \staticnnorm{g}\leq u \frac{\staticnnorm{g}}{\staticnoutnorm{g}} \leq u \cdot 2 \qtext{on event} \unifconc \cap \ktgoodgen,
\end{talign}
where the last inequality follows from the fact that $\twonorm{g} \geq \delta_n$ and by applying \cref{eq:nnorm-noutnorm}.
So we observe
\begin{talign}
    2 u^2 \leq \bigg|\frac{1}{n} \sum_{i=1}^n \noise_i \widetilde g(x_i)\bigg|  \leq \sup_{\staticnnorm{\widetilde g}\leq 2u}\bigg|\frac{1}{n} \sum_{i=1}^n \noise_i \widetilde g(x_i)\bigg| = Z_n(2u)
\end{talign}

\subsection{\pcref{lem:case-II}}\label{proof:case-II}

Our aim is to show for any $g \in \partial \mc H$ 
with $\knorm{ g }\geq 1$,
\begin{talign}
    \bigg| \frac{1}{\nout} \sum_{i = 1}^{\nout} \noise_i' g(x_i') \bigg|\leq 2 \xi_n \snoutnorm{g} + 2\xi_n^2 \knorm{g} +\frac{1}{16} \snoutnorm{ g }^2\quad \text{with high probability.}
\end{talign}

Note that it is sufficient to prove our aim for $ g \in \partial \mc H$ with $\knorm{g} = 1$---by proving only for $g$ with $\knorm{g} = 1$, then for any $h\in \partial \mc H$ with $\knorm{h} \geq 1$, we may plug $g = h/\knorm{h}$ into 
\begin{talign}
    \bigg| \frac{1}{\nout} \sum_{i = 1}^{\nout} \noise_i' g(x_i') \bigg|\leq 2 \xi_n \snoutnorm{g} + 2\xi_n^2 + \frac{1}{16} \snoutnorm{ g }^2\quad 
    \label{eq:real_target}
\end{talign}
to recover the aim of interest. So without loss of generality, we show \cref{eq:real_target} for all $g$ such that $g \in \partial \mc H$ and $\knorm{g} = 1$.

Let $\badeventB$ denote the event where \cref{eq:real_target} is violated, i.e. there exists $g \in \partial \mc H$ with $\knorm{g} = 1$ so that 
\begin{talign}\label{eq:badeventB_ineq}
    \bigg| \frac{1}{\nout} \sum_{i = 1}^{\nout} \noise_i' g(x_i') \bigg| > 3 \xi_n \snoutnorm{g} + 2\xi_n^2 + \frac{1}{4} \snoutnorm{ g }^2.
\end{talign}
We prove the following set inclusion, 
\begin{talign}
    &\badeventB \cap \ktgoodgen \cap \mc E_{\mrm{conc}}\\
    \subseteq &\biggbraces{ \exists\ g \in \partial \mc H ~ \text{s.t.} ~ \knorm{g} = 1 ~ \text{and} ~ \bigg| \frac{1}{n} \sum_{i = 1}^{n} \noise_i g(x_i) \bigg|> 2 \vareps_n \nnorm{g} + 2\vareps_n^2 + \frac{1}{16} \nnorm{ g }^2 },
    \label{eq:event_inclusion}
\end{talign}
where we know the RHS event of \cref{eq:event_inclusion} has probability bounded by $c e^{-n\critthresh^2/(c' \widetilde \sigma^2)}$ which is proven in \cite[Lem.~13.23]{wainwright2019high}. 
So the set inclusion \cref{eq:event_inclusion} implies a bound over the event $\badeventB,$
\begin{talign}
    \Prob( \badeventB ) \leq \Prob(\ktgoodgen^c) +\Prob( \mc E_{\mrm{conc}}^c ) + c e^{-n\critthresh^2/(c' \widetilde \sigma^2)},
\end{talign}
where $\Prob(\ktgoodgen^c) \leq \delta$ by \cref{eq:ktgoodgen} and $\Prob(\unifconc^c)$ by \cref{eq:unifconc-prob-bound}.

Choose $g$ so that $\knorm{g} = 1$ and \cref{eq:badeventB_ineq} holds. Condition $\knorm{g} = 1$ as well as the condition \cref{eq:knorm-twonorm} resulting from a finite rank kernel $\kernel$ implies $\delta_n \leq 1 \leq \knorm{g} \leq \sqrt{C_m}\twonorm{g}$.
Invoke \cref{prop:rel-guarantee} for the choice of $g$ that satisfies $\twonorm{g} \geq \delta_n / \sqrt{C_m} \geq \delta_n$, so that on the event $\ktgoodgen \cap \unifconc$, we have the following norm equivalence,
\begin{talign}\label{eq:nnorm-noutnorm}
    \frac{1}{2} \nnorm{g} \leq \staticnoutnorm{g} \leq \frac{3}{2}\nnorm{g} \quad \text{for any $n$ such that $C_m \cdot \ktcomplexity \leq 1/18$.}
\end{talign}
Then we have the following chain of inequalities, which holds on event $\unifconc \cap \ktgoodgen$
\begin{talign}
    \bigg| \frac{1}{n} \sum_{i = 1}^{n} \noise_i g(x_i) \bigg| &\sgrt{(i)} \bigg| \frac{1}{\nout} \sum_{i = 1}^{\nout} \noise_i' g(x_i') \bigg| - \bigg| \frac{1}{n} \sum_{i = 1}^{n} \noise_i g(x_i) - \frac{1}{\nout} \sum_{i = 1}^{\nout} \noise_i' g(x_i')  \bigg|\\
    &\sgrt{(ii)} 3 \xi_n \snoutnorm{g} + 2\xi_n^2 + \frac{1}{4} \snoutnorm{ g }^2 - \knorm{g} (\knorm{\fstar} + 1) \cdot  \ktcomplexity\\
    &\sgrt{\cref{eq:knorm-nnoutnorm}} 3 \xi_n \snoutnorm{g} + 2\xi_n^2 + \frac{1}{4} \snoutnorm{ g }^2 - 4\sqrt{C_m} \snoutnorm{g} (\knorm{\fstar} + 1) \cdot  \ktcomplexity\\
    &\sgrt{\cref{eq:nnorm-noutnorm}} (\frac{3}{2}\critthresh - 2\sqrt{C_m} (\knorm{\fstar} + 1)\cdot \ktcomplexity) \nnorm{g} +2\xi_n^2 + \frac{1}{16} \nnorm{g}^2,\label{eq:gsn-lower-II}
\end{talign}
where step (i) follows from triangle inequality and step (ii) follows from our assumption \cref{eq:badeventB_ineq} to bound the first term and our approximation guarantee \cref{eq:gsn-complexity-kt} to bound the second term. 
By definition of $\critthresh$, we have
\begin{talign}
    \frac{3}{2}\critthresh - 2\sqrt{C_m} (\knorm{\fstar} + 1)\cdot \ktcomplexity \geq 2 \critthresh.
\end{talign}
Using this to refine \cref{eq:gsn-lower-II}, we have
\begin{talign}
    \bigg| \frac{1}{n} \sum_{i = 1}^{n} \noise_i g(x_i) \bigg| &\geq 2\critthresh \nnorm{g} +2\critthresh^2 + \frac{1}{16} \nnorm{g}^2,
\end{talign}
which directly implies the inclusion \cref{eq:event_inclusion} as desired.

\section{\pcref{cor:krr-kt}}
\label{app:proof-krr-explicit-rates}

We state a more detailed version of the theorem:

\begin{theorem}[\textsc{KT-KRR} guarantee for infinite-dimensional RKHS, detailed] \label{cor:krr-kt-detailed}
Assume $\fstar \in \rkhs(\kernel)$ and \cref{assum:compact-support} is satisfied. If $\kernel$ is $\textsc{LogGrowth}(\alpha,\beta)$,
then for some constant $c$ (depending on $d,\alpha,\beta$), $\krrkt$ with $\lambda'=\bigO{1/\nout}$ satisfies
\begin{talign}\label{eq:krr-kt-log}
    \statictwonorm{\krrkt - \fstar}^2 \leq c 
    \parenth{\frac{\log^{\alpha}n}{n} + \frac{\sqrt{\log^{\alpha}\nout}}{\nout}}
    \cdot \brackets{\knorm{\fstar} + 1}^2
\end{talign}
with probability at least $1-2\delta - 2 e^{-\frac{n\delta_n^2}{c_1 (\knorm{\fstar}^2+\sigma^2) }}$.

If $\kernel$ is $\textsc{PolyGrowth}(\alpha,\beta)$ with $\alpha \in (0, 2)$, then for some constant $c$ (depending on $d,\alpha,\beta$), $\krrkt$ with $\lambda = \bigO{\nout^{-\frac{2-\alpha}{2}}}$ satisfies
\begin{talign}\label{eq:krr-kt-poly}
    \statictwonorm{\krrkt - \fstar}^2 &\leq c \knorm{\fstar}^{\frac{2}{2+\alpha}} n^{-\frac{2}{2 + \alpha}} + \brackets{\knorm{\fstar} + 1}^2 \nout^{-\frac{2-\alpha}{2}} \log \nout + c' b^{\frac{4}{2+\alpha}} n^{-\frac{2}{2 + \alpha}}
\end{talign}
with probability at least $1- 2\delta -2 e^{-\frac{n\delta_n^2}{c_1 (\knorm{\fstar}^2+\sigma^2) }}$.
\end{theorem}

\subsection{Generic KT-KRR guarantee}

We state a generic result for infinite-dimensional RKHS that only depends on the Rademacher and Gaussian critical radii as well as the KT approximation term, all introduced in \cref{proof:kt-krr-finite}.

\begin{theorem}[\textsc{KT-KRR}]\label{thm:krr-kt}    
Let $\fstar\in \rkhs(\kernel)$ 
and \cref{assum:compact-support} is satisfied. Let $\delta_n$, $\vareps_n$ denote the solutions to \cref{eq:critical-gaussian-cond}, \cref{eq:critical-rademacher-cond}, respectively. Denote $\krrkt$ with regularization parameter $\lambda' \geq 2 \ktcomplexity$, where $\ktcomplexity$ is defined by \cref{eq:kt-complexity}. 
Then with probability at least $1 - 2\delta - 2 e^{-\frac{n\delta_n^2}{c (\knorm{\fstar}^2+\sigma^2) }} - c_1 e^{-c_2 \frac{n\knorm{\fstar}^2 \vareps_n^2}{\sigma^2}}$, we have
\begin{talign}
    \statictwonorm{\krrkt - \fstar}^2 &\leq  \mbb U^{\mrm{full}} + \mbb U^{\mrm{KT}}, \qtext{where} \label{eq:krr-kt-L2-bound} \\
    \mbb U^{\mrm{full}} &\defeq c \parenth{\vareps_n^2 + \lambda'} \brackets{\knorm{\fstar} + 1}^2+ c \delta_n^2 \qtext{and} \label{eq:krr-term}\\
    \mbb U^{\mrm{KT}} &\defeq c \cdot \ktcomplexity ~ \brackets{\knorm{\fstar} + 1}^2.\label{eq:kt-term}
\end{talign}
\end{theorem}
\proofref{proof:krr-kt}
The term $\mbb U^{\mrm{full}}$ follows from the excess risk bound of \textsc{Full-KRR} $\krr$. The term $\mbb U^{\mrm{KT}}$ follows from our KT approximation.
Clearly, the best rates are achieved when we choose $\lambda = 2 \ktcomplexity$. 

\subsection{Proof of explicit rates}
The strategy for each setting is as follows:
\begin{enumerate}
    \item Bound the Gaussian critical radius \cref{eq:critical-rademacher-cond} using \citep[Cor.~13.18]{wainwright2019high}, which reduces to finding $\vareps>0$ satisfying the inequality
    \begin{talign}\label{eq:gsn-ineq}
        \sqrt{\frac{2}{n}} \sqrt{\sum_{j=1}^n \min\braces{\vareps^2, \hat \mu_j}} \leq \beta\vareps^2, \qtext{where}\beta \defeq \frac{\knorm{\fstar}}{4\sigma}
    \end{talign}
    and $\hat \mu_1 \geq \hat \mu_2 \geq \ldots \geq \hat \mu_n\geq 0$ are the eigenvalues of the normalized kernel matrix $\mkernel/n$, where $\mkernel$ is defined by \cref{eq:mkernel}.

    \item Bound the Rademacher critical radius \cref{eq:critical-gaussian-cond} using \citep[Cor.~14.5]{wainwright2019high}, which reduces to solving the inequality
    \begin{talign}\label{eq:rad-ineq}
        \sqrt{\frac{2}{n}} \sqrt{\sum_{j=1}^\infty \min\braces{\delta^2,\mu_j}} \leq \frac{\delta^2}{b},
    \end{talign}
    where $\parenth{\mu_j}_{j=1}^\infty$ are the eigenvalues of the $\kernel$ according to Mercer's theorem \cite[Thm.~12.20]{wainwright2019high} and $b$ is the uniform bound on the function class.
    \item Bound $\ktcomplexity$ \cref{eq:kt-complexity} using the covering number bound $\coveringnumber(\ball_2^d(\Rin),\eps)$ from \cref{assum:alpha-beta-kernel}.
\end{enumerate}
In the sequel, we make use of the following notation. Let
\begin{talign} 
    R_n \defeq 1 + \sup_{x \in \inputcoreset} \statictwonorm{x_i} \seq{\cref{eq:radius}} 1 + \mfk R_{\mrm{in}} \qtext{and} L_\kernel(r) \defeq \frac{\mfk C_d}{\log 2} r^\beta
\end{talign}
according to \citep[Eq.~6]{li2024debiased}, where $\mfk C_d$ is the constant that appears in \cref{assum:alpha-beta-kernel}.

\subsubsection{Proof of \cref{eq:krr-kt-poly}} \label{proof:krr-kt-poly}

We begin by solving \cref{eq:gsn-ineq}. 

\begin{lemma}[Critical Gaussian radius for \textsc{PolyGrowth} kernels]
\label{lem:polygrowth-gsn}
Suppose \cref{assum:compact-support} is satisfied and $\kernel$ is \textsc{PolyGrowth} with $\alpha<2$ as defined by \cref{assum:alpha-beta-kernel}. Then the Gaussian critical radius satisfies 
\begin{talign}\label{eq:polygrowth-vareps-n}
    \vareps_n^2 \simeq \parenth{\frac{2c}{\knorm{\fstar}/4\sigma}}^{\frac{4}{2+\alpha}} \parenth{2^{-\alpha} L_{\kernel}(R_n)(1 + \frac{32 \alpha }{2-\alpha})}^{\frac{2}{2+\alpha}} \cdot n^{-\frac{2}{2+\alpha}}.
\end{talign}
\end{lemma}
\begin{proof}
\citep[Cor.~B.1]{li2024debiased} implies that
\begin{talign}
    \hat \mu_j \leq 4\parenth{\frac{L_{\kernel}(R_n)}{j-1}}^{\frac{2}{\alpha}} \qtext{for all} j > L_{\kernel}(R_n)+1
\end{talign}
Let $k$ be the smallest integer such that
\begin{talign}
    k > L_{\kernel}(R_n)+1 \qtext{and} 4\parenth{\frac{L_{\kernel}(R_n)}{k-1}}^{\frac{2}{\alpha}} \leq \vareps^2.
\end{talign}
By \cref{assum:compact-support}, $R_n$ is a constant, so the first inequality is easily satisfied for large enough $n$ 
\begin{talign}\label{eq:k-cond}
    k \geq 2^{-\alpha} L_{\kernel}(R_n) \vareps^{-\alpha} + 1.
\end{talign}
Then 
\begin{talign} 
    \frac{2}{\sqrt{n}} \sqrt{\sum_{j = 1}^n \min\braces{ \vareps^2, \hat{\mu}_j }} &\leq \frac{2}{\sqrt{n}} \sqrt{ k\vareps^2 + \sum_{j = k + 1}^n 4\parenth{\frac{L_{\kernel}(R_n)}{j-1}}^{\frac{2}{\alpha}} } \\
    &\sless{(i)} \frac{2}{\sqrt{n}} \sqrt{k \vareps^2 + \frac{4 L_{\kernel}(R_n)^{2/\alpha}}{2/\alpha - 1} k^{1-2/\alpha}} \\
    &\sless{\cref{eq:k-cond}} \frac{2}{\sqrt{n}} \sqrt{ 2^{-\alpha} L_{\kernel}(R_n) \vareps^{2 - \alpha} + \frac{4 \cdot 2^{2-\alpha} L_{\kernel}(R_n)}{2/\alpha - 1} \vareps^{2-\alpha}},
\end{talign}
where step (i) follows from the approximation 
\begin{talign}
    \sum_{j = k}^{n-1} 4\parenth{\frac{L_{\kernel}(R_n)}{j}}^{\frac{2}{\alpha}} \leq 4 L_{\kernel}(R_n)^{2/\alpha}\int_{k}^\infty t^{-2/\alpha} dt = 4 L_{\kernel}(R_n)^{2/\alpha} \frac{1}{2/\alpha-1} k^{1-\frac{2}{\alpha}}.
\end{talign}
To solve \cref{eq:gsn-ineq}, it suffices to solve
\begin{talign}
    &\frac{2c}{\sqrt n} \sqrt{2^{-\alpha} L_{\kernel}(R_n) (1 + \frac{16 }{2/\alpha-1}) \vareps^{2-\alpha}} \leq \beta \vareps^2\\
    \implies & \vareps \geq \parenth{\frac{2c}{\beta}}^{\frac{2}{2+\alpha}} \parenth{2^{-\alpha} L_{\kernel}(R_n)(1 + \frac{32 \alpha }{2-\alpha})}^{\frac{1}{2+\alpha}} \cdot n^{-\frac{1}{2+\alpha}}.
\end{talign}
Since $\vareps_n$ is the smallest such solution to \cref{eq:gsn-ineq} by definition, we have \cref{eq:polygrowth-vareps-n} as desired.
\end{proof}

We proceed to solve \cref{eq:rad-ineq}.
\begin{lemma}\label{lem:polygrowth-rad}
Suppose \cref{assum:compact-support} is satisfied and $\kernel$ is \textsc{PolyGrowth} with $\alpha<2$ as defined by \cref{assum:alpha-beta-kernel}. Then the Rademacher critical radius satisfies 
\begin{talign}
    \vareps_n^2 \simeq b^{\frac{4}{2+\alpha}} \parenth{2^{-\alpha} L_{\kernel}(R_n)(1 + \frac{32 \alpha }{2-\alpha})}^{\frac{2}{2+\alpha}} n^{-\frac{2}{2 + \alpha}}.
\end{talign}
\end{lemma}
\begin{proof}

Thus, we can solve the following inequality
\begin{talign}
    \sqrt{\frac{2}{n}} \sqrt{\sum_{j=1}^n \min\braces{\delta^2, \hat \mu_j}} \leq \frac{1}{b}\delta^2,
\end{talign}
Following the same logic as in the proof of \cref{lem:polygrowth-gsn} but with $\beta = 1/b$ yields the desired bound.
\end{proof}

Finally, it remains to bound \cref{eq:kt-complexity}. We have
\begin{talign}
    \ktcomplexity &= \frac{\ktcomplexityconstant^2}{\nout} (2 + \errmmd_{\kernel}(n,\nout,\delta, \Rin,\frac{\ktcomplexityconstant}{\nout})) \\
    &\leq \frac{\ktcomplexityconstant^2}{\nout} (2 + \sqrt{\log\parenth{\frac{\nout \log(n/\nout)}{\delta}}  \cdot \brackets{\log\parenth{\frac{1}{\delta}} + \log \coveringnumber_{\kernel}(\ball_2^d(\Rin), \frac{\ktcomplexityconstant}{\nout}) }}) \\
    &\leq \frac{\ktcomplexityconstant^2}{\nout} (2 + \sqrt{\log\parenth{\frac{\nout \log(n/\nout)}{\delta}}  \cdot \brackets{\log\parenth{\frac{1}{\delta}} + \mfk C_d \parenth{\frac{\nout}{\ktcomplexityconstant}}^\alpha (\Rin+1)^\beta)}} \\
    &\leq \frac{\ktcomplexityconstant^2}{\nout} \parenth{2 + \sqrt{\log\parenth{\frac{\nout \log(n/\nout)}{\delta}}} \cdot \brackets{ \sqrt{ \log\parenth{\frac{1}{\delta}}} + \sqrt{\mfk C_d \frac{(\Rin + 1)^\beta}{\ktcomplexityconstant^\alpha}} \nout^{\frac{\alpha}{2}} }} \\
    &\leq \nout^{\frac{\alpha}{2} - 1} \cdot \ktcomplexityconstant^2 \parenth{2 + \sqrt{\log\parenth{\frac{\nout \log(n/\nout)}{\delta}}} \cdot \sqrt{\mfk C_d \frac{(\Rin + 1)^\beta}{\ktcomplexityconstant^\alpha}} }
\end{talign}
for some universal positive constant $c$.

In summary, there exists positive constants $c_0,c_1,c_2$ such that
\begin{talign}
    \vareps_n^2 \leq c_0 \parenth{\frac{\sigma}{\knorm{\fstar}}}^{\frac{4}{2+\alpha}} n^{-\frac{2}{2 + \alpha}} 
    \quad 
    \delta_n^2 \leq c_1 b^{\frac{4}{2+\alpha}} n^{-\frac{2}{2 + \alpha}} 
    \quad 
    \ktcomplexity \leq c_2 \ktcomplexityconstant^2 \nout^{- \frac{2-\alpha}{2}} \log \nout
\end{talign}

Setting $\lambda' = c_2 \ktcomplexityconstant^2 \nout^{-\frac{2-\alpha}{2}} \log \nout$, we have
\begin{talign}
    \statictwonorm{\krrkt - \fstar}^2 &\leq c \parenth{\vareps_n^2 + \lambda' + \ktcomplexity } \cdot \brackets{\knorm{\fstar} + 1}^2 + c' \delta_n^2 \\
    &\leq c \knorm{\fstar}^{\frac{2}{2+\alpha}} n^{-\frac{2}{2 + \alpha}} + \brackets{\knorm{\fstar} + 1}^2 \nout^{-\frac{2-\alpha}{2}} \log \nout + c' b^{\frac{4}{2+\alpha}} n^{-\frac{2}{2 + \alpha}}.
\end{talign}

\subsubsection{Proof of \cref{eq:krr-kt-log}} \label{proof:krr-kt-log}

We begin by solving \cref{eq:gsn-ineq}. 

\begin{lemma}[Critical Gaussian radius for \textsc{LogGrowth} kernels]\label{lem:loggrowth-gsn}

Under \cref{assum:compact-support} and \textsc{LogGrowth} version of \cref{assum:alpha-beta-kernel}, Gaussian critical radius satisfies
\begin{talign}
    \vareps_n^2 \simeq \frac{\sigma^2}{\knorm{\fstar}^2} \frac{\log(2e\cdot \frac{\knorm{\fstar}}{4\sigma} \sqrt n)^\alpha}{n} \cdot L_\kernel(R_n) C_\alpha''
\end{talign}
for some constant $C_\alpha''$ that only depends on $\alpha$.
where we ignore log-log factors.
\end{lemma}
\begin{proof}
\citep[Cor.~B.1]{li2024debiased} implies that
\begin{talign}
    \hat \mu_j \leq 4\exp\parenth{2 - 2\parenth{\frac{j-1}{L_{\kernel}(R_n)}}^{\frac{1}{\alpha}}} \qtext{for all} j > L_{\kernel}(R_n)+1
\end{talign}
Let $k$ be the smallest integer such that
\begin{talign}
    k > L_{\kernel}(R_n)+1 \qtext{and} 4\exp\parenth{2 - 2\parenth{\frac{j-1}{L_{\kernel}(R_n)}}^{\frac{1}{\alpha}}} \leq \vareps^2.
\end{talign}
By \cref{assum:compact-support}, $R_n$ is a constant, so the first inequality is easily satisfied for large enough $n$ 
\begin{talign}\label{eq:k-cond-log}
    k \geq L_{\kernel}(R_n) \log\parenth{\frac{2e}{\vareps}}^\alpha + 1.
\end{talign}
Thus, $k = \ceil{L_{\kernel}(R_n) \log\parenth{\frac{2e}{\vareps}}^\alpha + 1}$. Then 
\begin{talign} \label{eq:gsn-ineq-1}
    \frac{2}{\sqrt{n}} \sqrt{\sum_{j = 1}^n \min\braces{ \vareps^2, \hat{\mu}_j }} &\leq \frac{2}{\sqrt{n}} \sqrt{ k\vareps^2 + \sum_{j = k + 1}^n 4\exp\parenth{2 - 2\parenth{\frac{j-1}{L_{\kernel}(R_n)}}^{\frac{1}{\alpha}}}}
\end{talign}
Consider the following approximation:
\begin{talign}
    \sum_{\ell = k}^{n-1} 4\exp\parenth{2 - 2\parenth{\frac{\ell}{L_{\kernel}(R_n)}}^{\frac{1}{\alpha}}} \leq 4 e^2 \int_{k}^\infty e^{-\frac{2t}{L_{\kernel}(R_n)}^{1/\alpha}} dt = \int_{k-1}^\infty c^{t^{1/\alpha}} dt,
\end{talign}
where $c\defeq \exp(-(L_{\kernel}(R_n)/2)^{-1/\alpha}) \in (0,1)$. Defining $m\defeq -\log c> 0$ and $k' \defeq k-1$, we have
\begin{talign}\label{eq:gamma-upper-bound}
    \int_{k'}^\infty c^{t^{1/\alpha}} dt \leq C_\alpha (k' b^{-1} + b^{\alpha-1}m^{-\alpha}) e^{-mk'^{1/\alpha}},
\end{talign}
by \citet[Eq.~50]{li2024debiased}, where $C_\alpha>0$ is a constant satisfying $(x+y)^\alpha \leq C_\alpha (x^\alpha + y^\alpha)$ for any $x,y>0$ and $b$ is a known constant depending only on $\alpha$. Plugging in $k' = \ceil{L_{\kernel}(R_n) \log\parenth{\frac{2e}{\vareps}}^\alpha}$, we can bound the exponential by
\begin{talign}
    e^{-m k'^{1/\alpha}} \leq e^{-m L_k(R_n)^{1/\alpha} \log\parenth{\frac{2e}{\vareps}}} = \parenth{\frac{2e}{\vareps}}^{-m L_\kernel(R_n)^{1/\alpha}}.
\end{talign}
Note that we can simplify the exponent by $-m L_\kernel(R_n)^{1/\alpha} = -(L_\kernel(R_n) / 2)^{-1/\alpha} L_\kernel(R_n)^{1/\alpha} = -2^{1/\alpha}$.
Note that $k' = k-1 \geq L(R_n)= 2m^{-\alpha}$. Thus, we can absorb the $b^{\alpha-1} m^{-\alpha}$ term in \cref{eq:gamma-upper-bound} into $k$ and obtain the following bound
\begin{talign}
    \sum_{\ell = k}^{n-1} 4\exp\parenth{2 - 2\parenth{\frac{\ell}{L_{\kernel}(R_n)}}^{\frac{1}{\alpha}}} \leq C_\alpha' k' \parenth{\frac{2e}{\vareps}}^{-2^{1/\alpha}},
\end{talign}
where $C_\alpha'$ depends only on $\alpha$.
Plugging this bound into \cref{eq:gsn-ineq-1}, we have
\begin{talign}
    \frac{2}{\sqrt{n}} \sqrt{\sum_{j = 1}^n \min\braces{ \vareps^2, \hat{\mu}_j }} &\leq \frac{2}{\sqrt n} \sqrt{k\vareps^2 + C_\alpha' k \parenth{\frac{\vareps}{2e}}^{2^{1/\alpha}}} \\
    &\sless{\cref{eq:k-cond-log}} \frac{2c}{\sqrt n}\sqrt{ L_{\kernel}(R_n) \log\parenth{\frac{2e}{\vareps}}^\alpha (\vareps^2 + C_\alpha' \parenth{\frac{\vareps}{2e}}^{2^{1/\alpha}}) } \\
    &\leq \frac{2c}{\sqrt n}\sqrt{ L_{\kernel}(R_n) \log\parenth{\frac{2e}{\vareps}}^\alpha C_\alpha'' \vareps^{2^{ 1/(1 \vee \alpha)}} }
\end{talign}
for some constant $C_\alpha''$ that only depends on $\alpha$ and universal positive constant $c$.
To solve \cref{eq:gsn-ineq}, it suffices to solve
\begin{talign}
    & \frac{2c}{\sqrt n}\sqrt{ L_{\kernel}(R_n) \log\parenth{\frac{2e}{\vareps}}^\alpha C_\alpha'' \vareps^{2^{ 1/(1 \vee \alpha)}} } \leq \beta \vareps^2,
\end{talign}
which is implied by the looser bound
\begin{talign}
    \frac{1}{\beta^2} \cdot \frac{4c^2}{n}\cdot L_\kernel(R_n) C_\alpha'' \leq \vareps^2 \log\parenth{\frac{2e}{\vareps}}^{-\alpha}.
\end{talign}
The solution to \cref{{eq:gsn-ineq}} (up to log-log factors) is
\begin{talign}
    \vareps \simeq \frac{\log(2e\cdot \beta \sqrt n)^{\alpha/2}}{\sqrt n} \sqrt{\frac{4c^2}{\beta^2} \cdot L_\kernel(R_n) C_\alpha''}.
\end{talign}

\end{proof}

We proceed to solve \cref{eq:rad-ineq}.

\begin{lemma}[Critical Gaussian radius for \textsc{LogGrowth} kernels]\label{lem: loggrowth-rad}

Under \cref{assum:compact-support} and \textsc{LogGrowth} version of \cref{assum:alpha-beta-kernel}, the Rademacher critical radius satisfies 
\begin{talign}
    \delta_n^2 \simeq b^2 \frac{\log(\frac{2e}{b}\cdot \sqrt n)^{\alpha}}{n} \cdot L_\kernel(R_n) C_\alpha''.
\end{talign}
\end{lemma}
\begin{proof}
Following the same logic as in the proof of \cref{lem:loggrowth-gsn} but with $\beta = 1/b$ yields the desired bound.
\end{proof}

Finally, it remains to bound \cref{eq:kt-complexity}. We have
\begin{talign}
    \ktcomplexity &= \frac{\ktcomplexityconstant^2}{\nout} (2 + \errmmd_{\kernel}(n,\nout,\delta, \Rin,\frac{\ktcomplexityconstant}{\nout})) \\
    &\leq \frac{\ktcomplexityconstant^2}{\nout} (2 + \sqrt{\log\parenth{\frac{\nout \log(n/\nout)}{\delta}}  \cdot \brackets{\log\parenth{\frac{1}{\delta}} + \log \coveringnumber_{\kernel}(\ball_2^d(\Rin), \frac{\ktcomplexityconstant}{\nout}) }}) \\
    &\leq \frac{\ktcomplexityconstant^2}{\nout} (2 + \sqrt{\log\parenth{\frac{\nout \log(n/\nout)}{\delta}}  \cdot \brackets{\log\parenth{\frac{1}{\delta}} + 
    \mfk C_d \log(\frac{e \nout}{\ktcomplexityconstant})^\alpha (\Rin+1)^\beta}}) 
\end{talign}
for some universal positive constant $c$.

In summary, there exists universal positive constants $c_0,c_1, c_2$ such that 
\begin{talign}
    \vareps_n^2 \leq c_0 \frac{\sigma^2}{\knorm{\fstar}^2} \frac{\log(2e\cdot \frac{\knorm{\fstar}}{4\sigma} \sqrt n)^\alpha}{n}
    \quad \delta_n^2 \leq c_1 b^2 \frac{\log(\frac{2e}{b}\cdot \sqrt n)^{\alpha/2}}{n}
    \quad \ktcomplexity \leq c_2 \frac{\ktcomplexityconstant}{\nout} \log(\frac{e\nout}{\ktcomplexityconstant})^{\alpha/2}\Rin^{\beta/2}.
\end{talign}

Setting $\lambda' = 2c_2 \frac{\ktcomplexityconstant}{\nout} \log(\frac{e\nout}{\ktcomplexityconstant})^{\alpha/2}$, we have
\begin{talign}
    \statictwonorm{\krrkt - \fstar}^2 &\leq c \parenth{\vareps_n^2 + \lambda' + \ktcomplexity } \cdot \brackets{\knorm{\fstar} + 1}^2 + c' \delta_n^2 \\
    &\leq c \frac{\log(2e\cdot \frac{\knorm{\fstar}}{4\sigma} \sqrt n)^\alpha}{n} + \brackets{\knorm{\fstar} + 1}^2 \frac{c}{\nout} \log(\frac{e\nout}{\ktcomplexityconstant})^{\alpha/2} + c' b^2 \frac{\log(\frac{2e}{b}\cdot \sqrt n)^{\alpha/2}}{n}.
\end{talign}

\section{\pcref{thm:krr-kt}}\label{proof:krr-kt}

Our first goal is to bound the in-sample prediction error.
We relate $\staticnnorm{\krrkt-\fstar}^2$ to $\staticnnorm{\krr - \fstar}^2$, where the latter quantity has well known properties from standard analyses of the KRR estimator~(refer to~\cite{wainwright2019high}).
Note that regularization parameter $\lambda'$ of KT based estimator $\krrkt$ is independently chosen from the regularization parameter $\lambda$ of the estimator based on original samples $\krr$. For $\krr$, we choose the regularization parameter
\begin{talign}\label{eq:lambda}
    \lambda = 2\vareps_n^2,
\end{talign}
which is known to yield optimal $L^2$ error rates. 

Define the main event of interest,
\begin{align}
    \event \defeq \sbraces{ \stwonorm{\krrkt - \fstar}^2 \leq c(\vareps_n^2 + \delta_n^2 + \lambda' + \ktcomplexity)[\knorm{\fstar} + 1]^2}.
\end{align}
Our goal is to show $\event$ occurs with high probability. For that end, we introduce several additional events that are used throughout this proof.

For some constant $c > $, define the event of an appealing in-sample prediction error of $\krrkt$, 
\begin{talign}
    \ktnormbound(t) \defeq \braces{ \staticnnorm{\krrkt-\fstar}^2 \leq c \brackets{ t^2 + \lambda' + \ktcomplexity } \cdot (\knorm{\fstar} + 1)^2 }\qtext{for} t\geq \vareps_n.
\end{talign}
where $\ktcomplexity$ is defined in \cref{eq:kt-complexity}. Recall $\ktgoodgen$ is the event where \ktcompresspp succeeds as defined by \cref{eq:ktgoodgen}. 

Further as $\fstar$ and $\krrkt$ are both in $\sbraces{f \in \rkhs : \knorm{f} \leq R}$, we may deduce that all the functions under consideration satisfies $\infnorm{f} \leq \sinfnorm{\kernel} \knorm{f} \leq \sinfnorm{\kernel} R$ where $\sinfnorm{\kernel} < \infty$. Accordingly, we define a uniform concentration event, 
\newcommand{\unifconcp}{\mc E_{\trm{conc}}'}
\begin{talign}\label{eq:unif_conc_event_prime}
    \unifconcp \defeq \sbraces{\sup_{f \in \mc F}\big| \statictwonorm{f}^2 - \staticnnorm{f}^2 \big| \leq \statictwonorm{f}^2/2 + \delta_n^2/2} \qtext{where} \mc F = \sbraces{f\in \rkhs: \sinfnorm{f} \leq 2\sinfnorm{\kernel}R}.
\end{talign}
Event \cref{eq:unif_conc_event_prime} is analogous to the event $\unifconc$ previously defined in \cref{eq:unif_conc_event} when dealing with finite rank kernels.

We first show that 
\begin{talign}\label{eq:inf_rank_set_inc}
    \ktnormbound(\vareps_n\vee \delta_n) \cap \unifconcp \subseteq \event. 
\end{talign}
Notice that almost surely we have
\begin{talign}
    \sinfnorm{\krrkt - \fstar} \leq 2\sinfnorm{\kernel}R,
\end{talign}
thereby implying
\begin{talign}\label{eq:cons_unifconcp}
    \statictwonorm{\krrkt - \fstar}^2 \leq 2~ \staticnnorm{\krrkt-\fstar}^2 + \delta_n^2 \qtext{on the event $\unifconcp$.}
\end{talign}
Next invoking the event $\ktnormbound(\vareps_n\vee\delta_n)$ along with \cref{eq:cons_unifconcp}, we have
\begin{talign}
    \statictwonorm{\krrkt - \fstar}^2 &\leq 2c[(\vareps_n\vee\delta_n)^2 + \lambda' + \ktcomplexity] \cdot ( \knorm{\fstar} + 1 )^2 + \delta_n^2\\
    &\leq c(\vareps_n^2 + \delta_n^2 + \lambda' + \ktcomplexity)[\knorm{\fstar} + 1]^2.
\end{talign}
which recovers the event of $\event$. 

The remaining task is to show $\event$ is of high-probability, which amounts to showing events $\ktnormbound(t)$ and $\unifconcp$ are of high-probability by reflecting on \cref{eq:inf_rank_set_inc}. From \cite[Thm.~14.1]{wainwright2019high}, we may immediately derive 
\begin{talign}
    \Prob( \unifconcp ) \geq 1 - c_1 e^{-c_2 \frac{n\delta_n^2}{\sinfnorm{\kernel}^2R^2}}
\end{talign}
for some constants $c_1, c_2 > 0$.
We further claim that
\begin{talign}\label{eq:Ln-claim}
    \Parg{\ktnormbound(t) \mid \inputcoreset } 
    \geq 1-\delta - e^{-\frac{n t^2}{c_0 \sigma^2}} - c_1 e^{-c_2 \frac{n\knorm{\fstar}^2 t^2}{\sigma^2}}
\end{talign}
for some constants $c_0,c_1,c_2 > 0$. Proof of claim \cref{eq:Ln-claim} is deferred to \cref{proof:Ln-claim}. Plugging in $t = \vareps_n\vee\delta_n$ into \cref{eq:Ln-claim}, and invoking inequality $\vareps_n\vee \delta_n \geq \delta_n$ so as to decouple the dependence on $\inputcoreset$, we have
\begin{talign}
    \Parg{ \ktnormbound(\vareps_n\vee\delta_n) \mid \inputcoreset } \geq 1 - \delta - e^{-\frac{n\delta_n^2}{c_0\sigma^2}} - c_1 e^{-c_2 \frac{\knorm{\fstar}^2 n\delta_n^2}{\sigma^2}}
\end{talign}
which further implies
\begin{talign}
    \Parg{ \ktnormbound(\vareps_n\vee\delta_n)} \geq 1 - \delta - e^{-\frac{n\delta_n^2}{c_0\sigma^2}} - c_1 e^{-c_2 \frac{\knorm{\fstar}^2 n\delta_n^2}{\sigma^2}}.
\end{talign}
Putting the pieces together, for some constants $c_0, c_1 > 0$, we have
\begin{talign}\label{eq:inf_rank_prob_lower}
    \Parg{ \event } \geq 1 - \delta - c_0 e^{-c_1 \frac{ n\delta_n^2 }{ \sigma^2 \wedge (\sigma^2/\knorm{\fstar}^2) \wedge (\sinfnorm{\kernel}^2 R^2) } }.
\end{talign}
Overall, \cref{eq:inf_rank_set_inc,eq:inf_rank_prob_lower} collectively yields the desired result.

\vspone[-4]
\subsection{Proof of claim~\cref{eq:Ln-claim}}\label{proof:Ln-claim}

To prove claim \cref{eq:Ln-claim}, we introduce two new intermediary and technical events. For some positive constant $c_0$, define the event \footnote{Since the input points in $\inputcoreset$ are fixed, the randomness in $\krr$ originates entirely from the randomness of the noise variables $\boldsymbol \xi$.} when in-sample prediction error of $\krr$ is appealing
\begin{talign}\label{eq:fullnormbound}
    \fullnormbound(t) \defeq \braces{ \staticnnorm{\krr-\fstar}^2 \leq 3 c_0 \knorm{\fstar}^2 t^2} 
    \qtext{for} t\geq \vareps_n.
\end{talign}
The second intermediary event, denoted as $\mc E_{\what{\Delta}_{\trm{KT}}}(t)$, is the intersection of \cref{eq:Z_kt-inner} and \cref{eq:Z_kt-inner_second}, which we do not elaborate here due to its technical nature---event $\mc E_{\what{\Delta}_{\trm{KT}}}(t)$ plays an analogous role to $\badeventA^c\cap \badeventB^c$ defined in \cref{eq:control,eq:badeventB} respectively. 

Our goal here is two-folds: first is to show 
\begin{align}
    \sbraces{\fullnormbound(t) \cap \ktgoodgen \cap \mc E_{\what{\Delta}_{\trm{KT}}}(t) } \implies \ktnormbound(t)
\end{align}
and second is to prove the following bound
\begin{talign}
    \Parg{\fullnormbound(t) \cap \ktgoodgen \cap \mc E_{\what{\Delta}_{\trm{KT}}}(t) \mid \inputcoreset } 
    \geq 1-\delta - e^{-\frac{n t^2}{c_0 \sigma^2}} - c_1 e^{-c_2 \frac{n\knorm{\fstar}^2 t^2}{\sigma^2}},
\end{talign}
from which \cref{eq:Ln-claim} follows. Note that \citet[Thm.~13.17]{wainwright2019high} show 
\begin{talign}
    \Parg{\fullnormbound(t)} \geq 1-c_1 e^{-c_2 \frac{n\knorm{\fstar}^2 t^2}{\sigma^2}}
\end{talign}
for some constants $c_1,c_2 > 0$ and that $\Parg{\ktgoodgen\mid\inputcoreset}\geq 1 - \delta$. So it remains to bound the probability of event $\mc E_{\what{\Delta}_{\trm{KT}}}(t)$, which we show below.

Given $f$, define the following quantities
\begin{talign}\label{eq:Ln-Lnout-def}
    L_n(f) &\defeq \frac{1}{n}\sumn(f^2(x_i) -2 f(x_i) y_i) + \frac{1}{n}\sumn y_i^2 \qtext{and} \\
    L_{\nout}(f) &\defeq \frac{1}{\nout}\sumnout (f^2(x_i') -2 f(x_i') y_i') + \frac{1}{n}\sumn y_i^2.
\end{talign}

In the sequel, we repeatedly make use of the following fact: on event $\ktgoodgen$ defined in \cref{eq:ktgoodgen}, we have
\begin{talign}\label{eq:krrkt-diff}
    \abss{L_n(f) - L_{\nout}(f)} \leq \parenth{\knorm{f}^2 + 2} \cdot \ktcomplexity\qtext{for all non-zero} f\in \rkhs.
\end{talign}
The claim of \cref{eq:krrkt-diff} is deferred to the end of this section. Given $f$, we can show with some algebra that
\begin{talign}
    L_n(f) &= \frac{1}{n} \sum_{i=1}^n (f(x_i)-y_i)^2 = \staticnnorm{f-\fstar}^2 - \frac{2}{n} \inner{Z}{\boldsymbol \xi} + \frac{1}{n}\sum_{i=1}^n \xi_i^2, \label{eq:Ln-nnorm-relation}
\end{talign}
where $Z \defeq (f(x_1)-\fstar(x_1),\ldots, f(x_n)-\fstar(x_n))$ and $\boldsymbol \xi \defeq (\xi_1,\ldots,\xi_n)$ are vectors in $\reals^n$.
Define the shorthands
\begin{talign}
    \deltakt \defeq \krrkt - \fstar \qtext{and} \deltafull \defeq \krr - \fstar.
\end{talign}

In the sequel, we use the following shorthands:
\begin{talign}
    Z_{full} &\defeq (\deltafull(x_1),\ldots,\deltafull(x_n)) \qtext{and} Z_{KT} \defeq (\deltakt(x_1),\ldots,\deltakt(x_n)). \label{eq:Z_KT}
\end{talign}

Now for the main argument to bound $\staticnnorm{\krrkt-\fstar}^2$.
When $\staticnnorm{\deltakt}< t$, we immediately have $\staticnnorm{\deltakt}^2<t^2$, which implies \cref{eq:Ln-claim}.
Thus, we may assume that $\staticnnorm{\deltakt}\geq  t$. Note that
\begin{talign}
    \staticnnorm{\deltakt}^2 &\seq{\cref{eq:Ln-nnorm-relation}} L_n(\krrkt) + \frac{2}{n} \inner{Z_{KT}}{\xi} - \frac{1}{n}\sum_{i=1}^n \xi_i^2 \\
    &= L_n(\krr) + \brackets{L_n(\krrkt) - L_n(\krr)} + \frac{2}{n} \inner{Z_{KT}}{\xi} - \frac{1}{n}\sum_{i=1}^n \xi_i^2.
\end{talign}
Given the optimality of $\krr$ on the objective \cref{eq:krr-objective}, we have 
\begin{talign}
    L_n(\krr) \leq \frac{1}{n}\sumn \xi_i^2 + \lambda\braces{ \knorm{\fstar}^2 - \knorm{\krr}^2 } \leq \frac{1}{n}\sumn \xi_i^2 + \lambda \knorm{\fstar}^2,
\end{talign}
where the last inequality follows trivially from dropping the $-\knorm{\krr}^2$ term. Thus,
\begin{talign}
    \staticnnorm{\deltakt}^2 &= \frac{1}{n}\sumn \xi_i^2 + \lambda\knorm{\fstar}^2 + \brackets{L_n(\krrkt) - L_n(\krr)} + \frac{2}{n} \inner{Z_{KT}}{\xi} - \frac{1}{n}\sum_{i=1}^n \xi_i^2 \\
    &\leq \frac{2}{n} \inner{Z_{KT}}{\xi} + \lambda \knorm{\fstar}^2 + \brackets{L_n(\krrkt) - L_n(\krr)}. \label{eq:1}
\end{talign}
Using standard arguments to bound the term $\frac{2}{n} \inner{Z_{KT}}{\xi}$, we claim that on the event $\mc E_{\what{\Delta}_{\trm{KT}}}$, we have
\begin{talign}\label{eq:nnorm-deltakt-1}
    \staticnnorm{\deltakt}^2 &\leq c t^2 (\knorm{\fstar}+1)^2 + c' \brackets{L_n(\krrkt) - L_n(\krr)}
\end{talign}
for some positive constants $c,c'$, and that $\Prob(\mc E_{\what{\Delta}_{\trm{KT}}}\mid\inputcoreset) \geq 1 - e^{-\frac{nt^2}{2\sigma^2}}$. We defer the proof of claim \cref{eq:nnorm-deltakt-1} to the end of this section. 

Now we bound the stochastic term $\brackets{L_n(\krrkt) - L_n(\krr)}$ in \cref{eq:nnorm-deltakt-1}---first observe the following decomposition:
\begin{talign}\label{eq:Ln-diff}
    L_n(\krrkt) - L_n(\krr) &= \parenth{L_n(\krrkt) - L_{\nout}(\krrkt)} + \parenth{L_{\nout}(\krrkt) -  L_n(\krr)}.
\end{talign}
On the event $\ktgoodgen$ \cref{eq:ktgoodgen}, the first term in the display can be bounded by
\begin{talign}\label{eq:loss-diff-krrkt}
    L_n(\krrkt) - L_{\nout}(\krrkt) \sless{\cref{eq:krrkt-diff}} (\knorm{\krrkt}^2 + 2) ~\ktcomplexity.
\end{talign}
Note that $\krrkt$ is the solution to the following optimization problem,
\begin{talign}\label{eq:krr-coreset}
    \min_{f\in \rkhs(\kernel)} L_{\nout}(f) + \lambda' \knorm{f}^2,
\end{talign}
so the second term in the display can be bounded by the following basic inequality
\begin{talign}\label{eq:basic-ineq-lnout}
    L_{\nout}(\krrkt) + \lambda' \knorm{\krrkt}^2 & \leq L_{\nout}(\krr) + \lambda' \knorm{\krr}^2 \\
    \text{so that}\quad L_{\nout}(\krrkt) - L_{\nout}(\krr)
    &\leq \lambda' \braces{ \knorm{\krr}^2 - \knorm{\krrkt}^2 }.
\end{talign}
Thus, on event $\ktgoodgen$, we have
\begin{talign}
    L_n(\krrkt) - L_n(\krr) &\leq (\knorm{\krrkt}^2 + 2)~ \ktcomplexity + \lambda' \braces{ \knorm{\krr}^2 - \knorm{\krrkt}^2 } \\
    &= 2 \ktcomplexity + \lambda' \knorm{\krr}^2 + \braces{\ktcomplexity - \lambda'} \cdot \knorm{\krrkt}^2 \\
    &\sless{(i)} 2 \ktcomplexity + \lambda' \knorm{\krr}^2 \sless{(ii)} \lambda' (\knorm{\krr}^2 + 1)
\end{talign}
where steps (i) and (ii) both follow from the fact that $\lambda' \geq 2 \ktcomplexity$ (see assumptions in \cref{thm:krr-kt}). To bound $\knorm{\krr}^2$, we use the following lemma:

\begin{lemma}[RKHS norm of $\krr$]\label{lem:krr-h-norm}
On event $\fullnormbound$ \cref{eq:fullnormbound}, we have the following bound
\begin{talign}\label{eq:krr-h-norm}
    \knorm{\krr}^2 \leq c_0 (\knorm{\fstar} + 1)^2 
\end{talign}
for some constant $c_0>0$.
\end{lemma}
\proofref{proof:krr-h-norm}
Putting things together, we have
\begin{talign}
    L_n(\krrkt) - L_n(\krr) &\leq c \lambda' (\knorm{\fstar} + 1)^2
\end{talign}
for some constant $c$---substituting this bound into \cref{eq:nnorm-deltakt-1} yields
\begin{talign}
    \staticnnorm{\krrkt-\fstar}^2 &\leq c t^2 (\knorm{\fstar}+1)^2 + c' \lambda' (\knorm{\fstar} + 1)^2,
\end{talign}
for some constants $c, c'$, which directly implies \cref{eq:Ln-claim}, i.e. implying
\begin{align}
     \sbraces{\fullnormbound(t) \cap \ktgoodgen \cap \mc E_{\what{\Delta}_{\trm{KT}}}(t) } \implies \ktnormbound(t).
\end{align}

\vspone[-4]
\paragraph{Proof of claim~\cref{eq:krrkt-diff}.}
Given $f$, define the function
\begin{talign}\label{eq:ell_f_prime}
    \ell_f': \X \times \Y \to \reals,\qtext{where} \ell_f'(x,y) \defeq f^2(x) - 2y\cdot f(x)
\end{talign}
and note that
\begin{talign}\label{eq:Ln-minus-Lnout}
    L_n(f) - L_{\nout}(f) = \frac{1}{n}\sumn \ell_f'(x_i,y_i) - \frac{1}{\nout}\sumnout \ell_f'(x_i',y_i').
\end{talign}
We first prove a generic technical lemma:

\begin{lemma}[\ktcompresspp approximation bound using $\kernelrr$] \label{lem:loss-diff}
Suppose $f_1,f_2\in \rkhs(\kernel)$ and $a,b\in \reals$. Then the function
\begin{talign}\label{eq:Ln-minus-Lnout-ell_f'}
    \ell_{f_1,f_2}: \X \times \Y \to \reals,\qtext{where} \ell_{f_1,f_2} (x,y) \defeq a \cdot f_1(x) f_2(x) + b \cdot y f_1(x)
\end{talign}
lies in the RKHS $\rkhs(\kernelrr)$.
Moreover, on event $\ktgoodgen$, we have
\begin{talign}
    \Pin \ell_{f_1,f_2} - \Qout \ell_{f_1,f_2} \leq (\abss{a}\cdot \knorm{f_1} \knorm{f_2} + \abss{b}\cdot \knorm{f_2}) \cdot \ktcomplexity.
\end{talign}
uniformly for all non-zero $f_1,f_2\in \rkhs(\kernel)$.
\end{lemma}
\proofref{proof:loss-diff}
Applying the lemma with $f_1\defeq f,f_2\defeq g$ and $a=1,b=-2$, we have
\begin{talign}
    \Pin \ell_{f}' - \Qout \ell_{f}' \leq (\knorm{f}^2 + 2) \cdot \ktcomplexity,
\end{talign}
which combined with the observation \cref{eq:Ln-minus-Lnout-ell_f'} yields the desired claim.

\paragraph{Proof of claim~\cref{eq:nnorm-deltakt-1}.}
\underline{Case I}:\quad 
First suppose that $\knorm{\deltakt} \leq 1$.
Recall that $\staticnnorm{\deltakt}\geq t \geq \vareps_n$ by assumption.
Thus, we may apply \cite[Lem.~13.12]{wainwright2019high} to obtain
\begin{talign}\label{eq:Z_kt-inner}
    \frac{1}{n}\inner{Z_{KT}}{\boldsymbol \xi} \leq 2 \staticnnorm{\deltakt} t \qtext{w.p. at least}~1-e^{-\frac{n t^2}{2\sigma^2}}
\end{talign}
Plugging the above bound into \cref{eq:1}, we have with probability at least $1 - e^{-\frac{n t^2}{2\sigma^2}}$:
\begin{talign}
    \staticnnorm{\deltakt}^2 &\leq 4 \staticnnorm{\deltakt} t + \lambda \knorm{\fstar}^2 + \brackets{L_n(\krrkt) - L_n(\krr)}.
\end{talign}

We can solve for $\staticnnorm{\deltakt}$ using the quadratic formula. Specifically, if $a,b\geq 0$ and $x^2 - ax - b\leq 0$, then $x\leq a + \sqrt b$. Thus, we have with probability at least $1 - e^{-\frac{n\vareps_n^2}{2\sigma^2}}$:
\begin{talign}
    \staticnnorm{\deltakt} &\leq a + \sqrt b, \qtext{where}\\
    a &\defeq 4 t \qtext{and} \\
    b &\defeq \lambda \knorm{\fstar}^2 + \brackets{L_n(\krrkt) - L_n(\krr)}.
\end{talign}
Using the fact that $(a+\sqrt b)^2 \leq 2a^2 + 2b$, we have with probability at least $1 - e^{-\frac{n t^2}{2\sigma^2}}$:
\begin{talign}
    \staticnnorm{\krrkt-\fstar}^2 &\leq 32 t^2 + 2\lambda \knorm{\fstar}^2 + 2 \brackets{L_n(\krrkt) - L_n(\krr)} \\
    &\sless{\cref{eq:lambda}} c t^2 (\knorm{\fstar}+1)^2 + 2\brackets{L_n(\krrkt) - L_n(\krr)}
\end{talign}

\underline{Case II}:\quad 
Otherwise, we may assume that $\knorm{\deltakt} > 1$. Now we apply \citep[Thm.~13.23]{wainwright2019high} to obtain
\begin{talign}\label{eq:Z_kt-inner_second}
    \frac{1}{n}\inner{Z_{KT}}{\boldsymbol \xi} \leq 2 t \staticnnorm{\deltakt} + 2 t^2 \knorm{\deltakt} + \frac{1}{16} \staticnnorm{\deltakt}^2 \qtext{w.p. at least}~1 - c_1 e^{-\frac{n t^2}{c_2 \sigma^2}},
\end{talign}
for some universal positive constants $c_1,c_2$. Plugging the above bound into \cref{eq:1} and collecting terms, we have with probability at least $1 - c_1 e^{-\frac{n t^2}{c_2 \sigma^2}}$:
\begin{talign}
    \frac{7}{8}\staticnnorm{\deltakt}^2 \leq 4 t \staticnnorm{\deltakt} + 4 t^2 \knorm{\deltakt} +\lambda \knorm{\fstar}^2 + \brackets{L_n(\krrkt) - L_n(\krr)}.
\end{talign}
Solving for $\staticnnorm{\deltakt}$ using the quadratic formula, we have with probability at least $1 - c_1 e^{-\frac{n t^2}{c_2 \sigma^2}}$:
\begin{talign}
    \staticnnorm{\deltakt} &\leq a + \sqrt b, \qtext{where}\\
    a &\defeq \frac{32}{7} t \qtext{and} \\
    b &\defeq \frac{32}{7} t^2 \knorm{\deltakt}^2 + \frac{8}{7} \lambda \knorm{\fstar}^2 + \frac{8}{7}\brackets{L_n(\krrkt) - L_n(\krr)}.
\end{talign}
Using the fact that $(a+\sqrt b)^2 \leq 2a^2 + 2b$, we have with probability at least $1 - c_1 e^{-\frac{n t^2}{c_2 \sigma^2}}$:
\begin{talign}
    \staticnnorm{\krrkt-\fstar}^2 &\leq 42 t^2 + 10 t^2 \knorm{\deltakt}^2 + 2.3 \lambda \knorm{\fstar}^2 + 2.3 \brackets{ L_n(\krrkt) - L_n(\krr) } \\
    &\sless{(i)} 42 t^2 + 10 t^2 \knorm{\deltakt}^2 + 4.6 t^2 \knorm{\fstar}^2 + 2.3 \brackets{ L_n(\krrkt) - L_n(\krr) } \\
    &\sless{\text{\cref{eq:lambda},\cref{eq:krr-h-norm}}} c_3 t^2 \parenth{\knorm{\fstar} + 1}^2 + c_4 \brackets{ L_n(\krrkt) - L_n(\krr) }
\end{talign}
for some positive constants $c_3,c_4$, where step (i) follows from that fact that $\lambda = 2\vareps_n^2$ by \cref{eq:lambda}.

\subsection{\pcref{lem:krr-h-norm}}
\label{proof:krr-h-norm}

Given the optimality of $\krr$ on the objective \cref{eq:krr-objective}, we have the following basic inequality
\begin{talign}
    & L_n(\krr) + \lambda \knorm{\krr}^2 \leq \frac{1}{n} \sum_{i=1}^n \xi_i^2 + \lambda \knorm{\fstar}^2 \\
    \implies & \knorm{\krr}^2 \leq \knorm{\fstar}^2 + \frac{1}{\lambda} \parenth{ \frac{1}{n} \sum_{i=1}^n \xi_i^2 - L_n(\krr) }.
\end{talign}
Since $\staticnnorm{\krr-\fstar}^2\geq 0$, we also have the trivial lower bound 
\begin{talign}
    L_n(\krr) 
    &\seq{\cref{eq:Ln-nnorm-relation}} \staticnnorm{\krr-\fstar}^2 - \frac{2}{n} \inner{Z_{full}}{\boldsymbol \xi} + \frac{1}{n}\sum_{i=1}^n \xi_i^2 \\
    &\geq - \frac{2}{n} \inner{Z_{full}}{\boldsymbol \xi} + \frac{1}{n}\sum_{i=1}^n \xi_i^2.
\end{talign}
Thus,
\begin{talign}\label{eq:deltafull-knorm-1}
    \knorm{\krr}^2 &\leq \knorm{\fstar}^2 + \frac{1}{\lambda} \parenth{ \frac{2}{n}\inner{Z_{full}}{\xi} }
\end{talign}
and it remains to bound $\frac{2}{n} \inner{Z_{full}}{\xi}$.

\underline{Case I}:\quad First, suppose that $\knorm{\deltafull} > 1$. Then we may apply \citep[Lem.~13.23]{wainwright2019high} to obtain
\begin{talign}
    \frac{1}{n} \inner{Z_{full}}{\xi} \leq 2\vareps_n \staticnnorm{\deltafull} + 2\vareps_n^2 \knorm{\deltafull} + \frac{1}{16} \staticnnorm{\deltafull}^2 \qtext{w.p. at least} 1-c_1 e^{-\frac{n\vareps_n^2}{c_2 \sigma^2}}.
\end{talign}
Combining this bound with \cref{eq:deltafull-knorm-1}, we have with probability at least $1-c_1 e^{-\frac{n\vareps_n^2}{c_2 \sigma^2}}$:
\begin{talign}
    \knorm{\krr}^2 
    &\leq \knorm{\fstar}^2 + \frac{2\vareps_n^2}{\lambda} \knorm{\deltafull} + \frac{2}{\lambda}\parenth{2\vareps_n \staticnnorm{\deltafull} + \frac{1}{16}\staticnnorm{\deltafull}^2} \\
    &\sless{(i)} \knorm{\fstar}^2 + \frac{2\vareps_n^2}{\lambda} (\knorm{\krr} + \knorm{\fstar}) + \frac{2}{\lambda}\parenth{2\vareps_n \staticnnorm{\deltafull} + \frac{1}{16}\staticnnorm{\deltafull}^2} \\
    &\seq{\cref{eq:lambda}} \knorm{\fstar}^2 + \knorm{\krr} + \knorm{\fstar} + \frac{2}{\lambda}\parenth{2\vareps_n \staticnnorm{\deltafull} + \frac{1}{16}\staticnnorm{\deltafull}^2} ,
\end{talign}
where step (i) follows from triangle inequality.
Solving for $\knorm{\krr}$ using the quadratic formula, we have
\begin{talign}
    \knorm{\krr}^2 &\leq 2 + \knorm{\fstar}^2 + \knorm{\fstar} + \frac{2}{\lambda}\parenth{2\vareps_n \staticnnorm{\deltafull} + \frac{1}{16}\staticnnorm{\deltafull}^2}.
\end{talign}

On the event $\fullnormbound$ \cref{eq:fullnormbound}, we have $\staticnnorm{\deltafull} \leq c \knorm{\fstar} \vareps_n$ for some positive constant $c$, which implies the claimed bound \cref{eq:krr-h-norm} after some algebra.

\underline{Case II(a)}:\quad Otherwise, assume $\knorm{\deltafull} \leq 1$ and $\staticnnorm{\deltafull} \leq \vareps_n$. Applying \citep[Thm.~2.26]{wainwright2019high} to the function $\sup_{\substack{\knorm{g}\leq 1 \\ \nnorm{g} \leq \vareps_n}} \abss{\frac{1}{n} \sumn \xi_i g(x_i)}$, we have
\begin{talign}
    \frac{1}{n}\inner{Z_{full}}{\boldsymbol \xi} \leq \frac{\vareps_n^2}{2} \qtext{w.p. at least} 1- e^{-\frac{n \vareps_n^2}{8 \sigma^2}}
\end{talign}
Combining this bound with \cref{eq:deltafull-knorm-1}, we obtain
\begin{talign}
    \knorm{\krr}^2 \leq \knorm{\fstar}^2 + \frac{1}{\lambda} \vareps_n^2 \seq{\cref{eq:lambda}} \knorm{\fstar}^2 + \half,
\end{talign}
which immediately implies the claimed bound \cref{eq:krr-h-norm}.

\underline{Case II(b)}:\quad Finally, assume $\knorm{\deltafull} \leq 1$ and $\staticnnorm{\deltafull} > \vareps_n$. Applying \citep[Lem.~13.12]{wainwright2019high} with $u = \vareps_n$, we have
\begin{talign}
    \frac{1}{n}\inner{Z_{full}}{\boldsymbol \xi} \leq 2 \vareps^2 \qtext{w.p. at least} 1- e^{-\frac{n \vareps_n^2}{2 \sigma^2}}.
\end{talign}
Combining this bound with \cref{eq:deltafull-knorm-1}, we obtain
\begin{talign}
    \knorm{\krr}^2 \leq \knorm{\fstar}^2 + \frac{4}{\lambda} \vareps_n^2 \seq{\cref{eq:lambda}} \knorm{\fstar}^2 + 2,
\end{talign}
which immediately implies the claimed bound \cref{eq:krr-h-norm}.

\subsection{\pcref{lem:loss-diff}}
\label{proof:loss-diff}

By \citet[Lem.~4]{grunewalder2022compressed}, $\ell_{f_1,f_2}$ lies in the RKHS $\rkhs(\kernelrr)$, which is a direct sum of two RKHS:
\begin{talign}
    \rkhs(\kernelrr) &= \rkhs(\kernel_1) \oplus \rkhs(\kernel_2), 
\end{talign}
where $\kernel_1,\kernel_2 : \mc Z \times \mc Z \to \R$ are the kernels defined by
\begin{talign}\label{eq:kernel-1-2}
    \kernel_1((x_1,y_1), (x_2,y_2)) \defeq \kernel^2(x_1,x_2) \qtext{and}
    \kernel_2((x_1,y_1), (x_2,y_2)) \defeq \kernel(x_1,x_2) \cdot y_1 y_2 .
\end{talign}
Applying \cref{lem:ktcompresspp-mmd} with 
\begin{talign}
    \mc Z = \X \times \Y, \quad \kalg = \kernelrr, \qtext{and} \eps^
\star= \frac{ (\ininfnorm{\kernel}^{1/2} + \ymax )^2 }{\nout},
\end{talign}
yields the following bound on event $\ktgoodgen$ \cref{eq:ktgoodgen}:
\begin{talign}
    \sup_{\substack{h\in \rkhs(\kernelrr): \\ \norm{h}_{\kernelrr}\leq 1} } \abss{ (\Pin - \Qout) h } \leq 2\eps^\star + \frac{ \ininfnorm{\kernelrr}^{1/2}}{\nout} \cdot \errmmd_{\kernelrr}(n,\nout, \delta, \mfk R_{\mrm{in}}, \eps^\star).
\end{talign}
We claim that
\begin{talign}
    \ininfnorm{\kernelrr}^{1/2} &\leq \ininfnorm{\kernel} + \ymax^2 \qtext{and} \label{eq:kalg-infnorm-bound} \\
    \log \coveringnumber_{\kernelrr}^\dagger(\inputcoreset,\eps^\star) &\leq c \cdot \log\coveringnumber_{\kernel}(\inputcoreset, \frac{\ininfnorm{\kernel}^{1/2} + \ymax}{\nout}),\label{eq:kalg-coveringnumber-bound}
\end{talign}
for some positive constant $c$, where $\coveringnumber_{\kernelrr}^\dagger$ is the cardinality of the cover of $\mc B_{\kernelrr}^\dagger \defeq \braces{ \ell_f'/\staticnorm{\ell_f'}_{\kernelrr} : f\in \rkhs(\kernel) }$ for $\ell_f'$ defined by \cref{eq:ell_f_prime}. Proof of the claims \cref{eq:kalg-infnorm-bound,eq:kalg-coveringnumber-bound} are deferred to the end of this section. 
By definition of $\errmmd_{\kernelrr}$, we have
\begin{talign}
    \errmmd_{\kernelrr}(n,\nout, \delta,\mfk R_{\mrm{in}}, \eps^\star) \sless{\cref{eq:err-mmd}} \sqrt c \cdot \errmmd_{\kernel}(n,\nout,\delta,\mfk R_{\mrm{in}},\frac{\ininfnorm{\kernel}^{1/2} + \ymax}{\nout}) \defeq \errmmd_{\kernel}'.
\end{talign}
On event $\ktgoodgen$, we have
\begin{talign}
    \sup_{\substack{h \in \rkhs(\kernelrr): \\ \norm{h}_{\kernelrr}\leq 1} } \abss{ (\Pin - \Qout) h} 
    &\leq 
    \frac{ 2 (\ininfnorm{\kernel}^{1/2} + \ymax )^2 }{\nout} + \frac{ \ininfnorm\kernel + \ymax^2 }{\nout} \cdot \errmmd_{\kernel}' \\
    &\seq{(i)} \frac{ \ininfnorm\kernel + \ymax^2 }{\nout}\cdot \brackets{2 + \errmmd_{\kernel}'}, \label{eq:MMD-ell_f}
\end{talign}
where step (i) follows from the fact that $(\ininfnorm{\kernel}^{1/2} + \ymax )^2 \leq 2 (\ininfnorm{\kernel} + \ymax^2)$. 

Since $f_1,f_2$ are non-zero, we have $\knorm{\ell_{f_1,f_2}} >0$. Thus, the function $h\defeq \ell_f / \staticnorm{\ell_f}_{\kernelrr} \in \rkhs(\kernelrr)$ is well-defined and satisfies $\norm{h}_{\kernelrr} = 1$. Applying \cref{eq:MMD-ell_f}, we obtain
\begin{talign}
    \abss{ \Pin h - \Qout h } \leq \frac{ \ininfnorm\kernel + \ymax^2 }{\nout}\cdot (2 + \errmmd_{\kernel}') \qtext{on event} \ktgoodgen.
\end{talign}
Multiplying both sides by $\norm{\ell_f}_{\kernelrr}$ and noting that
\begin{talign}\label{eq:ell_f-hnorm}
    \norm{\ell_{f,g}}_{\kernelrr}^2 &= \norm{a\cdot f_1 f_2}_{\widehat{\rkhs \odot \rkhs}}^2 + \norm{b\cdot f_2\otimes \langle \cdot,1\rangle_{\R}}_{\rkhs \otimes \mc{R}}^2 \\
    &\leq a^2 \knorm{f_1}^2 \knorm{f_2}^2 + b^2 \knorm{f_2}^2 \\
    &\leq (\abss{a}\cdot \knorm{f_1} \knorm{f_2} + \abss{b}\cdot \knorm{f_2})^2,
\end{talign}
we have on event $\ktgoodgen$,
\begin{talign}
    \Pin \ell_{f_1,f_2} - \Qout \ell_{f_1,f_2} \leq (\abss{a} \cdot \knorm{f_1} \knorm{f_2} + \abss{b}\cdot \knorm{f_2}) \cdot \frac{ \ininfnorm\kernel + \ymax^2 }{\nout}\cdot (2 + \errmmd_{\kernel}'),
\end{talign}
which directly implies the bound \cref{eq:krrkt-diff}  after applying the shorthand \cref{eq:kt-complexity}.

\paragraph{Proof of \cref{eq:kalg-infnorm-bound}}
Define $\ymax \defeq \sup_{y\in (\inputcoreset)_y}y$. We have
\begin{talign}
    \ininfnorm{\kalg} &= \sup_{(x_1,y_1), (x_2,y_2)\in \inputcoreset} \braces{ \kernel(x_1,x_2)^2 + \kernel(x_1,x_2) \cdot y_1 y_2 + (y_1 y_2)^2 } \\
    &\leq \sup_{x_1,x_2\in (\inputcoreset)_x} \kernel(x_1,x_2)^2 + \sup_{x_1,x_2\in (\inputcoreset)_x} \kernel(x_1,x_2) \cdot \sup_{y_1,y_2\in (\inputcoreset)_y}y_1 y_2 \\
    &\qquad + \sup_{y_1,y_2\in (\inputcoreset)_y} (y_1 y_2)^2\\
    &= \ininfnorm\kernel^2 + \ininfnorm\kernel \cdot \ymax^2 +  \ymax^4 \\
    &\leq \parenth{ \ininfnorm\kernel + \ymax^2 }^2 \label{eq:kalg-infnorm}.
\end{talign}

\paragraph{Proof of \cref{eq:kalg-coveringnumber-bound}}
\tocheck{}
Since $\rkhs(\kernelrr)$ is a direct sum, we have
\begin{talign}\label{eq:log-coveringnumber-sum}
    \log\coveringnumber_{\kernelrr}^\dagger(\inputcoreset,\eps^
    \star) \leq \log\coveringnumber_{\kernel_1}^\dagger(\inputcoreset, \eps^
    \star/2) + \log\coveringnumber_{\kernel_2}^\dagger(\inputcoreset, \eps^
    \star/2),
\end{talign}
where $\coveringnumber_{\kernel_1}^\dagger$ and $\log\coveringnumber_{\kernel_2}^\dagger$ are the covering numbers of $\mc B_{\kernel_1}^\dagger \defeq \braces{ f^2/\staticnorm{f^2}_{\kernel_1} : f\in \rkhs(\kernel) }$ and $\mc B_{\kernel_2}^\dagger \defeq \braces{ f\otimes \inner{\cdot}{y}_\reals/\staticnorm{f\otimes \inner{\cdot}{y}_\reals}_{\kernel_2} : f\otimes \inner{\cdot}{y}_\reals\in \rkhs(\kernel_2) }$, respectively.

Note that
\begin{talign}\label{eq:covering-kernel-1}
    \log\coveringnumber_{\kernel_1}^\dagger(\inputcoreset, \eps^\star) &\sless{} 2 \log \coveringnumber_{\kernel}(\inputcoreset, \eps^\star/(2 \infnorm{\kernel}^{1/2})) 
    \\
    &\leq 2 \log \coveringnumber_{\kernel}(\inputcoreset, (1+\frac{\ymax}{\ininfnorm{\kernel}^{1/2}}) \frac{\ininfnorm{\kernel}^{1/2} + \ymax}{2\nout}) 
    \\
    &\leq 2 \log \coveringnumber_{\kernel}(\inputcoreset, \frac{\ininfnorm{\kernel}^{1/2} + \ymax}{2\nout}).
\end{talign}
Define a kernel on $\reals$ by $\kernel_{\R}(y_1,y_2)\defeq y_1 y_2$. When $\sup_{y\in (\inputcoreset)_y} \abss{y} \leq \ymax$, we have
\begin{talign}
    \coveringnumber_{\kernel_{\R}}([-\ymax,\ymax], \eps) = \bigO{ \ymax^2 / \eps } \qtext{for}\eps > 0.
\end{talign}

Similarly, note that
\begin{talign}\label{eq:covering-kernel-2}
    \log\coveringnumber_{\kernel_2}^\dagger(\inputcoreset, \eps^\star) &\leq \log\coveringnumber_{\kernel}(\inputcoreset, \eps^\star/(\infnorm{\kernel}^{1/2} + \infnorm{\kernel_{\R}}^{1/2})) + \log \coveringnumber_{\kernel_{\R}}(\inputcoreset, \eps^\star/(\ininfnorm{\kernel}^{1/2} + \infnorm{\kernel_{\R}}^{1/2})) \\
    & \lesssim \log\coveringnumber_{\kernel}(\inputcoreset, \frac{\ininfnorm{\kernel}^{1/2} + \ymax}{\nout}) + \log \parenth{\frac{\ymax^2 (\infnorm{\kernel}^{1/2} + \ymax) }{\nout}}
\end{talign}
Substituting the above two log-covering number expressions into \cref{eq:log-coveringnumber-sum} yields
\begin{talign}
    \log \coveringnumber_{\kalg}(\inputcoreset,\eps^\star) &\lesssim 3\log\coveringnumber_{\kernel}(\inputcoreset, \frac{\ininfnorm{\kernel}^{1/2} + \ymax}{\nout}) + \log \parenth{\frac{\ymax^2 (\ininfnorm{\kernel}^{1/2} + \ymax) }{\nout}}. \\
    &\leq c \cdot \log\coveringnumber_{\kernel}(\inputcoreset, \frac{\ininfnorm{\kernel}^{1/2} + \ymax}{\nout})
\end{talign}
for some universal positive constant $c$.

\section{Non-compact kernels satisfying \cref{assum:nw-kernel}}
\label{app:nw-kernel-examples}
The boundedness, Lipschitz assumption, square-integrability, and \cref{eq:kappa-dagger} follow from \citep[App.~O]{dwivedi2024kernel} and \citep[Rmk.~8]{dwivedi2024kernel}. Thus, it only remains to verify \cref{eq:nw-kernel-cond1,eq:nw-kernel-cond2} for each kernel.

\subsection{Gaussian}
For the kernel by $\k(x_1,x_2) \defeq \exp(-\frac{\snorm{x_1-x_2}^2}{2h^2})$, we have $\kappa(u) \defeq \exp(-u^2/2)$. 

\paragraph{Verifying \cref{eq:nw-kernel-cond1}.} 
For any $j\geq 1$ and $x\in \X$, we have
\begin{talign}
    \int_{\snorm{z}\in [(2^{j-1}-\half)h,(2^j+\half)h]} \k(x,x-z) dz 
    &\leq \int_{\snorm{z}\geq (2^{j-1}-\half)h} \k(x,x-z) dz \\
    &= \parenth{2\pi h^2}^{d/2} \Psubarg{X\sim \Gsn(0,h^2 \ident_d)}{\stwonorm{X} \geq (2^{j-1}-\half)h } \\
    &= \parenth{2\pi h^2}^{d/2} \Psubarg{X\sim \Gsn(0,\ident_d)}{\stwonorm{X}^2 \geq (2^{j-1}-\half)^2} \\
    &= \parenth{2\pi h^2}^{d/2} \Psubarg{X\sim \Gsn(0,\ident_d)}{\stwonorm{X}^2 - d \geq (2^{j-1}-\half)^2 - d}.\label{eq:gsn-cond1-step1}
\end{talign}
By \citep[Lem.~1]{laurent2000adaptive}, we have
\begin{talign}\label{eq:chi-squared-tail-bound}
    \Psubarg{X\sim \Gsn(0,\ident_d)}{\stwonorm{X}^2 - d > 2 \sqrt{dt} + 2t} \leq e^{-t}\qtext{for any} t\geq 0.
\end{talign}
Define $t\defeq 2 d$ and $R \defeq 2\sqrt{t}$. We can directly verify that $R^2 - d \geq 2\sqrt{dt} + 2t$. Thus, we may further upper bound \cref{eq:gsn-cond1-step1} by
\begin{talign}
    \Psubarg{X\sim \Gsn(0,\ident_d)}{\stwonorm{X}^2 - d \geq (2^{j-1}-\half)^2 - d} 
    &\leq \Psubarg{X\sim \Gsn(0,\ident_d)}{\stwonorm{X}^2 - d \geq 2\sqrt{dt} + 2t}  \\
    &\sless{\cref{eq:chi-squared-tail-bound}} e^{-R^2/4} \leq e^{-(2^{j-2} - \frac{1}{4})^2}, 
\end{talign}
whenever $R\geq 2\sqrt{2 d}$, or equivalently, $j \geq \log_2(1 + 4\sqrt{2d})$. Under this regime, we have
\begin{talign}
    2^{j\beta} \int_{\snorm{z}\in [(2^{j-1}-\half)h,(2^j+\half)h]} \k(x,x-z) dz \leq 2^{j\beta} (2\pi h^2)^{d/2} e^{-(2^{j-1}-\frac{1}{4})^2} < C
\end{talign}
for some positive constant $C$ that does not depend on $j$ or $n$ as desired.

\paragraph{Verifying \cref{eq:nw-kernel-cond2}.}
Fixing $j \geq 1$, we have
\begin{talign}
    (2^j + \half)^d 2^{j\beta} \kappa(2^{j-1}-1) &\leq 2^{j(d+\beta) + d} \exp(-\half (2^{j-1}-1)^2) \\
    &\leq \exp\parenth{(\log 2) (j(d+\beta)+d) - (2^{j-1}-1)^2)}\\
    &\leq \exp\parenth{\max_{j\in [T]}\braces{(\log 2) (j(d+\beta)+d) - (2^{j-1}-1)^2)}}
\end{talign}
Note that $(\log 2) (j(d+\beta)+d) - (2^{j-1}-1)^2$ is concave over $j> 1$ and maximized at $j=\floor{\log(1+ \sqrt{1 + 2(\beta+d)})}$. Treating $\int_0^{\half} \kappa(u) u^{d-1} du$ as a constant, there exists $c_2$ such that $\cref{eq:nw-kernel-cond2}$ is satisfied for all positive integers $j$.

\subsection{\Matern}

We consider the kernel $\k(x_1,x_2) \defeq c_{\nu-\frac{d}{2}}(\frac{\stwonorm{x_1-x_2}}{h})^{\nu-\frac{d}{2}} K_{\nu-\frac{d}{2}}(\frac{\stwonorm{x_1-x_2}}{h})$ with $\nu>d$, where $K_a$ denotes the modified Bessel function of the third kind \citep[Def.~5.10]{wendland2004scattered}, and $c_b \defeq \frac{2^{1-b}}{\Gamma(b)}$. We have $\k(x_1,x_2) = \kappa(\frac{\stwonorm{x_1-x_2}}{h}) \defeq \wtil \kappa_{\nu-\frac{d}{2}}(\frac{\stwonorm{x_1-x_2}}{h})$, where
\begin{talign}
    \wtil \kappa_b(u) \defeq c_b u^b K_b(u).
\end{talign}

\paragraph{Verifying \cref{eq:nw-kernel-cond1}.} 
Fix $j\in \braces{0,1,\ldots,T}$. Applying Jensen's inequality, we have
\begin{talign}
    \sup_{x\in \X} \int_{\snorm{z}\in [(2^{j-1}- \half)h,(2^j + \half)h]} \k(x,x-z) dz &\leq \sup_{x \in \X} \parenth{\int_{\snorm{z}\in [(2^{j-1} - \half)h,(2^j + \half)h]} \k^2(x,x-z) dx}^\half \\
    &\leq \sup_{x \in \X} \parenth{\int_{\snorm{z} \geq (2^{j-1} - \half)h} \k^2(x,x-z) dx}^\half \\
    &= \tau_{\k}((2^{j-1} -\half)h ),
\end{talign}
where $\tau_\k(\cdot)$ is the kernel tail decay defined by \citep[Assum.~3]{dwivedi2024kernel}. We may use the same logic as in \citep[App.~O.3.6]{dwivedi2024kernel}---but ignoring the $A_{\nu,\gamma,d}$ term and making the substitutions, $a\gets \nu - \frac{d}{2}$, $\gamma\gets \frac{1}{h}$, and $\Gamma(\nu-1) \gets \Gamma(2\nu-1)$---to obtain
\begin{talign}
    \tau_{\k}^2(R) &\sless{} h^d \cdot 2\pi \frac{2^{2-2\nu -d}}{\Gamma^2(\nu-\frac{d}{2})} \cdot \frac{\pi^{\frac{d}{2}}}{\Gamma(\frac{d}{2}+1)} \cdot \Gamma(2\nu-1) \exp(-\frac{R}{2h})
    \qtext{for}
    R\geq \frac{2\nu-d}{\sqrt 2} h.
\end{talign}
Hence, we have
\begin{talign}
    2^{j\beta} \sup_{x\in \X} \int_{\snorm{z}\in [(2^{j-1}- \half)h,(2^j + \half)h]} \k(x,x-z) dz \leq 2^{j\beta} C_{h,d,\subg} \exp(-\frac{1}{4} (2^{j-1} - \half)) = \bigO{1},
\end{talign}
whenever $(2^{j-1}-\half) h \geq \frac{2\subg - d}{\sqrt 2}h \iff j\geq 1 + \log(\frac{2\subg-d}{\sqrt 2}+\half)$ for some constant $C_{h,d,\subg}$ that does not depend on $j$.

\paragraph{Verifying \cref{eq:nw-kernel-cond2}.}
Following similar logic as in \citep[App.~O.3.1]{dwivedi2024kernel}, we have
\begin{talign}
    \wtil \kappa_a(u) \leq \min\braces{1, \sqrt{2\pi} c_a u^{a-1} \exp\parenth{-\frac{u}{2}}} \qtext{for} u\geq 2(a-1).
\end{talign}
Fixing $j\geq 1$, we have
\begin{talign}
    (2^j+1)^d 2^{j\beta} \kappa(2^{j-1}-1) \leq 2^{j(d+\beta) + d} \min\braces{ 1, \sqrt{2\pi} c_{\nu-\frac{d}{2}} (2^{j-1}-1)^{\nu-\frac{d}{2}-1} \exp(-\half (2^{j-1}-1)) }.
\end{talign}
Note that when $\nu-\frac{d}{2}-1<1$, the RHS can be rewritten as $C_{\beta,d} (2^{j-1})^{d+\beta} \exp(\half 2^{j-1})$ for some constant $C_{\beta,d}$ that doesn't depend on $j$. When $\nu-\frac{d}{2}-1\geq 1$, the RHS can be upper-bounded by $C_{\beta,d}' (2^{j-1})^{j(\nu+\beta+\frac{d}{2}-1)} \exp(-\half 2^{j-1})$ for some constant $C_{\beta,d}'$ that doesn't depend on $j$.
Observe that the function $t^b e^{-t/2}$ attains its maximum at $t=2b$. Hence, the RHS can be bounded by a constant that does not depend on $n$ as desired.

\newpage
\section*{NeurIPS Paper Checklist}

\begin{enumerate}

\item {\bf Claims}
    \item[] Question: Do the main claims made in the abstract and introduction accurately reflect the paper's contributions and scope?
    \item[] Answer: \answerYes{} %
    \item[] Justification: Abstract and introduction give clear outlines on our contributions, and we present our contributions in the main text accordingly. We have also included pointers in the introduction that would link to the referred main text containing specific contributions.
    \item[] Guidelines:
    \begin{itemize}
        \item The answer NA means that the abstract and introduction do not include the claims made in the paper.
        \item The abstract and/or introduction should clearly state the claims made, including the contributions made in the paper and important assumptions and limitations. A No or NA answer to this question will not be perceived well by the reviewers. 
        \item The claims made should match theoretical and experimental results, and reflect how much the results can be expected to generalize to other settings. 
        \item It is fine to include aspirational goals as motivation as long as it is clear that these goals are not attained by the paper. 
    \end{itemize}

\item {\bf Limitations}
    \item[] Question: Does the paper discuss the limitations of the work performed by the authors?
    \item[] Answer: \answerYes{} %
    \item[] Justification: In \cref{sec:conclusions}, we discuss limitations as well as future work to address these limitations.
    \item[] Guidelines:
    \begin{itemize}
        \item The answer NA means that the paper has no limitation while the answer No means that the paper has limitations, but those are not discussed in the paper. 
        \item The authors are encouraged to create a separate "Limitations" section in their paper.
        \item The paper should point out any strong assumptions and how robust the results are to violations of these assumptions (e.g., independence assumptions, noiseless settings, model well-specification, asymptotic approximations only holding locally). The authors should reflect on how these assumptions might be violated in practice and what the implications would be.
        \item The authors should reflect on the scope of the claims made, e.g., if the approach was only tested on a few datasets or with a few runs. In general, empirical results often depend on implicit assumptions, which should be articulated.
        \item The authors should reflect on the factors that influence the performance of the approach. For example, a facial recognition algorithm may perform poorly when image resolution is low or images are taken in low lighting. Or a speech-to-text system might not be used reliably to provide closed captions for online lectures because it fails to handle technical jargon.
        \item The authors should discuss the computational efficiency of the proposed algorithms and how they scale with dataset size.
        \item If applicable, the authors should discuss possible limitations of their approach to address problems of privacy and fairness.
        \item While the authors might fear that complete honesty about limitations might be used by reviewers as grounds for rejection, a worse outcome might be that reviewers discover limitations that aren't acknowledged in the paper. The authors should use their best judgment and recognize that individual actions in favor of transparency play an important role in developing norms that preserve the integrity of the community. Reviewers will be specifically instructed to not penalize honesty concerning limitations.
    \end{itemize}

\item {\bf Theory Assumptions and Proofs}
    \item[] Question: For each theoretical result, does the paper provide the full set of assumptions and a complete (and correct) proof?
    \item[] Answer: \answerYes{} %
    \item[] Justification: We have clarified the assumptions and models required for all the theorems and corollaries provided in the main text and appendix. Also we provide a complete proof in the appendix for all the stated results. 
    \item[] Guidelines:
    \begin{itemize}
        \item The answer NA means that the paper does not include theoretical results. 
        \item All the theorems, formulas, and proofs in the paper should be numbered and cross-referenced.
        \item All assumptions should be clearly stated or referenced in the statement of any theorems.
        \item The proofs can either appear in the main paper or the supplemental material, but if they appear in the supplemental material, the authors are encouraged to provide a short proof sketch to provide intuition. 
        \item Inversely, any informal proof provided in the core of the paper should be complemented by formal proofs provided in appendix or supplemental material.
        \item Theorems and Lemmas that the proof relies upon should be properly referenced. 
    \end{itemize}

    \item {\bf Experimental Result Reproducibility}
    \item[] Question: Does the paper fully disclose all the information needed to reproduce the main experimental results of the paper to the extent that it affects the main claims and/or conclusions of the paper (regardless of whether the code and data are provided or not)?
    \item[] Answer: \answerYes{} %
    \item[] Justification: We describe details of the experiments in \cref{sec:results} and provide links to all code and data.
    \item[] Guidelines:
    \begin{itemize}
        \item The answer NA means that the paper does not include experiments.
        \item If the paper includes experiments, a No answer to this question will not be perceived well by the reviewers: Making the paper reproducible is important, regardless of whether the code and data are provided or not.
        \item If the contribution is a dataset and/or model, the authors should describe the steps taken to make their results reproducible or verifiable. 
        \item Depending on the contribution, reproducibility can be accomplished in various ways. For example, if the contribution is a novel architecture, describing the architecture fully might suffice, or if the contribution is a specific model and empirical evaluation, it may be necessary to either make it possible for others to replicate the model with the same dataset, or provide access to the model. In general. releasing code and data is often one good way to accomplish this, but reproducibility can also be provided via detailed instructions for how to replicate the results, access to a hosted model (e.g., in the case of a large language model), releasing of a model checkpoint, or other means that are appropriate to the research performed.
        \item While NeurIPS does not require releasing code, the conference does require all submissions to provide some reasonable avenue for reproducibility, which may depend on the nature of the contribution. For example
        \begin{enumerate}
            \item If the contribution is primarily a new algorithm, the paper should make it clear how to reproduce that algorithm.
            \item If the contribution is primarily a new model architecture, the paper should describe the architecture clearly and fully.
            \item If the contribution is a new model (e.g., a large language model), then there should either be a way to access this model for reproducing the results or a way to reproduce the model (e.g., with an open-source dataset or instructions for how to construct the dataset).
            \item We recognize that reproducibility may be tricky in some cases, in which case authors are welcome to describe the particular way they provide for reproducibility. In the case of closed-source models, it may be that access to the model is limited in some way (e.g., to registered users), but it should be possible for other researchers to have some path to reproducing or verifying the results.
        \end{enumerate}
    \end{itemize}

\item {\bf Open access to data and code}
    \item[] Question: Does the paper provide open access to the data and code, with sufficient instructions to faithfully reproduce the main experimental results, as described in supplemental material?
    \item[] Answer: \answerYes{} %
    \item[] Justification: We provide a link to our GitHub repository containing all code in \cref{sec:results}. 
    \item[] Guidelines:
    \begin{itemize}
        \item The answer NA means that paper does not include experiments requiring code.
        \item Please see the NeurIPS code and data submission guidelines (\url{https://nips.cc/public/guides/CodeSubmissionPolicy}) for more details.
        \item While we encourage the release of code and data, we understand that this might not be possible, so “No” is an acceptable answer. Papers cannot be rejected simply for not including code, unless this is central to the contribution (e.g., for a new open-source benchmark).
        \item The instructions should contain the exact command and environment needed to run to reproduce the results. See the NeurIPS code and data submission guidelines (\url{https://nips.cc/public/guides/CodeSubmissionPolicy}) for more details.
        \item The authors should provide instructions on data access and preparation, including how to access the raw data, preprocessed data, intermediate data, and generated data, etc.
        \item The authors should provide scripts to reproduce all experimental results for the new proposed method and baselines. If only a subset of experiments are reproducible, they should state which ones are omitted from the script and why.
        \item At submission time, to preserve anonymity, the authors should release anonymized versions (if applicable).
        \item Providing as much information as possible in supplemental material (appended to the paper) is recommended, but including URLs to data and code is permitted.
    \end{itemize}

\item {\bf Experimental Setting/Details}
    \item[] Question: Does the paper specify all the training and test details (e.g., data splits, hyperparameters, how they were chosen, type of optimizer, etc.) necessary to understand the results?
    \item[] Answer: \answerYes{} %
    \item[] Justification: We provide train-test splits and hyperparameters in the main paper. 
    \item[] Guidelines:
    \begin{itemize}
        \item The answer NA means that the paper does not include experiments.
        \item The experimental setting should be presented in the core of the paper to a level of detail that is necessary to appreciate the results and make sense of them.
        \item The full details can be provided either with the code, in appendix, or as supplemental material.
    \end{itemize}

\item {\bf Experiment Statistical Significance}
    \item[] Question: Does the paper report error bars suitably and correctly defined or other appropriate information about the statistical significance of the experiments?
    \item[] Answer: \answerYes{} %
    \item[] Justification: In all figures, we plot error bars representing standard deviation across 100 trials. In all tables, we report mean +/- standard error across 100 trials.
    \item[] Guidelines:
    \begin{itemize}
        \item The answer NA means that the paper does not include experiments.
        \item The authors should answer "Yes" if the results are accompanied by error bars, confidence intervals, or statistical significance tests, at least for the experiments that support the main claims of the paper.
        \item The factors of variability that the error bars are capturing should be clearly stated (for example, train/test split, initialization, random drawing of some parameter, or overall run with given experimental conditions).
        \item The method for calculating the error bars should be explained (closed form formula, call to a library function, bootstrap, etc.)
        \item The assumptions made should be given (e.g., Normally distributed errors).
        \item It should be clear whether the error bar is the standard deviation or the standard error of the mean.
        \item It is OK to report 1-sigma error bars, but one should state it. The authors should preferably report a 2-sigma error bar than state that they have a 96\% CI, if the hypothesis of Normality of errors is not verified.
        \item For asymmetric distributions, the authors should be careful not to show in tables or figures symmetric error bars that would yield results that are out of range (e.g. negative error rates).
        \item If error bars are reported in tables or plots, The authors should explain in the text how they were calculated and reference the corresponding figures or tables in the text.
    \end{itemize}

\item {\bf Experiments Compute Resources}
    \item[] Question: For each experiment, does the paper provide sufficient information on the computer resources (type of compute workers, memory, time of execution) needed to reproduce the experiments?
    \item[] Answer: \answerYes{} %
    \item[] Justification: We indicate the computer resources for running all experiments in \cref{sec:results}.
    \item[] Guidelines:
    \begin{itemize}
        \item The answer NA means that the paper does not include experiments.
        \item The paper should indicate the type of compute workers CPU or GPU, internal cluster, or cloud provider, including relevant memory and storage.
        \item The paper should provide the amount of compute required for each of the individual experimental runs as well as estimate the total compute. 
        \item The paper should disclose whether the full research project required more compute than the experiments reported in the paper (e.g., preliminary or failed experiments that didn't make it into the paper). 
    \end{itemize}
    
\item {\bf Code Of Ethics}
    \item[] Question: Does the research conducted in the paper conform, in every respect, with the NeurIPS Code of Ethics \url{https://neurips.cc/public/EthicsGuidelines}?
    \item[] Answer: \answerYes{} %
    \item[] Justification: We have reviewed the NeurIPS Code of Ethics and our paper conforms with it. 
    \item[] Guidelines:
    \begin{itemize}
        \item The answer NA means that the authors have not reviewed the NeurIPS Code of Ethics.
        \item If the authors answer No, they should explain the special circumstances that require a deviation from the Code of Ethics.
        \item The authors should make sure to preserve anonymity (e.g., if there is a special consideration due to laws or regulations in their jurisdiction).
    \end{itemize}

\item {\bf Broader Impacts}
    \item[] Question: Does the paper discuss both potential positive societal impacts and negative societal impacts of the work performed?
    \item[] Answer: \answerNA{} %
    \item[] Justification: Our work reduces the computational costs of classical methods and is applied to standard datasets. Thus, it has no outsize societal impact. 
    \item[] Guidelines:
    \begin{itemize}
        \item The answer NA means that there is no societal impact of the work performed.
        \item If the authors answer NA or No, they should explain why their work has no societal impact or why the paper does not address societal impact.
        \item Examples of negative societal impacts include potential malicious or unintended uses (e.g., disinformation, generating fake profiles, surveillance), fairness considerations (e.g., deployment of technologies that could make decisions that unfairly impact specific groups), privacy considerations, and security considerations.
        \item The conference expects that many papers will be foundational research and not tied to particular applications, let alone deployments. However, if there is a direct path to any negative applications, the authors should point it out. For example, it is legitimate to point out that an improvement in the quality of generative models could be used to generate deepfakes for disinformation. On the other hand, it is not needed to point out that a generic algorithm for optimizing neural networks could enable people to train models that generate Deepfakes faster.
        \item The authors should consider possible harms that could arise when the technology is being used as intended and functioning correctly, harms that could arise when the technology is being used as intended but gives incorrect results, and harms following from (intentional or unintentional) misuse of the technology.
        \item If there are negative societal impacts, the authors could also discuss possible mitigation strategies (e.g., gated release of models, providing defenses in addition to attacks, mechanisms for monitoring misuse, mechanisms to monitor how a system learns from feedback over time, improving the efficiency and accessibility of ML).
    \end{itemize}
    
\item {\bf Safeguards}
    \item[] Question: Does the paper describe safeguards that have been put in place for responsible release of data or models that have a high risk for misuse (e.g., pretrained language models, image generators, or scraped datasets)?
    \item[] Answer: \answerNA{} %
    \item[] Justification: We do not release models or data as part of this paper.
    \item[] Guidelines:
    \begin{itemize}
        \item The answer NA means that the paper poses no such risks.
        \item Released models that have a high risk for misuse or dual-use should be released with necessary safeguards to allow for controlled use of the model, for example by requiring that users adhere to usage guidelines or restrictions to access the model or implementing safety filters. 
        \item Datasets that have been scraped from the Internet could pose safety risks. The authors should describe how they avoided releasing unsafe images.
        \item We recognize that providing effective safeguards is challenging, and many papers do not require this, but we encourage authors to take this into account and make a best faith effort.
    \end{itemize}

\item {\bf Licenses for existing assets}
    \item[] Question: Are the creators or original owners of assets (e.g., code, data, models), used in the paper, properly credited and are the license and terms of use explicitly mentioned and properly respected?
    \item[] Answer: \answerYes{} %
    \item[] Justification: We include URL and licenses for baseline code and datasets used in \cref{sec:real-data}.
    \item[] Guidelines:
    \begin{itemize}
        \item The answer NA means that the paper does not use existing assets.
        \item The authors should cite the original paper that produced the code package or dataset.
        \item The authors should state which version of the asset is used and, if possible, include a URL.
        \item The name of the license (e.g., CC-BY 4.0) should be included for each asset.
        \item For scraped data from a particular source (e.g., website), the copyright and terms of service of that source should be provided.
        \item If assets are released, the license, copyright information, and terms of use in the package should be provided. For popular datasets, \url{paperswithcode.com/datasets} has curated licenses for some datasets. Their licensing guide can help determine the license of a dataset.
        \item For existing datasets that are re-packaged, both the original license and the license of the derived asset (if it has changed) should be provided.
        \item If this information is not available online, the authors are encouraged to reach out to the asset's creators.
    \end{itemize}

\item {\bf New Assets}
    \item[] Question: Are new assets introduced in the paper well documented and is the documentation provided alongside the assets?
    \item[] Answer: \answerYes{} %
    \item[] Justification: We release our code with documentation at \url{https://github.com/ag2435/npr} under a BSD-3 Clause license.
    \item[] Guidelines:
    \begin{itemize}
        \item The answer NA means that the paper does not release new assets.
        \item Researchers should communicate the details of the dataset/code/model as part of their submissions via structured templates. This includes details about training, license, limitations, etc. 
        \item The paper should discuss whether and how consent was obtained from people whose asset is used.
        \item At submission time, remember to anonymize your assets (if applicable). You can either create an anonymized URL or include an anonymized zip file.
    \end{itemize}

\item {\bf Crowdsourcing and Research with Human Subjects}
    \item[] Question: For crowdsourcing experiments and research with human subjects, does the paper include the full text of instructions given to participants and screenshots, if applicable, as well as details about compensation (if any)? 
    \item[] Answer: \answerNA{} %
    \item[] Justification: We do not have any studies or results regarding crowdsourcing experiments and human subjects.
    \item[] Guidelines:
    \begin{itemize}
        \item The answer NA means that the paper does not involve crowdsourcing nor research with human subjects.
        \item Including this information in the supplemental material is fine, but if the main contribution of the paper involves human subjects, then as much detail as possible should be included in the main paper. 
        \item According to the NeurIPS Code of Ethics, workers involved in data collection, curation, or other labor should be paid at least the minimum wage in the country of the data collector. 
    \end{itemize}

\item {\bf Institutional Review Board (IRB) Approvals or Equivalent for Research with Human Subjects}
    \item[] Question: Does the paper describe potential risks incurred by study participants, whether such risks were disclosed to the subjects, and whether Institutional Review Board (IRB) approvals (or an equivalent approval/review based on the requirements of your country or institution) were obtained?
    \item[] Answer: \answerNA{}%
    \item[] Justification: We do not have any studies or results including study participants. 
    \item[] Guidelines:
    \begin{itemize}
        \item The answer NA means that the paper does not involve crowdsourcing nor research with human subjects.
        \item Depending on the country in which research is conducted, IRB approval (or equivalent) may be required for any human subjects research. If you obtained IRB approval, you should clearly state this in the paper. 
        \item We recognize that the procedures for this may vary significantly between institutions and locations, and we expect authors to adhere to the NeurIPS Code of Ethics and the guidelines for their institution. 
        \item For initial submissions, do not include any information that would break anonymity (if applicable), such as the institution conducting the review.
    \end{itemize}

\end{enumerate}

\end{document}